\documentclass[12pt]{report}   	
\usepackage{amsmath, amsthm, amsfonts, amssymb}
\usepackage[colorlinks,citecolor=blue,urlcolor=blue]{hyperref}
\usepackage{relsize}
\usepackage{enumerate}
\usepackage[T1]{fontenc} 
\usepackage[utf8]{inputenc} 
\usepackage{xcolor} 
\usepackage{graphicx}     
\usepackage{multirow}
\usepackage{multicol}
\usepackage{subfigure}
\usepackage{bbm}
\usepackage{mathtools, nccmath}
\usepackage{tikz}
\usepackage[titletoc]{appendix}
\usetikzlibrary{shapes,backgrounds}
\usetikzlibrary{shapes.geometric, arrows}
\usepgflibrary{snakes} % LATEX and plain TEX and pure pgf
\usepgflibrary[snakes] % ConTEXt and pure pgf
\usetikzlibrary{snakes} % LATEX and plain TEX when using Tik Z
\usetikzlibrary[snakes] % ConTEXt when using Tik Z

\newtheorem{theorem}{Theorem}[section]
\newtheorem{corollary}[theorem]{Corollary}
\newtheorem{lemma}[theorem]{Lemma}
\newtheorem{proposition}[theorem]{Proposition}

\theoremstyle{definition}
\newtheorem{definition}[theorem]{Definition}
\theoremstyle{remark}
\newtheorem{remark}[theorem]{Remark}
\newtheorem{example}[theorem]{Example}
\newcommand\xqed[1]{%
  \leavevmode\unskip\penalty9999 \hbox{}\nobreak\hfill
  \quad\hbox{#1}}
\newcommand\demo{\xqed{$\triangle$}}
	
\makeatletter\@addtoreset{chapter}{part}\makeatother
\begin{document}

\thispagestyle{empty}
{
\sffamily
\centering
\Large

~\vspace{\fill}

{\huge 
{{\rm \textbf{Topics in Random Matrices \\ \vspace{.2cm} and \vspace{.2cm} \\ Statistical Machine Learning} }}	
} 

\vspace{2.5cm}

{\LARGE
Sushma Kumari
}

\vspace{2.5cm}

A Dissertation Submitted to \\[1em]
Department of Mathematics, Graduate School of Science\\
Kyoto University \\[1em]
 In Partial Fulfillment of the Requirements for the\\
Degree of Doctor of Philosophy in\\[1em]
Mathematics\\[1em]

\vspace{1.5cm}

{\Large Kyoto, July 2018}

\vspace{\fill}

\date{\today}
}
\cleardoublepage
%%%---%%%---%%%---%%%---%%%---%%%---%%%---%%%---%%%---%%%---%%%---%%%---%%%
%%%---%%%---%%%---%%%---%%%---%%%---%%%---%%%---%%%---%%%---%%%---%%%---%%%

\clearpage
%%%%%%%%%%%%%%%%%%%%%%%%%%%%%%%%%%%%%%%%%%%%%%%%%%%%%%%%%%%%%%%%%%%%%%%%%%%%%%%%%%%%%%%
\chapter*{Abstract}
%%%%%%%%%%%%%%%%%%%%%%%%%%%%%%%%%%%%%%%%%%%%%%%%%%%%%%%%%%%%%%%%%%%%%%%%%%%%%%%%%%%%%%%
This thesis consists of two independent parts: random matrices, which form the first one-third of this thesis, and machine learning, which constitutes the remaining part.
  
The classical Wishart matrix has been defined only for the values $\beta = 1,2$ and $4$ (corresponding to real, complex and quaternion cases respectively), where $\beta$ indicates the number of real matrices needed to define a particular type of Wishart matrix. The moments and inverse moments of Wishart matrices have their theoretical and practical importance. In the works of Graczyk, Letac and Massam (2003, 2004), Matsumoto (2012), Collins et al. (2014), a certain additional condition is assumed in order to derive a formula for finite inverse moments of Wishart matrices. Here, we address the necessity of this additional condition. In general, we consider the question of having finite inverse moments for two bigger classes of Wishart-type matrices: the $(m,n,\beta)$-Laguerre matrices defined for continuous values of $\beta >0$ and compound Wishart matrices for the values of $\beta=1$ (real) and $2$ (complex).

We show that the $c$-th inverse moment of a $(m,n,\beta)$-Laguerre matrix is finite if and only if $c < (m-n+1)\beta/2$, for $\beta >0$. Moreover, we deduce that the $c$-th inverse moment of a compound Wishart matrix is finite if and only if $c < (m-n+1)\beta/2$, for $\beta =1,2$. The definition of compound Wishart matrix in quaternion case ($\beta =4$) is not so coherent yet, so the condition for finiteness of inverse moments in this case is a future work. 

The second part of the thesis is devoted to the subject of the universal consistency of the $k$-nearest neighbor rule in general metric spaces. The $k$-nearest neighbor rule is a well-known learning rule and one of the most important. Given a labeled sample, the $k$-nearest neighbor rule first find `$k$' data points in the sample, which are closest to $x$ based on a distance function and then predicts the label of $x$ as being the most commonly occurring label among the picked `$k$' labels. There is an error if the predicted label is not same as the true label. A learning rule is universally weakly consistent if the expected (average) learning error converges to the smallest possible error for the given problem (known as the Bayes error).

 According to the 2006 result of C{\'e}rou and Guyader, the $k$-nearest neighbor rule is universally weakly consistent in every metric space equipped with probability measure satisfying the strong differentiation property. A 1983 result announced by Preiss states necessary and sufficient condition for a metric space to satisfy the strong differentiation property for all finite Borel measures. This is the condition of being metrically sigma-finite dimensional in the sense of Nagata. Thus, in every sigma-finite dimensional metric space in the sense of Nagata, the $k$-nearest neighbor rule is universally weakly consistent.

The main aim of this part of the thesis is to prove the above result by direct means of statistical learning theory, bypassing the machinery of real analysis. Our proof is modeled on the classical
proof by Charles Stone for the Euclidean space. However, the main tool of his proof, the geometric Stone lemma, only makes sense in the presence of the finite dimensional linear structure. The lemma gives an upper bound on the number of points in a sample for which a given point can serve as one of the $k$-nearest neighbors. We search for an analogue of the geometric Stone lemma for metrically (sigma) finite dimensional spaces in Nagata's sense, making a number of interesting discoveries on the way. While in the absence of distance ties there is a straightforward analogue of the lemma, it is provably false in the presence of ties, and besides, we show that the distance ties in general metrically finite-dimensional (even zero-dimensional) spaces are unavoidable. At the same time, it turns out that the upper bound in the Stone lemma, although unbounded, grows slowly in $n$ (as the $n$-th
harmonic number), which allows to deduce the universal consistency. 

Further, we establish strong consistency in a metrically finite dimensional space, under the additional condition of zero probability of ties. In the Euclidean case, the result is known in the general case, but historically, it was also first proved in the absence of ties. We leave the question of validity of the result in a metrically sigma-finite dimensional space as an open question.

Finally, we work out in detail the necessity part of the proof of the Preiss theorem above. The original note by Preiss only briefly outline the ideas of the proof in a few lines, and to work out sufficiency, Assouad and Quentin de Gromard had written a 61-page long article. The details of the necessity part appear in our thesis for the first time.

%%%%%%%%%%%%%%%%%%%%%%%%%%%%%%%%%%%%%%%%%%%%%%%%%%%%%%%%%%%%%%%%%%%%%%%%%%%%%%%%%%%%%%%
\chapter*{Declaration}
I hereby declare the thesis entitled ``Topics in Random Matrices and Statistical Machine Learning'' has been undertaken by me and reflect my original work. All sources of knowledge used have been duly acknowledged. 
I declare that this thesis has never been submitted and/or published for any award to any other institution before.  

\bigskip
\vspace{2cm}
Sushma Kumari 

(18 July 2018)
%%%%%%%%%%%%%%%%%%%%%%%%%%%%%%%%%%%%%%%%%%%%%%%%%%%%%%%%%%%%%%%%%%%%%%%%%%%%%%%%%%%%%%%
\chapter*{Acknowledgments} \noindent
This thesis is a result of support and guidance of many people. I would like to extend my sincere thanks to all of them.
Firstly, I would like to express my sincere gratitude to my supervisors Dr. Beno{\^i}t Collins and Dr. Vladimir G. Pestov for their guidance and encouragement. I greatly acknowledge the support of Dr. Collins during my PhD. I am very grateful to Dr. Pestov for his guidance and critical comments on this research work over emails and calls irrespective of the 12-hour time difference between Brazil and Kyoto. I sincerely acknowledge the help by Dr. Hiroshi Kokubu and the financial support of Kyoto Top Global Unit (KTGU) for Brazil overseas trip. I greatly acknowledge the hospitality of Department of Mathematics, Kyoto University and financial support of JICA-IITH Friendship program for giving me  an opportunity to pursue a doctoral course at Kyoto University.  I would like to thank my thesis committee members for their support and suggestions. 

My special thanks goes to my best friend, Mr. Akshay Goel, for the umpteen number of discussions we had over varied topics of mathematics. I would like to thank my super-friends Ms. Jasmine Kaur and Ms. Akanksha Yadav for keeping with me in my good and bad days. I thank my fellow colleagues Gunjan, Prashant, Reddy, Mathieu, Felix for being there. Last but not the least, a very special thanks to my lovely family: mumy, papa and bhai (mother, father and elder brother). Their constant assurance and faith in me has made me come so far.

%%%%%%%%%%%%%%%%%%%%%%%%%%%%%%%%%%%%%%%%%%%%%%%%%%%%%%%%%%%%%%%%%%%%%%%%%%%%%%%%%%%%%%%
\chapter*{Dedication}
to my family...
%%%%%%%%%%%%%%%%%%%%%%%%%%%%%%%%%%%%%%%%%%%%%%%%%%%%%%%%%%%%%%%%%%%%%%%%%%%%%%%%%%%%%%%
\tableofcontents
%%%%%%%%%%%%%%%%%%%%%%%%%%%%%%%%%%%%%%%%%%%%%%%%%%%%%%%%%%%%%%%%%%%%%%%%%%%%%%%%%%%%%%%
\chapter*{List of Notations}
\addcontentsline{toc}{chapter}{List of Notations}
Here, we list the main notations, which are used in both parts of the thesis but in different context.
\begin{center}
Part I \vspace{-0.2cm}
\end{center}
\begin{align*}
\beta \ \ \ \ \ \ & \text{ parameter to define matrix ensemble} \ \\
A,B \ \ \ \ \ \ & \text{matrix} \ \\
Q \ \ \ \ \ \ & \text{ compound Wishart matrix} \ \\
\lambda,\xi \ \ \ \ \ \ & \text{ eigenvalues } \ \\
\chi_{s} \ \ \ \ \ \ & \text{ chi distribution with parameter s } \ \\
n, m \ \ \ \ \ \ & \text{ size of a matrix} 
\end{align*}
\begin{center}
Part II \vspace{-0.2cm}
\end{center}
\begin{align*}
\beta \ \ \ \ \ \ & \text{ dimension of a metric in Nagata sense} \ \\
A,B \ \ \ \ \ \ & \text{measurable sets } \ \\
\Omega \ \ \ \ \ \ & \text{separable metric space} \ \\
Q \ \ \ \ \ \ & \text{metric space, } Q \subseteq \Omega \ \\
n, m \ \ \ \ \ \ & \text{sample size, sub-sample size} \ \\
\chi_{M}(x) \ \ \ \ \ \ & \text{characteristic function of some set $M$,} \\ &  \text{equals to 1 if $x \in M$, else 0 } \vspace{0.3cm}
\end{align*} 
Some of the frequently used notations in the thesis are:
\begin{align*}
\mathbb{I}_{\{x_i \in A\}} \ \ \ \ \ & \text{indicator function, equals to 1 if $x_i$ is in} \\ & \text{set $A$, otherwise 0} \ \\
\sharp\{A\} \ \ \ \ \ & \text{cardinality of set $A$} \ \\
\rho \ \ \ \ \ & \text{metric}\ \\
B(x,r), \bar{B}(x,r),S(x,r) \ \ \ \ \ & \text{open ball, closed ball, sphere respectively, at $x$ and radius $r$} 
\end{align*}
%%%%%%%%%%%%%%%%%%%%%%%%%%%%%%%%%%%%%%%%%%%%%%%%%%%%%%%%%%%%%%%%%%%%%%%%%%%%%%%%%%%%%%%
\chapter*{List of Figures}
\addcontentsline{toc}{chapter}{List of Figures}
\begin{description}
\item[Figure \ref{fig:unconnected}] An illustration for an unconnected family of balls
\item[Figure \ref{fig:finite_dimension}] An illustration that real line has metric dimension 2
\item[Figure \ref{fig:nagata_dimension}] An illustration that real line has Nagata dimension 1
\item[Figure \ref{fig:ball-covering_dimension}] An illustration that real line has ball-covering dimension 2
\item[Figure \ref{fig:classification}] Illustration of a binary classification problem, classifying new data points into `rectangle' and `black dot'
\item[Figure \ref{fig:consistent rule}] Illustration of a weakly consistent rule and a smart learning rule \cite{Hatko_2015}
\item[Figure \ref{fig:bayes_rule}] Defining the Bayes classifier $g^{*}$ based on the values of regression function $\eta$ \cite{Cerou_Guyader_2006}
\item[Figure \ref{fig:knn_classification}] Illustration of the $k$-nearest neighbor rule with voting ties for $k=2$ and distance ties for $k=3$
\item[Figure \ref{fig:cones}] A cone of angle $\pi/6$ \cite{Devroye_Gyorfi_Lugosi_1996}
\item[Figure \ref{fig:stone_lemma}] Illustration of geometric Stone's lemma in Euclidean spaces (adapted from Fig. 5.5 of \cite{Devroye_Gyorfi_Lugosi_1996})
\item[Figure \ref{fig:stone_lemma_fails_construction}] Illustration of construction of a sample for which Stone's lemma fails 
\end{description}
%%%%%%%%%%%%%%%%%%%%%%%%%%%%%%%%%%%%%%%%%%%%%%%%%%%%%%%%%%%%%%%%%%%%%%%%%%%%%%%%%%%%%%%
\chapter*{An Overview}
\addcontentsline{toc}{chapter}{An Overview}
%%%%%%%%%%%%%%%%%%%%%%%%%%%%%%%%%%%%%%%%%%%%%%%%%%%%%%%%%%%%%%%%%%%%%%%%%%%%%%%%%%%%%%%
I started my `research life' in the second year of my Masters at Indian Institute of Technology, Hyderabad under the guidance of Dr. Balasubramaniam Jayaram. I worked with Dr. Jayaram on finding a yardstick to empirically measure the concentration of various distance functions. We published a paper entitled `Measuring Concentration of Distances-An Effective and Efficient Empirical Index' in IEEE-TKDE. 

In 2015, I joined as a doctoral student under Dr. Beno{\^i}t Collins to pursue my newly developing interest in random matrices. In the first year of my PhD, I worked on the problem of finding a necessary and sufficient condition to have finite inverse moments for $(m,n,\beta)$-Laguerre matrices. Based on this work, a paper \cite{Kumari_2018} entitled `Finiteness of Inverse Moments of $(m,n,\beta)$-Laguerre matrices' has been accepted in Infinite Dimensional Analysis, Quantum Probability, and Related Topics.

As I had some research experience in machine learning, after discussing with Dr. Collins, I decided to work in machine learning in the remaining time of my PhD. Dr. Collins introduced me to Dr. Vladimir Pestov, both of them were colleagues at University of Ottawa. Although I had known Dr. Pestov thorough his works on concentration of measure which I studied during my masters, the wish to work with him was made possible by Dr. Collins and Kyoto University. Dr. Pestov is my PhD co-supervisor and we studied the universal consistency of the $k$-nearest neighbor rule in metrically sigma-finite dimensional spaces. We hope to convert this joint work, which constitutes the second part of this thesis, to a scientific paper in the near future.  
 
\begin{center}
{\larger \textit{Overview of the thesis}}
\end{center} 
\pagenumbering{arabic}

This thesis is based on two different areas of mathematics broadly known as random matrix theory and statistical machine learning. The first part of the thesis examines the finiteness of inverse moments of $(m,n,\beta)$-Laguerre matrices and the second part investigates the universal consistency of $k$-nearest neighbor rule in metrically sigma-finite dimensional spaces. According to the literature, the theory of random matrices are often employed in different areas of machine learning such as in dimensionality reduction and random projections. Some recent works of Romain Couillet and others \cite{Liao_Couillet_2017,Mai_Couillet_2017,Louart_Liao_Couillet_2017} present the emerging applications of random matrices in machine learning.  

However, the two topics discussed in this thesis are entirely independent of each other.

\begin{center}
\textit{Part I}
\end{center} 
The extensive study of random matrices specifically, Wishart matrices, is credited to the pioneering work of John Wishart \cite{Wishart_1928} in 1928. John Wishart studied the real Wishart matrices in relation to the sample covariance matrices from a multivariate Gaussian distribution. The complex Wishart matrices were introduced by N. R. Goodman \cite{Goodman_1963}. Originally, the classical Wishart ensemble was defined only for the parameter $\beta = 1,2$ and $4$ corresponding to real, complex and quaternion Wishart matrices respectively. A while later in 2002, Dumitriu and Edelman \cite{Dumitriu_Edelman_2002}  generalized the classical Wishart ensemble to a tri-diagonal matrix ensemble called $(m,n,\beta)$-Laguerre ensemble, for the general values of $\beta >0$ having similar eigenvalue distribution. Another generalization of the Wishart matrices, called compound Wishart matrices for the values $\beta = 1$ and $2$, was introduced by Roland Speicher \cite{Speicher_1998}.

Let $A$ be a random matrix then for integer $c > 0$, $\mathbb{E}\{{ \rm Tr }A^{c} \}$ and $\mathbb{E}\{{ \rm Tr }A^{-c} \}$ are called the $c$-th moment and $c$-th inverse moment of $A$, respectively. Letac and Massam \cite{Letac_Massam_2004} were the first to compute all the general moments of Wishart and inverse Wishart matrices of the form  $\mathbb{E}\{Q(S)\}$ and $\mathbb{E}\{Q(S^{-1})\}$ in both real and complex cases, where $Q$ is a polynomial depending only on eigenvalues of the corresponding matrix $S$ or $S^{-1}$. Later, Sho Matsumoto \cite{Matsumoto_2012} gave the formula for all the general moments and inverse moments of Wishart matrices using Weingarten function. The explicit expression for the inverse moments of a compound Wishart matrix was obtained by Collins et al. in \cite{Collins_Matsumoto_Saad_2014}. 

In \cite{Collins_Matsumoto_Saad_2014,Graczyk_Letac_Massam_2003,Letac_Massam_2004,Matsumoto_2012}, an additional condition such as $c < m-n+1$ for $\beta =2 $ (in complex case) and $c < (m-n+1)/{2}$  for $\beta =1$ (in real case) was assumed to compute the finite $c$-th inverse moment of a Wishart and compound Wishart matrix. Interestingly, it is not known whether this additional condition is necessary to have finite inverse moments. This work is motivated by this question. We consider this property, to have finite inverse moments, in general for a broader family of $(m,n,\beta)$-Laguerre matrices. In this thesis, we present a necessary and sufficient condition for the finite inverse moments of $(m,n,\beta)$-Laguerre matrices and compound Wishart matrices to exist. The main contribution of the first part of the thesis is as follows:
\begin{enumerate}[(i)]
\item Let $S$ be a $(m,n,\beta)$-Laguerre matrix, then 
\begin{align*}
\mathbb{E}\{ {\rm Tr}(S^{-c})\} \text{\  is finite if and only if \ } c < \frac{(m -n +1)\beta}{2}. 
\end{align*}
\end{enumerate}
This finiteness condition is derived from the eigenvalue distribution. Since Wishart matrices and $(m,n,\beta)$-Laguerre matrices have same eigenvalue distribution for $\beta = 1,2$ and $4$, in particular, the finiteness condition also holds for Wishart matrices. As a natural consequence, we also give a necessary and sufficient condition for the compound Wishart matrices to have finite inverse moments. 

\begin{center}
\textit{Part II}
\end{center}
The $k$-nearest neighbor rule is one of the simplest, oldest and yet the most popular learning rules in statistical machine learning. To predict a label for $x$, the $k$-nearest neighbor rule first find `$k$' labeled data points among a given labeled sample of $n$ data points, which are closest to $x$ with regard to some distance function, not necessarily a metric, and takes a majority vote among the selected `$k$' labels. Large part of the theory developed for the $k$-nearest neighbor rule is for metric spaces due to their well-understood properties. The first proof for universal weak consistency of the $k$-nearest neighbor rule in a finite dimensional Euclidean space $\mathbb{R}^d$ was given by Charles Stone \cite{Stone_1977} in 1977. 
He showed that the expected misclassification error converges in probability to the smallest possible error (also known as Bayes error) as the sample size grows. Stone listed three important conditions that are sufficient to yield universal weak consistency of a learning rule in any finite dimensional normed space. Indeed, these conditions by Stone have more general importance, two of the Stone's conditions hold for any separable metric space whenever $n,k \rightarrow \infty$ and $k/n \rightarrow 0$.

 The proof of Stone's theorem was based on a geometrical argument, the so-called (geometric) Stone's lemma. The basic idea is to partition $\mathbb{R}^d$ into $L$ number of sets with some special convexity properties and show that a point cannot serve as the $k$-nearest neighbor of more than $kL$ sample points, where $L$ is a constant depending only on the dimension $d$ and the norm. The proof of geometric Stone's lemma highly relies on the structure of $\mathbb{R}^d$ and thus is limited to finite dimensional Euclidean spaces or, more generally, finite dimensional normed spaces \cite{Duan_2014}. The third condition by Stone is called the Stone's lemma and is known to be true only for finite dimensional normed spaces.

After almost three decades, C{\'e}rou and Guyader \cite{Cerou_Guyader_2006} proved, developing the ideas of Devroye \cite{Devroye_1981}, that the $k$-nearest neighbor rule is universally weakly consistent in a broader class of metric spaces namely, those satisfying the weak Lebesgue-Besicovitch differentiation property.  It is known (1983, David Preiss \cite{Preiss_1981}) that a complete separable metric space satisfies the strong Lebesgue-Besicovitch differentiation property if and only if the space is metrically sigma-finite dimensional. Therefore, the $k$-nearest neighbor rule is universally weakly consistent in a complete separable and metrically sigma-finite dimensional space. It was left open by Preiss whether sigma-finite metric dimension of a space is necessary for the weak Lebesgue-Besicovitch differentiation property to hold. Mattila \cite{Mattila_1971} showed that for a given measure the strong and weak Lebesgue-Besicovitch differentiation property may not be equivalent.
Strengthening the conclusion of Stone's theorem, Devroye et al. \cite{Devroye_Gyorfi_Krzyzak_Lugosi_1994} proved that the universal weak consistency and universal strong consistency are equivalent in Euclidean spaces.

Our focus is primarily on metric spaces with finite and sigma-finite metric dimension. The following flow diagram illustrate the bridge between universal consistency, differentiation property and dimension of a metric.

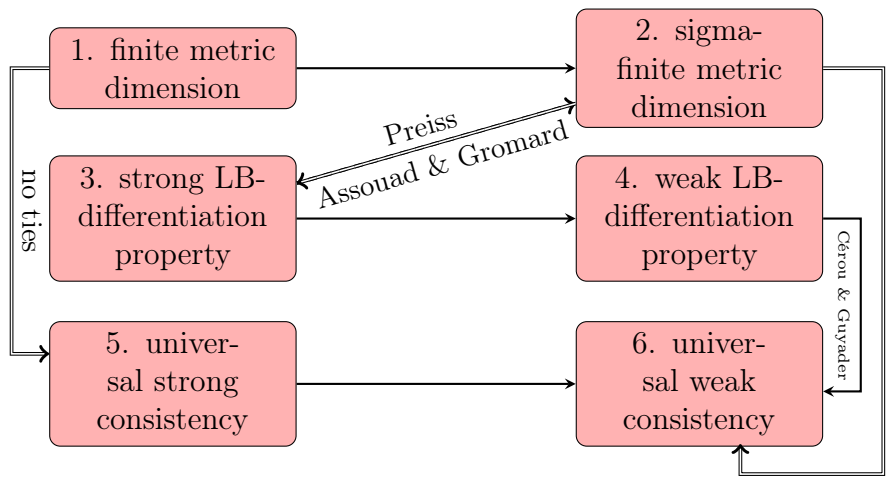
\vspace{0.28cm}
\begin{tikzpicture}[node distance=2cm]
\tikzstyle{arrow} = [thick,->,>=stealth]
\tikzstyle{startstop} = [rectangle, rounded corners, minimum width=3cm, minimum height=1cm,text centered, draw=black, fill=red!30]
\tikzstyle{io} = [rectangle, rounded corners, minimum width=3cm, minimum height=1cm,text centered,text width=3cm, draw=black, fill=red!30]

\tikzstyle{line} = [draw, -latex']

\node (start) [io] {2. sigma-finite metric dimension};
\node (in5) [io,xshift=-5cm, left of=start] {1. finite metric dimension};
\node (in1) [io, below of=in5] {3. strong LB-differentiation property};
\node (in2) [io, below of=in1,yshift=-0.2cm] {5. universal strong consistency};
\node (in3) [io, below of=start] {4. weak LB-differentiation property};
\node (in4) [io, below of=in3,yshift=-0.2cm] {6. universal weak consistency};

\draw[double,->] (in5.west) -- ++(-0.5cm,0) -- ++ (0,-3.8cm) node[pos=0.5,sloped,anchor = center,above] {{\small no ties}} -- ++(0.5cm,0) (in4);
\draw[arrow] (in5) -> (start);
\draw [double,->] (in1) -> node[pos=0.5,sloped,anchor=center,above]{{\small Preiss} \ \ }(start);
\draw[dashed,double,->] (start) ->  node[pos=0.5,sloped,anchor=center,below]{{\small Assouad \& Gromard}} (in1);
\draw [arrow] (in1) -> (in3);
\draw[arrow] (in2) -> (in4);
\draw [arrow] (in3.east) -- ++(0.5cm,0) -- ++(0,-2.3cm) node[pos=0.5,sloped,anchor = center,below] {{\tiny C{\'e}rou \& Guyader}}-- ++(-0.5cm,0)  (in4);
\draw[double,->] (start.east) -- ++(0.8cm,0) -- ++(0,-5.4cm) -- ++(-1.9cm,0) -- ++(0,0.4cm) (in4);
\end{tikzpicture}

In the above diagram, the thick double line represent our results. We already have the following implications:
\begin{enumerate}[(i)]
\item Preiss, Assouad and Gromard: $2 \Leftrightarrow 3$.
\item C{\'e}rou and Guyader: $4 \Rightarrow 6$.
\item Always true: $1 \Rightarrow 2$, $3 \Rightarrow 4$ and $5 \Rightarrow 6$.
\end{enumerate}
Our principal goal is to investigate the universal consistency of the $k$-nearest neighbor rule in a metrically sigma-finite dimensional space, from the machine learning perspective. The classical method to establish the universal consistency is to use Stone's theorem. While generalizing the Stone's theorem in a metrically finite dimensional space, we encountered a number of interesting observations, presented in Chapter \ref{chap:Consistency in a metrically finite dimensional spaces}. Furthermore, the study of strong consistency in Euclidean spaces was fundamentally initiated by Devroye around $1981$ \cite{Devroye_1981}.  The strong consistency was first proved under the assumption of absolute continuity of measures (that is zero probability of ties) \cite{Devroye_Gyorfi_1985}, while the universal strong consistency was established under appropriate tie-breaking method much later \cite{Devroye_Gyorfi_Krzyzak_Lugosi_1994}. In fact, a much stronger statement is true, the notions of weak and strong consistency for the $k$-nearest neighbor rule are equivalent in Euclidean spaces \cite{Devroye_Gyorfi_Lugosi_1996}.  When working with strong consistency in the presence of ties, a good tie-breaking method is needed as the solution become much more complicated.  In our attention, there are almost no developments on strong consistency of the $k$-nearest neighbor rule in metric spaces other than Euclidean spaces. 

After a brief review of the literature on universal consistency of the $k$-nearest neighbor rule, we see that there are numerous directions for theoretical work. We have accomplished some of them in this thesis. We outline our contributions for part II of this thesis as following:
\begin{enumerate}
\renewcommand{\labelenumi}{(\roman{enumi}) }
\setcounter{enumi}{1}
\item (\textit{$3 \Rightarrow 2$}): Preiss has sketched the proof of $2 \Leftrightarrow 3$ very briefly without any details, where the sufficiency part was completed by Assouad and Gromard \cite{Assouad_Gromard_2006}.  A thorough explanation of necessity of sigma-finite metric dimension has not been done before. We give a detailed proof of necessity part of the Preiss' result (refer to Section \ref{sec:A result by Preiss}), that is, the strong Lebesgue-Besicovitch differentiation property holds only if the space is metrically sigma-finite dimensional.
\item  We generalize the Stone's lemma in metric spaces with finite Nagata dimension under the assumption of no distance ties. As a consequence, we give an alternate proof for weak consistency of the $k$-nearest neighbor rule in metric spaces with finite Nagata dimension under the additional assumption of no distance ties (refer to Section~\ref{sec:Consistency without distance ties}). We also present some examples reflecting the problem with distance ties. One of the major issues is that the Stone's lemma fails in presence of distance ties, which indicates that the classical method of using Stone's theorem may not be a right way to prove universal consistency in such metric spaces (refer to Section~\ref{sec:Consistency with distance ties}). 
\item (\textit{$2 \Rightarrow 6$}): We reestablish the universal weak consistency of the $k$-nearest neighbor rule in a metrically sigma-finite dimensional space under the random uniform tie-breaking method. Stone's lemma fails in the presence of distance ties, so we give another geometric lemma to work with distance ties. Using this lemma and not the argument of differentiation property, we give a direct and simpler proof for universal weak consistency of the $k$-nearest neighbor rule (refer to Section~\ref{sec:Consistency with distance ties}). This may provide an insight in establishing universal consistency for other learning rules where the Lebesgue-Besicovitch differentiation property and other real analysis techniques are not so coherent.
\item (\textit{$1 \Rightarrow 4$} partially): We also prove the strong consistency of the $k$-nearest neighbor rule in a separable space which has finite metric dimension under the assumption of zero probability of distance ties (refer to Section~\ref{sec:Strong consistency}). This is a new result in this direction as all the previous results on strong consistency in \cite{Devroye_Gyorfi_Lugosi_1996} are limited to Euclidean spaces. 
\item Davies \cite{Davies_1971} has constructed an example of a compact metric space of diameter 1 and two distinct Borel measures which gives equal values to all closed ball of radius $<1$. The Davies' example fails the differentiation property  and therefore, by the result of C{\'e}rou and Guyader, the $k$-nearest neighbor rule is not consistent on Davies' example. We modify the two Borel measures, constructed by Davies, to show the inconsistency of $k$-nearest neighbor rule on Davies' example directly without using the differentiation argument.
\end{enumerate} 

This thesis is organized in the following way: In Part I, the Chapter \ref{chap:Random matrices} introduces the $(m,n,\beta)$-Laguerre matrix, Wishart and compound Wishart matrix and their joint eigenvalue distribution. While in Chapter \ref{chap:Finiteness of inverse moments}, a necessary and sufficient condition to have finite inverse moments has been derived. 

In part II, the Chapter \ref{chap:Dimension of a metric} introduces the various notions of metric dimension and differentiation property followed by our proof for the necessary part of Preiss' result. Further, Chapter \ref{chap:Statistical machine learning} gives an introduction to mathematical concepts in statistical machine learning and then the $k$-nearest neighbor rule is presented in Chapter \ref{chap:The $k$-nearest neighbor rule} with a proof of Stone's theorem. In Chapter~\ref{chap:Consistency in a metrically finite dimensional spaces} and Chapter \ref{chap:Future Prospects} we present our main results and some possible future directions based on it. 

%%%%%%%%%%%%%%%%%%%%%%%%%%%%%%%%%%%%%%%%%%%%%%%%%%%%%%%%%%%%%%%%%%%%%%%%%%%%%%%%%%%%%%%
%%%%%%%%%%%%%%%%%%%%% Part I  Inverse moments of \beta-ensembles %%%%%%%%%%%%%%%%%%%%%%
\part{Finiteness of Inverse Moments of $(m,n,\beta)$-Laguerre Matrices}
%%%%%%%%%%%%%%%%%%%%%%%%%%%%%%%%%%%%%%%%%%%%%%%%%%%%%%%%%%%%%%%%%%%%%%%%%%%%%%%%%%%%%%%
\chapter{Random Matrices}
\label{chap:Random matrices}
%%%%%%%%%%%%%%%%%%%%%%%%%%%%%%%%%%%%%%%%%%%%%%%%%%%%%%%%%%%%%%%%%%%%%%%%%%%%%%%%%%%%%%%
In this chapter, we introduce the Wishart matrices and two of its generalizations, namely, $(m,n,\beta)$-Laguerre matrices and compound Wishart matrices. We also briefly discuss the joint eigenvalue densities of Wishart matrices and $(m,n,\beta)$-Laguerre matrices.
%%%%%%%%%%%%%%%%%%%%%%%%%%%%%%%%%%%%%%%%%%%%%%%%%%%%%
\section{Wishart matrix}
\label{sec:Wishart matrix}
%%%%%%%%%%%%%%%%%%%%%%%%%%%%%%%%%%%%%%%%%%%%%%%%%%%%%
A matrix whose at least one of the entries is a random variable is called a random matrix. Consider the experiment of tossing a coin and let $\Omega = \{H,T\}$ denote the outcomes. Define a set of random variables $M_{ij}$ from $\Omega$ to $\{0,1\}$ such that $M_{ij}(H) = 0$ and $M_{ij}(T) = 1$ for $1\leq i,j \leq 2$. Then, 
\begin{align*}
M = \begin{pmatrix}
M_{11} & M_{12} \\
M_{21} & M_{22} 
\end{pmatrix}  
\end{align*}
is a $2 \times 2$ random matrix and $ \begin{pmatrix}
0 & 1 \\
1 & 1 
\end{pmatrix}  $
is one of the realizations of $M$.

Random matrix theory is a subject which evolved mainly because of its applications. Random matrices are relevant in numerous fields ranging from number theory and physics to mathematics and machine learning.   

Based on the distribution of its entries, random matrices are categorized as ensembles. Wigner ensemble, Gaussian orthogonal (unitary) ensemble and Wishart ensemble are the most studied with entries from Gaussian distribution. The exceptional properties of Gaussian distribution make these ensembles special in their areas of applications. The main interests of random matrix theory share its interests with probability and matrix theory, like studying the limiting distribution of its eigenvalues. The study of random matrices has progressed very rapidly and now there are lots of books available based on the prospects of random matrices you want to explore. \cite{Mehta_2004,Forrester_2010} are the classical texts whereas \cite{Tao_2013}  is a nice way to get introduced to random matrices. Other books like \cite{Couillet_Debbah_2011} and \cite{Foucart_Rauhut_2013} introduces the applications of random matrices in machine learning and compressed sensing with an adequate flavor of pure mathematics. 

This thesis focuses only on Wishart matrices and its generalizations. 	
Let $\mathcal{K}_{m,n}$ denote the space of all $m \times n$ random matrices having independent and identically distributed entries from standard Gaussian distribution. 
\begin{definition}[Wishart matrix \cite{Wishart_1928}]
Consider four random matrices $A_1, \allowbreak A_2$, $C_1,C_2 \in \mathcal{K}_{m,n}$.
\begin{enumerate}[(i)]
\item  The matrix $P_1 = A_{1}^{*}A_1$ is called a  \emph{real Wishart} matrix, where $A_{1}^{*}$ is the transpose of $A_1$.
\item Let $A = A_1 + i A_2$. The matrix $P_2 = A^{*}A $ is called a  \emph{complex Wishart} matrix, where $A^{*}$ denotes the conjugate transpose of $A$.
\item Let $C = C_1 + i C_2$. Then a matrix  of the form 
\begin{align*}
P_4 =  \begin{pmatrix}
A & C \\
-C & A 
	\end{pmatrix}^{*} 
	\begin{pmatrix}
A & C \\
-C & A 
	\end{pmatrix}
	\end{align*}
	is called a \emph{quaternion Wishart} matrix.
\end{enumerate}
\end{definition}
Therefore, a Wishart matrix is of the form $P_{\beta}= A^*A$ defined for the parameter $\beta =1,2$ and $4$ corresponding to real, complex and quaternion Wishart matrices, respectively. Here, the parameter $\beta$ denote the number of different real matrices require to define a Wishart matrix such as, we require 2 real matrices $A_1$ and $A_2$ to define a complex Wishart matrix.  
 
Without explicitly stating, we always assume $m \geq n$, in the first part of this thesis. Suppose $P_{\beta}$ is a positive-definite matrix, then there are $n$ real and positive eigenvalues. Let $0 < \lambda_1\leq \ldots \leq \lambda_n$ be the eigenvalues of $P_{\beta}$ such that the exact value of $\beta$ will be specified whenever necessary. The joint eigenvalue density of a Wishart matrix \cite{Wishart_1928} can be stated as following.
\begin{align}
\label{eqn:jpdf_eigenvalues_wishart} 
\displaystyle h_{\beta}(\lambda_1,\lambda_2,\dots,\lambda_n) \  = \  Z_{m,n}^{\beta} \prod_{i=1}^{n}\lambda_{i}^{\alpha -1} e^{\left(- \frac{1}{2}\sum_{i=1}^{n}\lambda_i\right)} \prod_{k<j}(\lambda_{j} - \lambda_{k})^{\beta},
\end{align}
where $Z_{m,n}^{\beta}$ is the normalization constant and can be explicitly computed. 
The equation \eqref{eqn:jpdf_eigenvalues_wishart} is defined only for the values of $\beta =1,2$ and $4$. In simpler words, corresponding to the values of $\beta =1,2$ and $4$, there are real, complex and quaternion Wishart matrices which have eigenvalue density function $h_{\beta}(\lambda_1,\lambda_2,\dots,\lambda_n)$.

Then, the natural question is whether there any random matrix of Wishart-type which has similar joint eigenvalue density for every values of $\beta > 0$? This was answered by Dumitriu and Edelman \cite{Dumitriu_Edelman_2002}, where they constructed a tri-diagonal matrix of Wishart-type having same joint eigenvalue density as in the equation \eqref{eqn:jpdf_eigenvalues_wishart}. The $(m,n,\beta)$-Laguerre matrix is presented in Section~\ref{sec:mnb_matrix}.

%%%%%%%%%%%%%%%%%%%%%%%%%%%%%%%%%%%%%%%%%%%%%%%%%%%%%
\section{Compound Wishart matrix}
\label{sec:Compound Wishart matrix}
%%%%%%%%%%%%%%%%%%%%%%%%%%%%%%%%%%%%%%%%%%%%%%%%%%%%%
A different yet interesting generalization of Wishart matrices are compound Wishart matrices, which were introduced by Roland Speicher \cite{Speicher_1998}. A compound Wishart matrix is defined only for $\beta=1$ and $2$ corresponding to real and complex compound Wishart matrices, respectively.
\begin{definition}[Compound Wishart matrix \cite{Speicher_1998}] Let $B$ be a $m \times m$ complex deterministic matrix and let $ A$ be a $m \times n$ complex random matrix with independent and identically distributed entries from a standard complex Gaussian distribution, then 
\begin{align*}
 Q & =   A^{*}BA
\end{align*}
is called a complex \emph{compound Wishart} matrix. The matrix $B$ is known as the shape parameter.
\label{def:compound_wishart}
\end{definition}
The matrix $Q$ is a complex Wishart matrix if $B$ is an identity matrix. We note the following observation for $Q$ when $B$ is a positive definite Hermitian matrix. 
\begin{remark}
If $B$ is a positive definite Hermitian matrix, then $B$ has an eigenvalue decomposition that is, $B  = UDU^*$, where $D = diag({\xi_1,\dots,\xi_m})$ such that $0< \xi_1 \leq \dots \leq \xi_m$ and $U $ is a unitary matrix consisting of eigenvectors of $B$. As, $U^{*}  A$ has the same distribution as $ A$, the matrix $Q$ has the same distribution as $ A^{*}DA$.
\demo \end{remark}
In this thesis, we assume that $B$ is positive definite and Hermitian so that the matrix $Q$ has real and positive eigenvalues. The real compound Wishart matrices can also be defined analogously. 
%%%%%%%%%%%%%%%%%%%%%%%%%%%%%%%%%%%%%%%%%%%%%%%%%%%%%
\section{The $(m,n,\beta)$-Laguerre matrix}
\label{sec:mnb_matrix}
%%%%%%%%%%%%%%%%%%%%%%%%%%%%%%%%%%%%%%%%%%%%%%%%%%%%%
Ioana Dumitriu and Alan Edelman generalized the Wishart matrix to $ \allowbreak (m,n,\beta)$-Laguerre matrix such that the equation \eqref{eqn:jpdf_eigenvalues_wishart} is defined for all positive values of $\beta$. By the process of bi-diagonalization, a Wishart matrix can be reduced to its corresponding $(m,n,\beta)$-Laguerre matrix for $\beta =1,2$ and $4$. In this section, we study the $ \allowbreak (m,n,\beta)$-Laguerre matrix and its joint eigenvalue distribution.
\begin{definition}[$(m,n,\beta)$-Laguerre matrix \cite{Dumitriu_Edelman_2002}]
Let $X$ be a bi-diagonal matrix with mutually independent diagonal and sub-diagonal non-zero entries following the distribution,
   \begin{align*}
	\displaystyle X \sim 
	\begin{pmatrix}
	\chi_{ m \beta} & \chi_{(n-1)\beta} & & & \\
	& \chi_{(m-1)\beta}& \chi_{(n-2)\beta}& & \\
	& \ddots & \ddots & & \\
	& & \chi_{(m-n+2)\beta} & \chi_{\beta}& \\
	& &  & \chi_{(m-n+1) \beta}& 
	\end{pmatrix},
	\end{align*}
where $\chi_{s}$ is the chi distribution with parameter $s$. The tri-diagonal matrix $S = X^*X$ is called a \emph{$(m,n,\beta)$-Laguerre matrix}.
\end{definition}
As, $S$ is a $n \times n$ positive-definite Hermitian matrix, there are $n$ real and positive eigenvalues. Dumitriu and Edelman has computed the explicit expression for the joint eigenvalue density function of $(m,n,\beta)$-Laguerre matrices, which can be stated as the following theorem. 
\begin{theorem}[Joint eigenvalue density \cite{Dumitriu_Edelman_2002}]
\label{jpdf_eigenvalues} 
Suppose $\beta > 0$ and let $0 < \lambda_{1} \leq \lambda_2 \leq \dots \leq \lambda_{n}$ be the $n$ ordered eigenvalues of $(m,n,\beta)$-Laguerre matrix $S$. The joint eigenvalue density function is
\begin{align}
\label{eqn:jpdf_eigenvalues} 
\displaystyle h_{\beta}(\lambda_1,\lambda_2,\dots,\lambda_n) \  = \  Z_{m,n}^{\beta} \prod_{i=1}^{n}\lambda_{i}^{\alpha -1} e^{\left(- \frac{1}{2}\sum_{i=1}^{n}\lambda_i\right)} \prod_{k<j}(\lambda_{j} - \lambda_{k})^{\beta} \mathbb{I}_{\{\lambda_1 \leq \ldots \leq \lambda_n \} } ,
\end{align}
where $\alpha = (m-n+1)\beta/2$ and the normalization constant given by
\begin{align*}
\displaystyle  Z_{m,n}^{\beta} \ = \  \frac{2^{-mn\beta/2}}{n!} \prod_{j=1}^{n}  \frac{ \Gamma{\left(1+\frac{\beta}{2}\right)} }{\Gamma{\left(1+\frac{\beta }{2}j\right)} \Gamma{\left( \frac{\beta}{2}(m-n+j)\right)}}.
\end{align*}
\end{theorem}
The indicator function $\mathbb{I}_{\{\lambda_1 \leq \ldots \leq \lambda_n \} }$ in the equation \eqref{eqn:jpdf_eigenvalues} ensures that the density function $h_{\beta}$  is non-zero if and only if the eigenvalues ${\lambda_1,\ldots,\lambda_n}$ are ordered. We omit writing $\mathbb{I}_{\{\lambda_1 \leq \ldots \leq \lambda_n \} }$ in the remaining sections as we work only with ordered eigenvalues.
The joint eigenvalue density function of a real, complex and quaternion Wishart matrix is same as the joint eigenvalue density function of a $(m,n,\beta)$-Laguerre matrix for the values of $\beta = 1,2$ and $4$, respectively.

It is important to acknowledge the fact that the Wishart ensemble have invariance properties. It means that a complex Wishart matrix is invariant under unitary conjugation, that is, a complex Wishart matrix $P_2$ has same distribution as $U^{*}P_2U$ for any non-random unitary matrix $U$. Similarly, a real Wishart matrix and a quaternion Wishart matrix is invariant under orthogonal and symplectic conjugation, respectively. However, while generalizing the Wishart matrix ensemble to $(m,n,\beta)$-Laguerre ensemble in order to have the equation~\eqref{eqn:jpdf_eigenvalues} well-defined for all positive values of $\beta$, we lose this invariance property. 

%%%%%%%%%%%%%%%%%%%%%%%%%%%%%%%%%%%%%%%%%%%%%%%%%%%%%%%%%%%%%%%%%%%%%%%%%%%%%%%%%%%%%%%
\chapter{Finiteness of Inverse Moments}
\label{chap:Finiteness of inverse moments}
%%%%%%%%%%%%%%%%%%%%%%%%%%%%%%%%%%%%%%%%%%%%%%%%%%%%%%%%%%%%%%%%%%%%%%%%%%%%%%%%%%%%%%%
We start this chapter by discussing the research problem and the motivation behind it. We compute the gap probability for the smallest eigenvalue of a $(m,n,\beta)$-Laguerre matrix and then, we present our results for $(m,n,\beta)$-Laguerre matrices and compound Wishart matrices followed by some remarks. 
%%%%%%%%%%%%%%%%%%%%%%%%%%%%%%%%%%%%%%%%%%%%%%%%%%%%%
\section{Motivation}
\label{subsec:motivation}
%%%%%%%%%%%%%%%%%%%%%%%%%%%%%%%%%%%%%%%%%%%%%%%%%%%%%
Here, we examine the formula for the inverse moments of complex Wishart matrices as given by Graczyk et al. in \cite{Graczyk_Letac_Massam_2003}. The moments and inverse moments of some random matrices, in particular, Wishart matrices can be expressed as sums of Weingarten functions, as can be seen in Theorem \ref{thm:moment_letac}.  Weingarten functions were first introduced by Don Weingarten in 1978 and further studied in depth by Collins \cite{Collins_2003}. Weingarten functions are used to compute the integrals of product of matrix coefficients over unitary groups with respect to Haar measure . This section has been adapted from \cite{Collins_Matsumoto_Saad_2014}, we refer to \cite{Collins_Matsumoto_Saad_2014} for detailed understanding of the technical terms.  

Let $c$ be a positive integer. Consider a set of $l$ positive integers $\eta = (\eta_1,\ldots,\eta_l)$ such that $\sum_{i=1}^{l} \eta_{i} = c$ then $\eta$ is called a partition of $c$. Let $\mathcal{S}_c$ be the symmetric group defined on $[c]=\{1,2,\dots,c\}$. Every permutation $\sigma \in \mathcal{S}_c$ can be decomposed uniquely into cycles of lengths $\eta = (\eta_1,\eta_2,\dots,\eta_l)$ such that $\eta$ is a partition of $c$ associated to $\sigma$. Let $p(\sigma)$ denote the length of vector $\eta$. The identity permutation in $\mathcal{S}_c$ is denoted by $\sigma_{e}$.  

For $z \in \mathbb{C}$ and a partition $\lambda = (\lambda_1,\ldots,\lambda_m)$ of $c$ define,
$$ {\psi}^{\lambda}(z) = \prod_{i=1}^{m} \prod_{j=1}^{\lambda_i} (z + j - i).  $$
Let $\chi^{\lambda}$ be the irreducible characters in $\mathcal{S}_c$. Given a complex number $z$ and a permutation $\sigma \in \mathcal{S}_c $, the unitary Weingarten function is,
\begin{align}
\label{eqn:weigngarten_function}
\displaystyle {\rm Wg}(\sigma,z) \ = \ \frac{1}{c!} \ \sum_{ \substack{\lambda \\ {\psi}^{\lambda}(z) \neq 0 }} \  \frac{\chi^{\lambda}(\sigma_{e})}{ {\psi}^{\lambda}(z)} \chi^{\lambda}(\sigma) ,
\end{align}
where the sum is over all partitions $\lambda$ of $c$ such that  ${\psi}^{\lambda}(z) \neq 0 $. The following result states the inverse moment formula for a complex Wishart matrix.
\begin{theorem}[\cite{Graczyk_Letac_Massam_2003}]
\label{thm:moment_letac}
Let $P$ be a $n \times n$ complex Wishart matrix and  $\pi = (12\ldots c)$ be a cycle in $\mathcal{S}_{c}$.  If $c < (m -n + 1)$, then
\begin{align*}
\displaystyle \mathbb{E}\{ { \rm Tr}(P_{2}^{-c}) \} = \ (-1)^{c} \sum_{\sigma \in \mathcal{S}_{c}} {\rm Wg}(\pi {\sigma}^{-1};n-m)n^{p(\sigma)},
\end{align*}
where ${ \rm Wg}(\pi \sigma^{-1};n-m) $ is the unitary Weingarten function as described in the equation \eqref{eqn:weigngarten_function}.
\end{theorem}
In Theorem \ref{thm:moment_letac}, a sufficient condition $c < (m-n+1)$ is assumed to define the $c$-th inverse moment of a complex Wishart matrix. A similar condition has been assumed in \cite{Collins_Matsumoto_Saad_2014,Letac_Massam_2004,Matsumoto_2012} to find the inverse moments of real, complex Wishart matrices and compound Wishart matrices. We wonder if this condition is necessary too. We aim to find a necessary and sufficient condition to have finite inverse moments for the bigger class of $(m,n,\beta)$-Laguerre matrices and hence the finiteness condition hold for Wishart matrices also.
\begin{description}
\item[Statement of problem:] given a $(m,n,\beta)$-Laguerre matrix $S$ and a compound Wishart matrix $Q$, find functions ${g_1(m,n,\beta)}$ and ${g_2(m,n,\beta)}$ such that
\begin{enumerate}[(i)]
\item 
\begin{center}
$\ { \displaystyle \mathbb{E}\{ {\rm Tr}(S^{-c})\} < \infty \text{ \ \text{ if and only if} \ } c < g_1(m,n,\beta) }.$
\end{center} 
\item \begin{center}
$\ { \displaystyle \mathbb{E}\{ {\rm Tr}(Q^{-c})\} < \infty \text{ \ \text{ if and only if} \ } c < g_2(m,n,\beta) }.$
\end{center} 
\end{enumerate}
\end{description}
We make a note in advance that the finiteness of inverse moments of a $(m,n,\beta)$-Laguerre matrix depends on the behavior of its smallest eigenvalue near zero. So, we first estimate the gap probability of the smallest eigenvalue of a $(m,n,\beta)$-Laguerre matrix in the following section.
%%%%%%%%%%%%%%%%%%%%%%%%%%%%%%%%%%%%%%%%%%%%%%%%%%%%%
\section{Gap probability near zero}
\label{subsec:gap_probability}
%%%%%%%%%%%%%%%%%%%%%%%%%%%%%%%%%%%%%%%%%%%%%%%%%%%%%
Let  $ 0 < \lambda_{1} \leq \dots \leq \lambda_n$ be the ordered eigenvalues of the $(m,n,\beta)$-Laguerre matrix $S$, hence $\lambda_1$ denote the smallest eigenvalue of $S$ in the remaining sections of part I of this thesis.
\begin{definition}[Gap probability \cite{Forrester_2010}] The term \emph{\lq gap probability near zero'} means the probability that no eigenvalue of a matrix lies in the neighborhood of zero. Let $a\in (0,1]$. Then,
\begin{align*}
\displaystyle \mathbb{P}(\text{no eigenvalues} \in (0,a)) \ = \ & \ \mathbb{P}(\lambda_1 \notin (0,a)) \ \  \ \\
\displaystyle 	\ = \ & \   1 -  \mathbb{P}(\lambda_1 < a).
\end{align*}
\end{definition}
We prove in the following that the probability of the smallest eigenvalue being less than $a$ is asymptotically equivalent to some constant times $a^{\alpha}$, where the constant depends only on $m,n$ and $\beta$.
\begin{lemma}
\label{lem:gap_prob}
Let $S$ be a $(m,n,\beta)$-Laguerre matrix. Then for $a \leq 1$, we have
\begin{align}
\mathbb{P}(\lambda_1 < a) \ \underset{a\rightarrow 0}{\approx} \ C^{\beta}_{m,n} a^{\alpha}, 
\end{align}
where $\alpha = (m-n+1)\beta/2$ and $C^{\beta}_{m,n}$ is a non-zero constant depending only on $m,n$ and $\beta$. In simpler words, when $a$ is small enough the probability of $\lambda_1$ belonging to the interval $(0,a)$ is asymptotically equivalent to  $C^{\beta}_{m,n} a^{\alpha}$ at the neighborhood of $0$.
\end{lemma}
\begin{proof}
From the equation \eqref{eqn:jpdf_eigenvalues}, we have $ \mathbb{P}(\lambda_1 < a)  = $
\begin{align*}
& Z_{m,n}^{\beta} \int_{0}^{a} \int_{\lambda_1}^{\infty} \ldots \int_{\lambda_{n-1}}^{\infty}\prod_{i=1}^{n}\left(\lambda_{i}^{\alpha -1} \right)\prod_{k<j}(\lambda_{j} - \lambda_{k})^{\beta} e^{\left(-\frac{1}{2}\sum_{i=1}^{n}\lambda_i\right)} d\lambda_{n} \ldots d\lambda_1.
\end{align*}
By the change of variables $(\lambda_1,\lambda_2,\dots,\lambda_n) = (a x_1,a x_1 + x_2,\dots,a x_1 + x_2 + \dots x_n)$, we have
\begin{align*}
\mathbb{P}(\lambda_1 < a) & = Z_{m,n}^{\beta} a^{\alpha}  \int_{0}^{1} \int_{0}^{\infty} \dots \int_{0}^{\infty}f_{a}(x_1,\ldots,x_n)dx_{n} \ldots dx_1, 
\end{align*}
where, 
\begin{align*}
f_{a}(x_1,\dots,x_n) & =  x_{1}^{\alpha-1}  e^{(-nx_1 a/2)} \prod_{i=2}^{n}\bigg(ax_1 + \sum_{j=2}^{i} x_j\bigg)^{\alpha-1}  \prod_{j=2}^{n} x_{j}^{\beta} e^{-(n-j+1)x_j /2}\\ & \ \ \ \ \prod_{k=3}^{n}\bigg( \prod_{i=1}^{k-2} \bigg( \sum_{j=i+1}^{k} x_j \bigg)^{\beta} \bigg) .
\end{align*}
The function $f_a$ has a point-wise limit, 
\begin{align*}
\lim_{a \rightarrow 0} f_{a}(x_1,\dots,x_n) & = x_{1}^{\alpha-1}\prod_{i=2}^{n}\left(\sum_{j=2}^{i} x_j\right)^{\alpha-1} \prod_{k=3}^{n}\left( \prod_{i=1}^{k-2} \left( \sum_{j=i+1}^{k} x_j \right)^{\beta} \right) \\ & \ \ \ \ \prod_{j=2}^{n} x_{j}^{\beta} e^{-(n-j+1)x_j/2}  \ \\
& := f(x_1,\dots,x_n).
\end{align*}
Now, we try to find a dominating function for $f_a$. By simple calculations and using the fact that $a \leq 1$, we have the following bounds for the expressions in the function $f_{a}(x_1,\dots,x_n)$.
\begin{enumerate}
\renewcommand{\labelenumi}{(\roman{enumi})}
\item $ 
\!
\begin{aligned}[t]
  \prod_{i=2}^{n}\left(ax_1 + \sum_{j=2}^{i} x_j\right)^{\alpha-1} & \leq \prod_{j=2}^{n} {x_j}^{-1} \ \prod_{i=2}^{n}\left(1 + \sum_{j=2}^{i} x_j\right)^{\alpha}  \ \\
 & \leq \prod_{j=2}^{n}{x_j}^{-1} \left(1 + \sum_{j=2}^{n} x_j\right)^{(n-1)\alpha}. 
\end{aligned}
$
\item $ 
\!
\begin{aligned}[t]
\prod_{k=3}^{n}\left( \prod_{i=1}^{k-2} \left( \sum_{j=i+1}^{k} x_j \right)^{\beta} \right) & \leq  \prod_{k=3}^{n}\left( \prod_{i=1}^{k-2} \left( 1 + \sum_{j=2}^{n} x_j \right)^{\beta} \right) \ \\
& \leq \left( 1 + \sum_{j=2}^{k} x_j \right)^{((n-1)(n-2)\beta /2) + (n-1)\alpha}.
\end{aligned}
 $
\end{enumerate}
From the above computations, we have an upper bound for $f_a$,
\begin{align}
\label{eqn:eq1}
 f_{a}(x_1,\dots,x_n) & \leq x_1^{\alpha -1} \prod_{j=2}^{n} {x_j}^{\beta -1}e^{-(n-j+1)x_j/2}  \left( 1 + \sum_{j=2}^{n} x_j \right)^{((n-1)(n-2)\beta /2) + (n-1)\alpha}.
\end{align}
For any positive real numbers $w_1,w_2$ and $p$, the following inequality holds
\begin{align*}
(w_1 + w_2)^{p} \ \leq \ 2^{p} (w_1^{p} + w_2^{p}). 
\end{align*}
Set $p \ = \ (n-1)\alpha + (n-1)(n-2)\beta /2 > 0$. Using the above inequality in the equation \eqref{eqn:eq1}, we obtain a dominating function for $f_a$,
\begin{align*}
g(x_1,\dots,x_n) \  = \ & \ {x_1}^{\alpha-1} \ \prod_{j=2}^{n} {x_j}^{\beta-1} e^{-(n-j+1)x_j /2} \bigg(2^{p} (1+x_2)^{p}  + \sum_{j=3}^{n} 2^{(j-2)p} {x_j}^{p} \bigg), 
\end{align*}
such that 
$$ f_{a}(x_1,\dots,x_n)  \ \leq \  g{(x_1,\dots,x_n)}. $$
The function $g$ is a finite sum of integrable functions and hence is integrable. By the Dominated Convergence Theorem, we have
\begin{align*}
& \int_{0}^{1} \int_{0}^{\infty} \dots \int_{0}^{\infty}f_{a}(x_1,\dots,x_n)dx_{n}\dots dx_1 \\ &  \underset{a\rightarrow 0}{\approx}  \int_{0}^{1} \int_{0}^{\infty} \dots \int_{0}^{\infty}f(x_1,\dots,x_n)dx_{n}\ldots dx_1 	\\
& =  \ z_{m,n}^{\beta}\ \  \text{ (is strictly positive by the positivity of $f$). } 
\end{align*} 
Hence, we have
\begin{align*}
\hspace{2.1cm} \mathbb{P}(\lambda_1  <  a) & \underset{a\rightarrow 0}{\approx}  \  a^{\alpha} Z_{m,n}^{\beta} z_{m,n}^{\beta}  \ \\
& \ = \ a^{\alpha} C^{\beta}_{m,n}.  
\end{align*}
\end{proof}

%%%%%%%%%%%%%%%%%%%%%%%%%%%%%%%%%%%%%%%%%%%%%%%%%%%%%
\section{A finiteness condition}
\label{sec:A finiteness condition}
%%%%%%%%%%%%%%%%%%%%%%%%%%%%%%%%%%%%%%%%%%%%%%%%%%%%%
Moments of a random matrix are useful in various theoretical and practical settings. It is always useful to know whether an expression is finite or not without explicitly having to compute it. Here, we present one such finiteness condition, with only three values $m,n$ and $\beta$ known beforehand, which tells the finiteness of the $c$-th inverse moment of a $(m,n,\beta)$-Laguerre matrix.  
\subsection{Preliminaries}
\label{subsec:Def_and_not}
In this part, we simply state the results (without proofs) which is needed to find the finiteness condition. The proof of all these lemmas can be found in standard texts.

Let $\mathbb{M}_{m,n}$ denote the space of $m\times n$ matrices with entries from either $\mathbb{R}$ or $\mathbb{C}$ which will be clear from the context. For $m=n$, we simply write $\mathbb{M}_{n}$ for $\mathbb{M}_{n,n}$. Let ${\rm Tr}(A)$ denote the un-normalized trace of the matrix $A \in \mathbb{M}_{n}$.

\begin{definition}[Loewner Partial order \cite{Horn_Johnson_1986}]
Let $A,B \in \mathbb{M}_m$. We write $A \preceq B$ if $A$ and $B$ are Hermitian matrices and $B-A$ is positive semi-definite. The relation $"\preceq "$ is a partial order, which is known as Loewner partial order.
\end{definition}
The following lemma helps in proving the necessity to have finite inverse moments for a compound Wishart matrix.
\begin{lemma}[{\cite[Theorem 7.7.2]{Horn_Johnson_1986}}]
\label{lem:partila_order_loewner}
Suppose $A,B$ are two $m \times m$ Hermitian matrices. Let $\sigma_1(A) \leq \sigma_2(A) \leq \dots \leq \sigma_m(A)$ and $\sigma_1(B) \leq \sigma_2(B) \leq \dots \leq \sigma_m(B)$ be the ordered eigenvalues of the matrices $A$ and $B$, respectively. If $A \preceq B$, then
\begin{enumerate}
\renewcommand{\labelenumi}{(\roman{enumi})}
\item $S^*AS \preceq S^*BS$ for $S \in \mathbb{M}_{m,n}$.
\item  $\sigma_i(A) \leq \sigma_{i}(B) $ for every $i = 1,\dots,m$.
\end{enumerate}
\end{lemma}
The following lemma for a positive random variable assists in finding the condition for the finiteness of inverse moments. 
\begin{lemma}[\cite{Ross_2006}]
\label{positive_rv}
Let $Z$ be a positive random variable and let $g(z)$ be a measurable function of $z$. If there is a non-negative real number $a$ such that $\mathbb{P}(Z \geq a) = 1$, then
\begin{align*}
\displaystyle \mathbb{E}\{g(Z)\} = g(a) + \int_{a}^{\infty} g'(z) \mathbb{P}(Z>z) dz.
\end{align*}
\end{lemma}
%%%%%%%%%%%%%%%%%%%%%%%%%%%%%%%%%%%%%%%%%%%%%%%%%%%%%
\subsection{Inverse moments of a $(m,n,\beta)$-Laguerre matrix}
%%%%%%%%%%%%%%%%%%%%%%%%%%%%%%%%%%%%%%%%%%%%%%%%%%%%%
Now, we present the necessary and sufficient condition for $(m,n,\beta)$-Laguerre matrices to have finite inverse moments.
\begin{theorem}
\label{thm:beta_moment}
Let  $\beta >0$ and let $S$ be a $n \times n$ $(m,n,\beta)$-Laguerre matrix. Then for integer $c > 0$, we have
\begin{align*}
\displaystyle \mathbb{E}\{{ \rm Tr }(S^{-c})\} \text{\  is finite if and only if \ } c < (m-n+1)\beta/{2}.
\end{align*}
\end{theorem}
\begin{proof}
For $c > 0$, we have
\begin{align*}
 {\lambda_{1}^{-c}} \ & \ \leq \ \sum_{i=1}^{n} {\lambda_{i}^{-c}} \  \leq  \  n{\lambda_{1}^{-c}},  \    
\end{align*}
so it follows that, 
\begin{align}
 \displaystyle   \hspace{2cm} \mathbb{E} \left\{ {\lambda_{1}^{-c}} \right\} \ & \ \leq  \  \mathbb{E} \left\{ Tr(S^{-c}) \right\}  \  \leq  \ n \mathbb{E}\left\{{\lambda_{1}^{-c}} \right\}.
\label{eqn:trace_ev}
\end{align}
From the equation \eqref{eqn:trace_ev}, we understand that $\mathbb{E}\{\allowbreak Tr(S^{-c} \} $ is finite if and only if $\mathbb{E} \left\{ {\lambda_{1}^{-c}} \right\} $ is finite. Thus, it is sufficient to find the necessary and sufficient condition for the finiteness of $\mathbb{E} \left\{ {\lambda_{1}^{-c}} \right\}$. 

Let $\delta >0 $ be small enough. By applying the Lemma \ref{positive_rv} to the positive random variable $\lambda_{1}^{-1}$, we get
\begin{align*}
\displaystyle  \mathbb{E} \left\{ {\lambda_{1}^{-c}} \right\} \ = \ & \ \int_{0}^{\infty} c t^{c-1} \mathbb{P}\left({\lambda_1^{-1}} > t \right) dt \ \\
\displaystyle \ = \ & \ \displaystyle  c \  \int_{0}^{\infty} t^{c-1} \mathbb{P}\left(\lambda_1 <	 {t^{-1}} \right)  dt \ \\
\displaystyle \ = \ & \ \displaystyle  c \ \int_{0}^{\infty}  w^{-c-1} \mathbb{P}\left(\lambda_1 <	w \right)  dw     \text{\ \ \ \ \ \ \ \ (put $1/t = w$)} \ \\
\displaystyle \ = \ & \ \displaystyle   c \  \int_{0}^{\delta} w^{-c-1} \mathbb{P}\left(\lambda_1 < w \right) dw +   c \ \int_{\delta}^{\infty}  w^{-c-1} \mathbb{P}\left(\lambda_1 < w \right) dw.
\end{align*}
So, we have
\label{eqn:bound}
\begin{enumerate}
\renewcommand{\labelenumi}{(\roman{enumi})}
\item 
$\displaystyle   c \ \int_{0}^{\delta} w^{-c-1} \mathbb{P}\left(\lambda_1 < w \right) dw \ \leq \ \mathbb{E} \left\{ {\lambda_{1}^{-c}} \right\}  $  and, 
\item
$ \displaystyle \mathbb{E} \left\{ {\lambda_{1}^{-c}} \right\} \  \leq \   c \ \int_{0}^{\delta} w^{-c-1} \mathbb{P}\left(\lambda_1 < w \right) dw \ + \   c \ \int_{\delta}^{\infty}  w^{-c-1} dw. $
\end{enumerate}
Consider the inequality (i), then by the Lemma \ref{lem:gap_prob} we have 
\begin{align*}
\displaystyle  \mathbb{E} \left\{ {\lambda_{1}^{-c}} \right\} \ \geq \ & \  c \int_{0}^{\delta}  w^{-c-1} \mathbb{P}\left(\lambda_1 < {w} \right) \ dw \ \\
\displaystyle  \approx \ & \ c \  C^{\beta}_{m,n}\int_{0}^{\delta} w^{-c-1} w^{\alpha} \ dw  \\
\displaystyle  = \ & \  \infty \ \ \ \text { whenever \ } (\alpha - c) \leq 0.
\end{align*}
Using the Lemma \ref{lem:gap_prob} to evaluate the inequality (ii), 
\begin{align*}
\displaystyle \mathbb{E} \left\{ {\lambda_{1}^{-c}} \right\} \ \leq \ & \  c \  \int_{0}^{\delta} w^{-c-1} \mathbb{P}\left(\lambda_1 < w \right) \ dw \ + \  c \ \int_{\delta}^{\infty}  w^{-c-1} \ dw \ \\
\displaystyle  \approx \ & \  c \ C^{\beta}_{m,n} \  \int_{0}^{\delta}  w^{-c-1} w^{\alpha} dw \ + \   c \ \int_{\delta}^{\infty}  w^{-c-1}  dw \ \\
\displaystyle \ < \ & \  \infty \ \ \ \text { whenever \ } {(\alpha - c )} > 0.
\end{align*}
This implies that $\displaystyle  \hspace{.1cm} \mathbb{E} \left\{ {\lambda_{1}^{-c}} \right\} $ is finite if and only if $c < \alpha$. Thus, $\mathbb{E}\left\{ Tr(S^{-c}) \right\}$ is finite if and only if $c < (m-n+1)\beta/2$.
\end{proof}

%%%%%%%%%%%%%%%%%%%%%%%%%%%%%%%%%%%%%%%%%%%%%%%%%%%%%
\subsection{Inverse moments of a compound Wishart matrix}
\label{subsec:Inverse moments of a compound Wishart matrix}
%%%%%%%%%%%%%%%%%%%%%%%%%%%%%%%%%%%%%%%%%%%%%%%%%%%%%
A simple but interesting consequence of the Theorem \ref{thm:beta_moment} is the following necessary and sufficient condition for the finiteness of inverse moments of compound Wishart matrices.
\begin{theorem}
\label{thm:compound_moment}
Let $Q$  be a $n \times n$ non-degenerate complex compound Wishart matrix. For $c > 0$,
\begin{align*}
\mathbb{E}\{{ \rm Tr}(Q^{-c})\} \text{\  is finite if and only if \ } c < (m-n+1).
\end{align*}
\end{theorem}
\begin{proof}
As, $Q$ is a $n \times n$ complex compound Wishart matrix, it follows from Definition \ref{def:compound_wishart} that $Q$ has the same distribution as $ A^*DA$, where $A$ is a complex random matrix with i.i.d. entries from a standard complex Gaussian distribution. We can understand $ A^{*}A = P_2$ as a $(m,n,\beta)$-Laguerre matrix $S$ for $\beta =2$. It follows from the Lemma \ref{lem:partila_order_loewner} that $\xi_1 S \preceq Q \preceq \xi_m S$ because $\xi_1 I \preceq D \preceq \xi_mI$, where $I$ is a  $m \times m$ identity matrix. Let $0 < \mu_1\leq \mu_2\leq \dots\leq \mu_n$ be the ordered eigenvalues of $Q$, it follows from part $(ii)$ of Lemma \ref{lem:partila_order_loewner},
\begin{align*}
\xi_1 \lambda_1 \ \leq \  \ \mu_1 \leq \ & \ \xi_{m} \lambda_1.
\end{align*}
This implies that for $c > 0$, we have
$$ \xi_1 \mathbb{E}\{\lambda_1^{-c}\} \leq \mathbb{E}\{\mu_1^{-c}\} \leq \ \ \xi_m \mathbb{E}\{\lambda_1^{-c} \}. $$
It follows from the proof of the Theorem \ref{thm:beta_moment} for $\beta =2$ that, $\mathbb{E}\{\mu_1^{-c}\}$ is finite if and only if $c < (m-n+1)$. Revisiting the initial part of the proof of the Theorem \ref{thm:beta_moment}, we know that the inverse moments of a random matrix with positive eigenvalues is finite if and only if the inverse moments of its smallest eigenvalue is finite. So, $\mathbb{E}\{{ \rm Tr}(Q^{-c})\}$ is finite if and only if $c < (m-n+1)$.
\end{proof}
We have a similar result for real compound Wishart matrix stated as the following remark. 
\begin{remark}
The $c$-th inverse moment of a real compound Wishart matrix, which is defined analogous to complex compound Wishart matrix as in Definition \ref{def:compound_wishart}, is finite if and only if $c < (m-n+1)/2$.
\demo \end{remark}
%%%%%%%%%%%%%%%%%%%%%%%%%%%%%%%%%%%%%%%%%%%%%%%%%%%%%
\section{Some remarks}
\label{sec:Some remarks}
%%%%%%%%%%%%%%%%%%%%%%%%%%%%%%%%%%%%%%%%%%%%%%%%%%%%%
The Lemma \ref{lem:gap_prob} investigates the probability of the smallest eigenvalue of a $(m,n,\beta)$-Laguerre matrix in the neighborhood of zero. Our main results, Theorem \ref{thm:beta_moment} and Theorem \ref{thm:compound_moment} give a necessary and sufficient condition to have finite inverse moments for the smallest eigenvalue of a $(m,n,\beta)$-Laguerre matrix and a compound Wishart matrix, respectively. These results may find their use where the inverse moments of the smallest eigenvalue are relevant. The existence of inverse moments for a $(m,n,\beta)$-Laguerre matrix is equivalent to the existence of inverse moments of the smallest eigenvalue of a $(m,n,\beta)$-Laguerre matrix.

We summarize our result for $(m,n,\beta)$-Laguerre matrix $S$ as:
 \emph{$\mathbb{E}\{Tr(S^{-c})\} \allowbreak < \infty $ if and only if $c < (m-n+1)\beta /2 $, that is, all the finite integer inverse moments of a $(m,n,\beta)$-Laguerre matrix lies in the interval $ \left(0, (m-n+1)\beta/2 \right)$. }
		
	We studied the compound Wishart matrix for the values $\beta =1$ and $2$. As, compound Wishart matrices are the generalization of Wishart matrices, so the result extend naturally to compound Wishart matrices. The $c$-th inverse moment of a compound Wishart matrix exists if and only if $c < (m-n+1)\beta /2 $ and thus, we obtain that all the finite integer inverse moments lies in $\displaystyle \left(0, (m-n+1)\beta/2 \right)$. 
	
	Recently, the inverse moments of $(m,n,\beta)$-Laguerre matrices has been studied in \cite{Mezzadri_Reynolds_Winn_2017}. Our results are consistent and complete with the other results on the inverse moments of $(m,n,\beta)$-Laguerre matrices and compound Wishart matrices as in \cite{Collins_Matsumoto_Saad_2014,Letac_Massam_2004,Matsumoto_2012,Mezzadri_Reynolds_Winn_2017}. We expect that our results can be extended to more general matrix models involving Wishart matrices and leave it for the future work.  
\begin{remark}
In the general $\beta$ case, there is no well defined notion of a compound Wishart matrix, which explains why we focused on compound Wishart case for the values of $\beta = 1$ and $2$.
\demo \end{remark} 

%%%%%%%%%%%%%%%%%%%%%%%%%%%%%%%%%%%%%%%%%%%%%%%%%%%%%%%%%%%%%%%%%%%%%%%%%%%%%%%%%%%%%%%
%%%%%%%%%%%%%%%%%% Part II Universal consistency of kNN rule %%%%%%%%%%%%%%%%%%%%%%%%%%
\part{Universal Consistency of the $k$-Nearest Neighbor Rule}
%%%%%%%%%%%%%%%%%%%%%%%%%%%%%%%%%%%%%%%%%%%%%%%%%%%%%%%%%%%%%%%%%%%%%%%%%%%%%%%%%%%%%%%
\chapter{Dimension of a Metric}
\label{chap:Dimension of a metric}
%%%%%%%%%%%%%%%%%%%%%%%%%%%%%%%%%%%%%%%%%%%%%%%%%%%%%%%%%%%%%%%%%%%%%%%%%%%%%%%%%%%%%%%
In this chapter, we discuss about the various properties of dimension of a metric space in the sense of Nagata and Preiss. We also give a detailed proof of the necessity part of the Preiss' result. 
%%%%%%%%%%%%%%%%%%%%%%%%%%%%%%%%%%%%%%%%%%%%%%%%%%%%%
\section{Sigma-finite metric dimension}
\label{sec:Sigma-finite metric dimension}
%%%%%%%%%%%%%%%%%%%%%%%%%%%%%%%%%%%%%%%%%%%%%%%%%%%%%
David Preiss introduced the notion of sigma-finite metric dimension in his article \cite{Preiss_1983} in order to describe the metric spaces having strong Lebesgue-Besicovitch differentiation property. Let $\beta \geq 1$ denotes an integer and let $(\Omega,\rho)$ be a metric space throughout the remaining sections of this thesis.       
\begin{definition}[\cite{Preiss_1983}]
\label{sigma_finite}
A subset $Q$ of $\Omega$ has \emph{metric dimension} $\beta$ on a scale $s \in(0,+\infty) $ in $\Omega$ if any finite set $F = \{ x_1,\ldots,x_m \} \subseteq Q$, $m > \beta$ and  $r_1,\ldots,r_m \in (0,s)$, where each $r_i$ depends on $x_i$ to satisfy the condition that for $i \neq j $, $x_i \notin \bar{B}(x_j,r_j)$ and  $x_j \notin \bar{B}(x_i,r_i)$, implies that for every $x \in \Omega$
\begin{align*}
\sum_{x_i \in F} \chi_{_{\bar{B}(x_i,r_i)}}(x) \leq \beta.
\end{align*}
\end{definition}
Generally, we consider families of closed balls satisfying certain property stated as the following definition. 
\begin{definition}[Unconnected family \cite{Assouad_Gromard_2006}]
Let $I$ be an index set and $\mathcal{F}= \{ \bar{B}(x_i,r_i): i \in I\}$ be a any family of closed balls. Then, $\mathcal{F}$ is called an \emph{unconnected} family if for every $i\neq j$ in $I$, $x_i$ does not belong to $\bar{B}(x_j,r_j)$ and vice-versa. In simpler words, distance between $x_i$ and $x_j$ is strictly greater than $r_i$ and $r_j$.
\end{definition}
%%%%%%%%%%%%%%%%%%%%%%%%%%%%%%%%%%%%%%%%%%%%%%%%%%%%%
\vspace{0.6cm}
\begin{figure}
\centering
\begin{tikzpicture}
%\draw[<->] (30,0) -- (38,0);
\draw (35,0) circle (1cm);
\fill (35,0) circle (2pt);
\draw (35,1.2) circle (1cm);
\fill (35,1.2) circle (2pt);
\draw (36.1,0.5) circle (1cm);
\fill (36.1,0.5) circle (2pt);
\draw (34,0.6) circle (1cm);
\fill (34,0.6) circle (2pt);
\end{tikzpicture}
\caption{An unconnected family means that each center belongs to exactly one ball.} 
\label{fig:unconnected}
\end{figure}
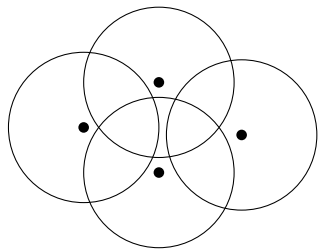
\vspace{0.6cm}

\vspace{0.6cm}
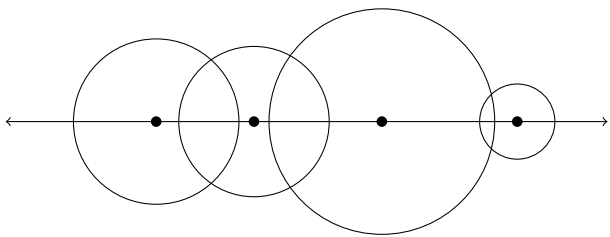
\begin{figure}
\centering
\begin{tikzpicture}
\draw[<->] (30,0) -- (38,0);
\draw (32,0) circle (1.1cm);
\fill (32,0) circle (2pt);
\draw (33.3,0) circle (1cm);
\fill (33.3,0) circle (2pt);
\draw (35,0) circle (1.5cm);
\fill (35,0) circle (2pt);
\draw (36.8,0) circle (0.5cm);
\fill (36.8,0) circle (2pt);
\end{tikzpicture}
\caption{The real line has metric dimension 2 on any scale as any real number can belong to at most 2 balls in a unconnected family of balls in $\mathbb{R}$.} 
\label{fig:finite_dimension}
\end{figure}
\vspace{0.6cm}
%%%%%%%%%%%%%%%%%%%%%%%%%%%%%%%%%%%%%%%%%%%%%%%%%%%%%
In other words, $Q$ has metric dimension $\beta$ on scale $s$ if every element of $\Omega$ can belong to at most $\beta$ closed balls in an unconnected family of closed balls having centers in $Q$ and radius bounded by scale $s$. The value $\beta$ is the smallest possible integer and $s$ is the largest possible positive real number satisfying the above property. A metric space is called metrically finite dimensional if there is a pair of $\beta < \infty$ and $0< s < \infty$ satisfying Definition~\ref{sigma_finite}, or sometimes we simply say, a space has metric dimension $\beta$ on scale $s$. The real line has metric dimension 2 as illustrated in figure \ref{fig:finite_dimension}.
\begin{definition}[Sigma-finite metric dimension]
The space $(\Omega,\rho)$ is {\emph{metrically sigma-finite dimensional}} if there is a sequence of subsets $Q_1,Q_2,\ldots$ of $\Omega$ such that $\Omega = \bigcup_{i=1}^{\infty} Q_{i}$ and each $Q_i$ has finite metric dimension in $\Omega$.
\end{definition}
Note that in the above definition, it is possible that each $Q_i$ has different metric dimension. The definition of metric dimension trivially implies that any metric space with finite number of elements has metric dimension equal to its cardinality on any scale. Therefore, every metric space with countable number of elements has sigma-finite dimension. Similarly, a metric space with 0-1 metric (refer to \ref{app:discrete_metric}) has metric dimension 1 on scale $s \in (0,1]$.  Moreover, a finite union of metric spaces having finite metric dimension also has finite metric dimension.
\begin{lemma}
\label{lem:finite_union_fd}
Suppose that $Q_1, Q_2 \subseteq \Omega$ have metric dimension $\beta_1$ and $\beta_2$ on scale $s_1$ and $s_2$ in $\Omega$, respectively. Then $Q_1 \cup Q_2$ has metric dimension $ \beta_1 + \beta_2 $ on scale $s = \min\{s_1,s_2\}$ in $\Omega$. 
\end{lemma}
\begin{proof}
Consider a finite set $F = F_1 \cup F_2 \subseteq Q_1 \cup Q_2$ of cardinality $ > \beta_1 + \beta_2$, where $F_1 \subseteq Q_1$  and $F_2 \subseteq Q_2$. It implies that either $F_1$ has cardinality > $\beta_1$ or $F_2$ has cardinality $> \beta_2$. As, $s$ is the minimum of $s_1,s_2$ and $Q_1,Q_2$ are metrically finite dimensional, it follows from the Definition~\ref{sigma_finite} that $Q_1 \cup Q_2$ has metric dimension $\beta_1 + \beta_2$ on scale $s$.
\end{proof}
However, the countable union of metrically finite dimensional spaces may not have finite metric dimension and which is why we need the concept of sigma-finite metric dimension. The following is an example of a metric space which does not have finite metric dimension but is metrically sigma-finite dimensional. 
\begin{example}
Let $\Omega = \{ x_1,x_2,x_3,\ldots \}$ and define a function on $\Omega$, 
\begin{align*}
\rho(x_n,x_m) \ = \ 
\begin{cases}
1/n +1/m \ & \text{ if } n \neq m \ \\
0 \ & \text{   otherwise }
\end{cases}
\end{align*}
Then $\rho$ is a metric as $\rho(x_n,x_m) = 0$ if and only if $n = m$. Let $Q_l = \{x_1,\ldots,x_l\}$ and $\Omega$ is the union of all $Q_l, l\in \mathbb{N}$. For a fixed $l$, $Q_l$ is a finite set and so has metric dimension $l$ on any scale $s'>0$ in $\Omega$. 

Let  $\beta \geq 1$ be a integer and $s$ be a positive real number. Choose $l \in \mathbb{N}$ large enough that $s > 1/l $ and pick $\beta + 1$ elements, $F = \{ x_{l+1},\ldots,x_{l+\beta+1} \}$ from $\Omega$. For $l+1\leq i \leq {l+\beta+1}$, set $r_i = \rho(x_i,x_{l+\beta+2})$ then we have $\rho(x_i,x_{j}) > \max\{r_i,r_j\}$. This means the any closed balls in $\mathcal{F} = \{ \bar{B}(x_i,r_i): x_i \in F \}$ does not contain the center of any other ball in $\mathcal{F}$, hence $\mathcal{F}$ is an unconnected family. However, $x_{l+\beta+2}$ belongs to every ball in $\mathcal{F}$ and has multiplicity $\beta+1$ in $\mathcal{F}$. Therefore, $\Omega$ does not have finite metric dimension but it is the countable union of such spaces, hence $\Omega$ is metrically sigma-finite dimensional. \demo 
\end{example}
We make the following important remark about subsets inheriting the metric dimension.
\begin{remark}
\label{rem:subset_fd}
Let $(\Omega,\rho)$ be a metric space and suppose $Q_1\subseteq Q_2 \subseteq \Omega$. If $Q_2$ is metrically finite dimensional in $\Omega$, then $Q_1$ is also metrically finite dimensional in $\Omega$. Given a family of subsets $\{Q_i:i \in I\}$ of $\Omega$ such that at least one of the $Q_i$ is metrically finite dimensional in $\Omega$. Then, $\cap_{i \in I} Q_i$ is also metrically finite dimensional in $\Omega$. \demo 
\end{remark}
Furthermore, we show that a set has finite metric dimension if and only if its closure has finite metric dimension. 
\begin{lemma}
\label{lem:fd_closed}
Let $Q \subseteq \Omega$ has metric dimension $\beta$ on scale $s$ in $\Omega$, then $\bar{Q}$ has metric dimension $\beta$ on scale $s$ in $\Omega$.
\end{lemma}
\begin{proof}
Let $F = \{x_1,\ldots,x_m\} \subseteq \bar{Q}$ such that $m>\beta$ and let $\{r_1,\ldots,r_m\} \subseteq (0,s)$ such that for every distinct $1\leq i,j \leq m$, $\rho(x_i,x_j)>\max\{r_i,r_j\}$. Then $\mathcal{F} = \{\bar{B}(x_i,r_i): r_i <s, 1 \leq i \leq m\}$ is an unconnected family of closed balls. It is sufficient to show that the multiplicity of $\mathcal{F}$ is at most $\beta$.

Let $\rho(x_1,x_j) = k_j > \max\{r_1, r_j\}$. Choose $\varepsilon_1 > 0$ such that for all $j \neq 1$, $k_j - 2\varepsilon_1 > \max(r_1, r_j)$ and $r_1 + \varepsilon_1<s$.
Let $y_1 \in Q$ such that $\rho(y_1, x_1) < \varepsilon_1$.
It is clear that $\rho(x_j, \bar{B}(x_1,\varepsilon_1)) = k_j - \varepsilon_1$ and so $\rho(x_j, y_1) \geq k_j - \varepsilon_1 > r_1 + \varepsilon_1$ for all $j \neq 1$. 
Thus, we have $\bar{B}(y_1, r_1 + \varepsilon_1)$, which contains $\bar{B}(x_1,r_1)$ but not any other $x_j$, $j \neq 1$.
Moreover, note that $y_1 \notin \bar{B}(x_j,r_j)$ for all $j \neq 1$.

Replace $x_1$ and $r_1$ by $y_1$ and $r_1 + \varepsilon_1$ in the original set $F$ and form a new set $F'$, and apply the same technique for $x_2$ but for $F'$. Then, we have $\bar{B}(y_2, r_2 + \varepsilon_2)$, where $y_2 \in Q$ and $r_2 + \varepsilon_2 < s$ with $\varepsilon_2 > 0$, such that $\bar{B}(x_2,r_2)$ is contained in it and none of the points in $F \setminus \{x_2\}$ belongs in the ball. 

Doing in the same way for all the $m-2$ points $\{x_3, x_4, \ldots, x_m\}$, we obtain a new family of closed balls $\mathcal{F'} = \{\bar{B}(y_i,r_i + \varepsilon_i): 1\leq i \leq m\}$, where 
$y_i \in Q$ and $r_i + \varepsilon_i < s$ with $\varepsilon_i > 0$ for all $1\leq i \leq m$, such that $y_j \notin \bar{B}(y_i,r_i +\varepsilon_i)$ whenever $i \neq j$ and 
$\bar{B}(x_i, r_i) \subseteq \bar{B}(y_i,r_i + \varepsilon_i)$ for all $1\leq i \leq m$. 

Since $Q$ has metric dimension $\beta$ on scale $s$, then $x \in \mathcal{F}$ belongs to at most $\beta$ balls in $\mathcal{F'}$. 
\end{proof}
An interesting result is that $Q$ is metrically finite dimensional in itself but there is a super-set $\Omega$ which contains $Q$ but $Q$ does not have finite metric dimension in $\Omega$.
\begin{example}
Let  $(Q \cup \{a^*\},\rho')$ be a metric space, where $a^*$ is not an element of $Q$, such that the distance $\rho'(a,b)$ is 1, if $a,b$ are distinct elements of $Q$, $1/2$ if only one of $a,b$ is $a^*$ and 0 otherwise.  Now, we consider disjoint copies of $Q \cup \{a^*\}$. For  $n \in \mathbb{N}$, set $Q_n = (Q \cup \{a^*\},n)$ and let the metric on $Q_n$ be $\rho_n$. The distance between any two elements $a_n = (a,n), b_n = (b,n)$ of $Q_n$ is $ \rho(a,b)/ n$ where $a,b \in Q \cup \{a^*\}$. Since $\rho'$ is a metric so $\rho_n$ is metric for each $n$.

Let $\Omega = \cup_{n \in \mathbb{N}} Q_n$ and define a function $\rho$ on $\Omega$, for $a_n,b_m \in \Omega$,  $\rho(a_n,b_m) = \rho_n(a,b)$ if $n = m$, otherwise 1. 
We have $\rho(a_n,b_m) \Leftrightarrow \rho_{n}(a,b) = 0 \Leftrightarrow \rho(a,b) = 0 $, and for $a_n,b_m,c_{p} \in \Omega$, $\rho(a_n,c_p) + \rho(c_p,b_m) \geq \rho(a_n,b_m)$. Therefore, $\rho$ is a metric. It is evident that $Q$  has metric dimension $1$ on any scale $s \in (0,1]$ in $Q$ with respect to $\rho'$. 

\textit{$Q$ is not metrically finite dimensional in $\Omega$}: Let $\beta$ be any positive integer and suppose  that $s \in (0,1]$. We choose a $m \in \mathbb{N}$ large enough so that $1/m < s$, and consider  a family of closed balls in $Q_m$, $\mathcal{F} = \{ \bar{B}(a^i_{m},1/2m) : a^i_m = (a^i,m) \in Q_m, a^i \in Q, 1 \leq i \leq \beta +1\}$. For $i \neq j$, $\rho(a^i_m,a^j_m) = \rho(a^i, a^j)/m = 1/m$, so any two balls in $\mathcal{F}$ with distinct centers does not contain each other's center and $\mathcal{F}$ form an unconnected family. The closed balls are actually in $\Omega$, like $\bar{B}(a^i_m,1/2m) = \{ b_n \in \Omega : \rho(a^i_m,b_n) \leq 1/2m < 1/m < s \leq 1 \}$. As the radius of every ball in $\mathcal{F}$ is less than equal to $1/2m < 1$, this implies that $b_n$ is in $Q_m$ and $\bar{B}(a^i_m,1/2m) = \{ a^i_m,a^{*}\}$ for $1 \leq i \leq \beta+1$. 

It means that the multiplicity of $a^*$ is more than $\beta$ in $\mathcal{F}$ and because centers of balls in $\mathcal{F}$ are in $Q$, so $Q$ is not metrically finite dimensional in $\Omega$. For the other case when $s > 1$, we choose $m > 1$ large enough such that $1/m < s$. The rest of the argument is same as for $s \leq 1$.
\demo \end{example}
The following characterization for a metrically sigma-finite dimensional space is important and so we state it as a remark.
\begin{remark}
\label{rem:fd_closed}
If $\Omega = \cup_{i=1}^{\infty} A_i$ such that each $A_i$ has finite metric dimension $\beta_i$ on scale $s_i$.  Let $Q_l = \cup_{i=1}^{l} A_{i}$, and the sets $\{Q_l\}_{l \in \mathbb{N}}$ form an increasing chain. Since for a fixed $l$, $Q_{l}$ is a finite union of finite metric dimensional space, due to the Lemma \ref{lem:finite_union_fd}, $Q_l$ has metric dimension $\beta_1 + \ldots + \beta_l$ on scale $s = \min\{s_i : 1 \leq i \leq l\}$ in $\Omega$. Without loss of generality, we can assume each $Q_l$ is closed (because of Lemma \ref{lem:fd_closed}). Therefore, $\Omega = \cup_{l=1}^{\infty} Q_l$, where each $Q_l$ is closed and has finite metric dimension.
\demo \end{remark}
Now, we present an important result about complete metric spaces. Every metrically sigma-finite dimensional complete metric space contains a non-empty open set which is metrically finite dimensional. 
\begin{proposition}
\label{prop:subset_open}
Suppose $(\Omega,\rho)$ is a complete metric space and is metrically sigma-finite dimensional. Then there is a non-empty open set which is metrically finite dimensional in $\Omega$. 
\end{proposition}
\begin{proof}
There is a sequence of metrically finite dimensional subsets $(Q_i)_{i \in \mathbb{N}}$ such that $\Omega = \cup_{i \in \mathbb{N}} Q_i$. Due to the Lemma \ref{lem:fd_closed}, we can assume each $Q_i $ is closed. 

 It follows from the Baire Category Theorem \cite{Engelking_1989} that at least one of the $Q_i$ has non-empty interior. So there is an non-empty open ball contained in $Q_i$, which is metrically finite dimensional in $\Omega$ (follows from the Remark \ref{rem:subset_fd}).
\end{proof}
Assouad and Gromard \cite{Assouad_Gromard_2006} have studied the dimension of a metric in more general environment, for example, they considered families of balls (open or closed) in a semimetric space \footnote{a semimetric is a distance function which satisfy every axioms of a metric but not necessarily the triangle's inequality}. A related yet somewhat different notion called, Nagata dimension, is presented in the following section.
%%%%%%%%%%%%%%%%%%%%%%%%%%%%%%%%%%%%%%%%%%%%%%%%%%%%%
%\vspace{0.6cm}
\begin{figure}
\centering
\begin{tikzpicture}
\draw[<->] (30,0) -- (32,0)  node[sloped,below] {$x_1$} -- (33.3,0) node[sloped,below] {$a$} -- (35,0) node[sloped,below] {$x_2$} -- (36.8,0)node[sloped,below] {$x_3$} -- (38,0);
%%\draw (32,0) circle (1.1cm);
\fill (32,0) circle (2pt);
%%\draw (33.3,0) circle (1cm);
\fill (33.3,0) circle (2pt);
%%\draw (35,0) circle (1.5cm);
\fill (35,0) circle (2pt);
%%\draw (36.8,0) circle (0.5cm);
\fill (36.8,0) circle (2pt);
\draw[|<->|] (35,0.5) -- (36.8,0.5);
\draw[|<->|] (33.3,0.9) -- (36.8,0.9);
\end{tikzpicture}
\caption{The real line has Nagata dimension 1, for every four reals, there is a pair $(x_2,x_3)$ such that $\rho(x_2,x_3) \leq \max\{\rho(a,x_2),\rho(a,x_3)\}$.} 
\label{fig:nagata_dimension}
\end{figure}
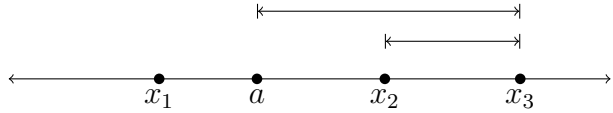
%\vspace{0.6cm}
%%%%%%%%%%%%%%%%%%%%%%%%%%%%%%%%%%%%%%%%%%%%%%%%%%%%%
\section{Nagata dimension}
\label{sec:Nagata dimension}
%%%%%%%%%%%%%%%%%%%%%%%%%%%%%%%%%%%%%%%%%%%%%%%%%%%%%
This section is based on an important paper by Assouad and Gromard \cite{Assouad_Gromard_2006}. The essence of their paper is the generalization of many dimension related concepts to general metric spaces. As, the paper is in French language so it is difficult for researchers not knowing french language to absorb their interesting results. Here, we present an English translation of some of the dimension related concepts picked from their article. In general, Nagata dimension is defined somewhat differently in dimension theory, but we tend to follow the same notions as in paper by Assouad and Gromard.  

We also want to state that the Nagata dimension was first introduced by Nagata and then it was modified by Preiss to define metric dimension by introducing scales $s$. In this thesis, our major focus is finite metric dimension in the sense of Preiss, as it will clear after this section that a metric space with finite Nagata dimension has finite metric dimension for all scales but the converse statement is not true. 
\begin{definition}[Nagata dimension \cite{Assouad_Gromard_2006}]
\label{def:nagata_dim}
A subset $Q$ of $\Omega$ has {\emph{Nagata dimension}} $\beta-1 \geq 0$ in $\Omega$ if for every set of $\beta +1$ elements $x_1,\ldots,x_{\beta+1}$ in $Q$ and every $a \in \Omega$, there exist distinct $i,j$ in $\{1,\ldots,\beta+1\}$ such that $\rho(x_i,x_j) \leq \max\{ \rho(a,x_i),\rho(a,x_j) \}$.
\end{definition}
A space with 0-1 metric has Nagata dimension zero whereas, it has metric dimension 1. The real line with usual metric has Nagata dimension one (see figure \ref{fig:nagata_dimension}). The following proposition follows immediately from the definition of Nagata dimension. 
\begin{proposition}
\label{cor:archi_Nagata}
An ultrametric space has Nagata dimension zero.
\end{proposition}
\begin{proof} Let $(\Omega,\rho)$ be an ultrametric space and let $Q \subseteq \Omega$. Then by the strong triangle’s inequality of $\rho$, for any $x_1,x_2 \in Q$ and $a \in \Omega$, we have $\rho(x_1,x_2) \leq \max\{ \rho(a,x_1),\rho(a,x_2) \}$. 
\end{proof}
%%%%%%%%%%%%%%%%%%%%%%%%%%%%%%%%%%%%%%%%%%%%%%%%%%%%%
%\vspace{0.6cm}
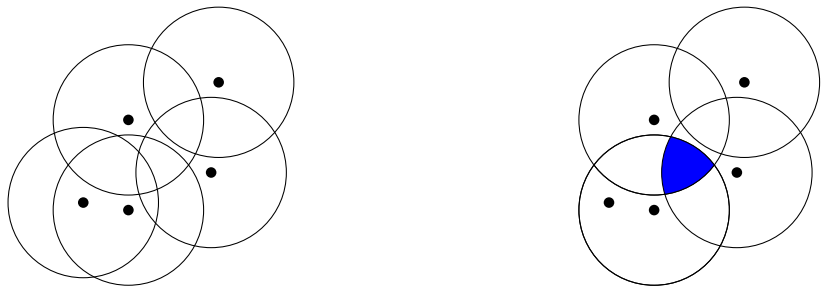
\begin{figure}
\centering
\begin{tikzpicture}
\draw (35,0) circle (1cm);
\fill (35,0) circle (2pt);
\draw (35,1.2) circle (1cm);
\fill (35,1.2) circle (2pt);
\draw (36.1,0.5) circle (1cm);
\fill (36.1,0.5) circle (2pt);
\draw (34.4,0.1) circle (1cm);
\fill (34.4,0.1) circle (2pt);
\draw (36.2,1.7) circle (1cm);
\fill (36.2,1.7) circle (2pt);
\end{tikzpicture}
\hspace*{3.5cm}
\begin{tikzpicture}
\begin{scope}
  \draw[clip] (35,0) circle (1cm);
  \draw[clip] (35,1.2) circle (1cm);
	\fill[blue] (36.1,0.5) circle (1cm);	
\end{scope}
\draw (35,0) circle (1cm);
\fill (35,0) circle (2pt);
\draw (35,1.2) circle (1cm);
\fill (35,1.2) circle (2pt);
\draw (36.1,0.5) circle (1cm);
\fill (36.1,0.5) circle (2pt);
\fill (34.4,0.1) circle (2pt);
\draw (36.2,1.7) circle (1cm);
\fill (36.2,1.7) circle (2pt);
\end{tikzpicture}
\caption{Ball-covering dimension is 3 as the subfamily has multiplicity at most 3. We cannot extract a subfamily from the original family such that it contains every centers from the original family and has multiplicity less than 3.} 
\label{fig:ball-covering_dimension}
\end{figure}
%\vspace{0.6cm}
%%%%%%%%%%%%%%%%%%%%%%%%%%%%%%%%%%%%%%%%%%%%%%%%%%%%%
We say $\Omega$ has \emph{sigma-finite Nagata dimension}, if $\Omega$ can be written as countable union of subsets having finite Nagata dimension in $\Omega$. Another interesting notion for a family of balls is the ball-covering dimension. 
\begin{definition}[Ball-covering dimension]
A subset $Q$ of $\Omega$ is said to have {\emph{ball-covering dimension}} $\beta$ in $\Omega$, if for every countable family $\mathcal{F}$ of closed balls with centers in $Q$, there exists a subfamily $\mathcal{F}' \subseteq \mathcal{F}$ such that center of every ball in $\mathcal{F}$ belongs to some ball in $\mathcal{F}'$ and the multiplicity of $\mathcal{F}'$ is at most $\beta$ in $\Omega$.
\end{definition} 
The ball-covering dimension of a family of balls is shown in figure \ref{fig:ball-covering_dimension}. Assouad and Gromard proved that Nagata dimension + one is equal to ball-covering dimension of the space.
\begin{lemma}[\cite{Assouad_Gromard_2006}]
\label{lem:wcp_nd}
The following are equivalent:
\begin{enumerate}[(i)]
\item $Q$ has Nagata dimension $\beta-1$ in $\Omega$.
\item An unconnected family of closed balls with centers in $Q$ have multiplicity at most $\beta$ in $\Omega$.
\item $Q$ has ball-covering dimension $\beta$ in $\Omega$.
\end{enumerate}
\end{lemma}
\begin{proof} $(i) \Rightarrow (ii)$: Let $I$ be an index set and consider an unconnected family of balls, $\mathcal{F} = \{ \bar{B}(a_i,r_i):a_i \in Q, i \in I\}$ such that for $i \neq j$, $a_i \notin \bar{B}(a_j,r_j)$ and $a_j \notin \bar{B}(a_i,r_i)$. Let $a \in \Omega$ and suppose that every ball in $\{ \bar{B}(a_i,r_i):a_i \in Q, i \in J \subseteq I\}$ contains $a$. Since the balls are from unconnected family, $\max\{\rho(a_i,a),\rho(a_j,a) \}< \rho(a_i,a_j)$, but $Q$ has Nagata dimension $\beta-1$ which implies that $\sharp J \leq \beta$. Therefore, $\mathcal{F}$ has multiplicity at most $\beta$ in $\Omega$.

$(ii) \Rightarrow (iii)$: Consider a family of closed balls $\mathcal{F} = \{ \bar{B}(a_i,r_i):a_i \in Q, i \in \mathbb{N}\}$. Choose $I_1$ as the maximal set of all indices in $\mathbb{N}$ such that $\rho(a_i,a_j) > \max\{r_i,r_j\}$ for all $i \neq j$ in $I_1$ and let $\mathcal{F}_1=\{\bar{B}(a_i,r_i):i\in I_1\}$. Again choose $I_2$ to be the maximal set of all indices $i \in \mathbb{N} \setminus I_1$
 such that $a_i \notin \mathcal{F}_1$ and for all $i \neq j$ in $I_2, \rho(a_i,a_j) > \max\{r_i,r_j\}$ and let $\mathcal{F}_2=\{\bar{B}(a_i,r_i):i\in I_2\}$. Continuing these steps until we exhaust $\mathcal{F}$ gives $I = \cup_{p=1}^{\infty}I_{p}$  and $\mathcal{F}'= \cup_{i \in I} \bar{B}(a_i,r_i)$ such that, every center of balls in $\mathcal{F}$ is in $\mathcal{F}'$. Since $\mathcal{F}'$ is an unconnected family, so by $(ii)$, any $x \in \Omega$ belongs to at most $\beta$ balls in $\mathcal{F}'$.

$(iii) \Rightarrow (i)$
Suppose $Q$ does not have Nagata dimension $\beta -1$ in $\Omega$. Then, there exists a set of $\beta+1$ elements, $x_1,\ldots,x_{\beta+1} \in Q$ and $a \in \Omega$ such that for every $i \neq j$, $\rho(x_i,x_j) > \max\{\rho(a,x_i),\rho(a,x_j)\}$. This implies that $\mathcal{F}=\{\bar{B}(x_i,\rho(x_i,a)):1\leq i \leq \beta+1\}$ is an unconnected family of closed balls and hence we cannot extract a proper sub-family containing every $x_i$. Also, $\mathcal{F}$ has multiplicity more than $\beta$. So, $Q$ cannot have ball-covering dimension $\beta$.
\end{proof}
The Lemma \ref{lem:wcp_nd} gives an impression that the metric dimension in the sense of Preiss and Nagata dimension seems to be interrelated. To find metric dimension, we consider an unconnected family of closed balls but with radius of each ball bounded by scale $s$, whereas to find ball-covering dimension or Nagata dimension there is no restriction on radius of a ball. Because of scale $s$, there is a subtle difference between metric dimension and Nagata dimension. Let us look at an example.  
\begin{example}
\label{ex:not_nagata_fd}
Let $\mathbb{R}$ be equipped with metric,
\begin{align*}
\rho(x,y) \ = \ 
\begin{cases}
 0  \ & \ \text{if \ \ } x = y \ \\
1/2 \ & \ \text{if \ \ } x \neq y \ \& \  xy = 0\ \\
1 \ & \ \text{otherwise} \ 
\end{cases}
\end{align*}
Let $Q = (0,1)$. Then
\begin{enumerate}[(i)]
\item \textit{$Q$ has infinite Nagata dimension in $\mathbb{R}$.} Let $a = 0$ and $x_i = 3^{-i}$ for $i \in \mathbb{N}$. For all $i \neq j$, $\rho(x_i,x_j) = 1 > \max\{\rho(0,x_i), \rho(0,x_j)\} = 1/2$. Then there is no $\beta-1$ for which $\rho(x_i,x_j) \leq \max\{\rho(0,x_i), \rho(0,x_j)\}$ and so, $Q$ has infinite Nagata dimension.
\item \textit{$Q$ has finite metric dimension in $\mathbb{R}$.} For $\beta = 1$ and any $s \in (0,1/2]$, there are finite sets $F \subseteq Q$ with cardinality $>1$ such that if we consider an unconnected family of closed balls $\bar{B}(x,r_{x})$ with centers in $F$ and $r_{x} \in (0,s)$ then $\bar{B}(x,r_{x})= \{x\}$ and the multiplicity of any real number can be at most $1$. \demo
\end{enumerate}
\end{example}
It is easy to see that if a set $Q$ has Nagata dimension $\beta-1$ in $\Omega$, then it has metric dimension $\beta$ for all scales $s$, in the sense of Preiss. However, the Example \ref{ex:not_nagata_fd} shows that a metric space can be metrically finite dimensional for one scale and not for another scale, that is, not every metrically finite dimensional space will have finite Nagata dimension.
%%%%%%%%%%%%%%%%%%%%%%%%%%%%%%%%%%%%%%%%%%%%%%%%%%%%%
%%%%%%%%%\subsection{Weak Nagata dimension}
%%%%%%%%%\label{subsec:Weak Nagata dimension}
%%%%%%%%%%%%%%%%%%%%%%%%%%%%%%%%%%%%%%%%%%%%%%%%%%%%%

We can tweak the notion of Nagata dimension to define a new concept called `weak Nagata dimension' which is equivalent to metric dimension.
\begin{definition}[Weak Nagata dimension]
\label{def:weak_nagata}
 A subset $Q$ of $\Omega$ has \emph{weak Nagata dimension} $\beta-1$ on scale $s>0$ in $\Omega$ if for any $a \in \Omega$, whose open ball $B(a,r), r < s$ contains at least $\beta + 1$ elements from $Q$ say $x_1,\ldots,x_{\beta+1}$, there exist distinct $i,j \in 1,\ldots,\beta+1$ such that $\rho(x_i,x_j) \leq \max\{\rho(a,x_i),\rho(a,x_j) \}$. 
\end{definition}
In the similar way, we can define the notion of weak ball-covering dimension from the ball-covering dimension with respect to finite family of closed balls having radius $<s$. Now, we prove the equivalency of metric dimension, weak Nagata dimension and weak ball-covering dimension.
\begin{lemma}
\label{lem:restricted_Nagata_FD}
Let $Q$ be a subset of metric space $(\Omega,\rho)$. The following are equivalent:
\begin{enumerate}[(i)]
\item $Q$ has weak Nagata dimension $\beta-1$ on scale $s$ in $\Omega$.
\item $Q$ has metric dimension $\beta$ on scale $s$ in $\Omega$. In other words, any unconnected finite family of closed balls with centers in $Q$ and radius strictly less than $s$, has multiplicity at most $\beta$ in $\Omega$.
\item $Q$ has weak ball-covering dimension $\beta$ on scale $s$ in $\Omega$.
\end{enumerate}
\end{lemma}
\begin{proof} $(i) \Rightarrow (ii)$: Suppose $Q$ has weak Nagata dimension $\beta-1$ on the scale $s$ in $\Omega$. Let $a \in \Omega$ and consider an unconnected finite family of $m$ closed balls containing $a$, $\mathcal{F}=\{\bar{B}(x_i,r_i): x_i \in Q, r_i < s, x_i \notin \bar{B}(x_j,r_j), 1 \leq i \neq j \leq m\}$. Choose $\varepsilon > 0$ such that $ r = \max\{r_i: 1 \leq i \leq m\} + \varepsilon  < s$. So, the open ball $B(a,r)$ contains every $x_i$. Since every closed ball at $x_i$ contain $a$, for all distinct $i,j \in \{1,\ldots,m\}$ we have $\rho(x_i,x_j)> \max\{\rho(a,x_i),\rho(a,x_j)\}$. By our assumption, $m \leq \beta$. So, an element $x \in \Omega$ can belong to at most $\beta$ numbers of ball in $\mathcal{F}$. 

$(ii) \Rightarrow (iii)$: Consider any finite family of closed balls with centers in $Q$ and radius bounded by $r$, then we can extract an unconnected sub-family which contain every centers in the original family and has multiplicity at most $\beta$. The rest of the argument is exactly same as in the proof of $(ii) \Rightarrow (iii)$ in Lemma \ref{lem:wcp_nd} considering finite families of closed balls.

$(iii) \Rightarrow (i)$: Suppose the weak Nagata dimension  of $Q$ is not $\beta-1$ on scale $s$. Then, there exist $a$ in $\Omega$ and $x_1,\ldots,x_{\beta+1} \in Q$ such that the open ball $B(a,r)$ for some $0<r<s$ contains every $x_i$ and for all $i \neq j$, $\rho(x_i,x_j) > \max\{\rho(a,x_i), \rho(a,x_j)\}$. Let $\mathcal{F} = \{ \bar{B}(x_i,\rho(a,x_i)): x_i \in Q, 1 \leq i \leq \beta+1  \} $, then for any $ 1 \leq i \neq j \leq \beta+1$, $\rho(a,x_i)<r<s$ and $\bar{B}(x_i,\rho(a,x_i))$ contains element $a$ but does not contain $x_j$. So, $a$ belongs to $\beta+1$ balls in $\mathcal{F}$.
\end{proof}

A result by Preiss says that all the complete and separable metric spaces satisfying strong Lebesgue-Besicovitch differentiation property are essentially the spaces having sigma-finite metric dimension. In the following section, we present the differentiation property of a metric space. 
%%%%%%%%%%%%%%%%%%%%%%%%%%%%%%%%%%%%%%%%%%%%%%%%%%%%%
\section{The differentiation property}
\label{sec:The Lebesgue-Besicovitch differentiation property}
%%%%%%%%%%%%%%%%%%%%%%%%%%%%%%%%%%%%%%%%%%%%%%%%%%%%%
Let $f$ be an integrable (with respect to Lebesgue measure) real-valued function on $\mathbb{R}$. Henri Lebesgue proved that derivative of integral of $f$ at $x$ with respect to Lebesgue measure is $f(x)$ almost everywhere, which is known as Lebesgue differentiation theorem. Later in 1945, Abram Besicovitch extended this result for any locally finite Borel measure on a Euclidean space. This result is called the Lebesgue-Besicovitch differentiation theorem.
Based on the type of convergence, we have two notions of differentiation property.

\begin{definition}[Strong differentiation property \cite{Besicovitch_1945}]
\label{def:Lebesgue-Besicovitch diff_thm}
We say the \emph{strong differentiation property} holds for a locally finite Borel measure $\nu$ on a metric space $(\Omega,\rho)$, if for any $f \in L_{\nu}^{1}(\Omega)$
\begin{align} 
\label{eqn:strong_LBDT}
\lim_{r \rightarrow 0} \frac{1}{\nu(\bar{B}(x,r))} \int_{\bar{B}(x,r)} f(y) d\nu(y) & =  f(x), \ \ \text{ for $\nu$ a.e. $x $},  
\end{align}
where $L_{\nu}^{1}(\Omega)$ is the space of all $\nu$-integrable functions on $\Omega$.
\end{definition}
 If we replace almost everywhere convergence with convergence in measure in the equation \eqref{eqn:strong_LBDT} then it is called as \emph{weak differentiation property} for $\nu$.

If $f = \chi_{ _M}$ for any measurable $M \subseteq \Omega$ , then we get a special case of differentiation property.
\begin{definition}[Strong density property]
\label{def:Lebesgue-Besicovitch density_thm}
The \emph{strong density property} is said to hold for a locally finite Borel measure $\nu$, if for any measurable set $M \subseteq \Omega$ such that $\nu(M) < \infty$, we have
\begin{align*} 
\lim_{r \rightarrow 0 } \frac{\nu( \bar{B}(x,r) \cap M)}{\nu(\bar{B}(x,r))} = \chi_{ _M}(x) \ \ \text{ for $\nu$ a.e. $x $,} 
\end{align*}
where $\chi_{ _M} $ is a characteristic function of $M$. In the same manner, the \emph{weak density property} is defined with convergence in measure instead of almost everywhere convergence. 
\end{definition}
To avoid ambiguity, we state the following remark regarding notations.
\begin{remark}
If the differentiation or density property holds for all locally finite Borel measures on $\Omega$, then we say $\Omega$ satisfy the Lebesgue-Besicovitch differentiation or density property.
\demo \end{remark}

%%%%%%%%%%%%%%%%%%%%%%%%%%%%%%%%%%%%%%%%%%%%%%%%%%%%%
\section{A result by Preiss}
\label{sec:A result by Preiss}
%%%%%%%%%%%%%%%%%%%%%%%%%%%%%%%%%%%%%%%%%%%%%%%%%%%%%
This section explains the necessary part of the following theorem given by Preiss.
\begin{theorem}[\cite{Preiss_1983}]
\label{thm:Priess_theorem}
	Let $(\Omega, \rho)$ be a complete separable metric space. The strong Lebesgue-Besicovitch differentiation property holds for $\Omega$ if and only if $\Omega$ is metrically sigma-finite dimensional.
\end{theorem}
In his short paper, Preiss did not give the proof of the above theorem as such, instead he just outlined the basic ideas of the proof in a few sentences. To work out a complete proof of sufficiency, Assouad and Gromard have written a 61-page long paper \cite{Assouad_Gromard_2006}. The necessity condition was never given a full proof. It is the first time that we give a detailed proof of necessity of $\Omega$ to be metrically sigma-finite dimensional in the theorem. 

To prove the necessary part of Theorem \ref{thm:Priess_theorem}, we require several results. Firstly, we prove the following lemma about the existence of a non-empty open subset of a metrically not sigma-finite dimensional space, which does not contain any non-empty open metrically finite dimensional set. 
\begin{lemma}
\label{lem:p_result_1}
Let $(\Omega,\rho)$ be a complete separable metric space. Suppose $(\Omega,\rho)$ is not metrically sigma-finite dimensional, then there is a non-empty open subset $W$ of $\Omega$ such that $W$ is not metrically sigma-finite dimensional in itself and does not contain any non-empty open metrically finite dimensional subset. 
\end{lemma}
\begin{proof}
Suppose the collection $\mathcal{F} = \{O \subseteq \Omega: O \text{ is non-empty, open and has} \allowbreak \text{ finite metric dimension in } \Omega\}$ is non-empty (otherwise, there is nothing to prove). Let $U$ denote the union of $\mathcal{F}$. 

Since $U$ is both metrizable and separable, from the two lemmas \ref{app:para_space} and \ref{app:lindelof}, there is a countable open locally finite refinement of $\mathcal{F}$, say $\{V_j: j \in \mathbb{N}\}$ such that each $V_j$ is a subset of some $O$ in $\mathcal{F}$. So $V_j$ is metrically finite dimensional in $\Omega$ and therefore $U = \cup_{j=1}^{\infty} V_{j}$ is metrically sigma-finite dimensional in $\Omega$. It follows from the Lemma \ref{app:countable_closed} that $\bar{U} = \cup_{j=1}^{\infty} \bar{V}_{j}$. Further the Lemma \ref{lem:fd_closed} implies that each $\bar{V}_j$ is metrically finite dimensional in $\Omega$. Thus $\bar{U}$ is metrically sigma-finite dimensional and hence a proper subset of $\Omega$, for otherwise $\Omega$ would be metrically sigma-finite dimensional. 

The set $W = \Omega \setminus \bar{U}$ is a non-empty open set. Let us prove that $W$ contains no non-empty open subsets which are metrically finite dimensional in $W$. Suppose there is a non-empty open subset $Y $ of $W$ which has metric dimension $\beta$ on some scale $s>0$ in $W$. Notice that $Y$ is open in $W$ and $W$ is open in $\Omega$, so $Y$ is open in $\Omega$. Let $x \in Y$ and choose $r>0$ such that $B^{\Omega}(x,r)$ is contained in $Y$ and is metrically finite dimensional in $W$. 

Let $t = \min\{s,r/2\}$. We show that indeed $B^{\Omega}(x,t)$ is metrically finite dimensional on scale $t$ in $\Omega$ and get the contradiction. Let $\mathcal{G} = \{ \bar{B}^{\Omega}(x_i,r_i): x_i \in B^{\Omega}(x,t), 0 <r_i < t\}$ be any finite family of closed balls in $\Omega$ with centers in $B^{\Omega}(x,t)$ and radius bounded above by $t$. As $Y$ has finite metric dimension, there exists a subfamily $\mathcal{G}' $ which contain all $x_i$ and has multiplicity at most $\beta$ in $W$.

If $a \in \bar{B}^{\Omega}(x_i,r_i)$, then by triangle's inequality $\rho(a,x) \leq \rho(a,x_i) + \rho(x_i,x) < r$. So every $\bar{B}^{\Omega}(x_i,r_i)$ is a subset of $B^{\Omega}(x,r)$ and hence a subset of $Y$. This implies that no element of $U$ can belong to any ball in $\mathcal{G}'$ and therefore $\mathcal{G}'$ has multiplicity at most $\beta$ in $\Omega$. Thus $B^{\Omega}(x,t)$ is a non-empty open set which has metric dimension $\beta$ on scale $t$ in $\Omega$ and by the definition of $U$, $B^{\Omega}(x,t)$ must be a subset of $U$ which is not possible.

As $W$ is homeomorphic to a complete metric space, from the Proposition \ref{prop:subset_open} we conclude that $W$ is not metrically sigma-finite dimensional in itself.
\end{proof}

Suppose $(\Omega,\rho)$ is a complete separable metric space. Let $\mathcal{P}_{\Omega}$ be the set of all probability measures on ${\Omega}$. 
By the Portmanteau theorem \cite{Billingsley_1999}, the convergence of measures in weak topology is equivalent to the convergence of  measures in the metric space 
$(\mathcal{P}_{\Omega}, \pi)$, where for $\mu,\nu \in \mathcal{P}_{\Omega}$, $\pi(\mu,\nu)$ is the infimum of the set 
\begin{align*}
 \{\delta> 0 : \mu(A) \leq \nu(A^{\delta}) + \delta, \nu(A) \leq \mu(A^{\delta}) + \delta \ \ \forall A \in \mathcal{B}(\Omega) \}.  
\end{align*}
Here $A^{\delta}= \cup_{x \in A} B(x,\delta)$. This metric is known as Lévy-Prokhorov metric.
Note that if $\mu(A) \leq \nu(A^{\delta}) + \delta$ holds for all Borel subsets $A$ of $\Omega$ then $\pi(\mu,\nu) < \delta$ \cite{Billingsley_1999}.

For each $n \in \mathbb{N}$, define $M_n \subseteq \mathcal{P}_{\Omega}$ as the collection of measures $\mu$ such that for each $\mu$, there exists a measure $\nu \in \mathcal{P}_{\Omega}$ such that  
\begin{align*}
   \mu \{ x : \nu(\bar{B}(x,r)) \leq n \mu(\bar{B}(x,r)) \text{ for every } r < 1/n \} < 1/n.
\end{align*}
Let $M_n^{\circ}$ denote the interior of $M_n$. We denote by $\delta_{x}$ the Dirac measure supported at $x$, that is, $\delta_{x}(A)$ is equal to 1 if $x \in A$ and is 0 otherwise. Our goal is to prove that $M_n^{\circ}$ is dense in $\mathcal{P}_{\Omega}$, for which it is enough to show that the closure of $M_n^{\circ}$ contains the set $\mathcal{N}$ of all finitely supported measures:
\begin{align*}
\mathcal{N} = \bigg\{ \sum_{i=1}^{k} \beta_i \delta_{d_i}: k \in \mathbb{N}, d_i \in S, \beta_i \in [0,1], \sum_{i=1}^{k} \beta_i = 1\bigg\}.
\end{align*}
Here, $S$ is the countable dense subset of $\Omega$. It is known that $\mathcal{N}$ is a dense subset of $\mathcal{P}_{\Omega}$ \cite{Billingsley_1999}.

For each $n \in \mathbb{N}$, let $\mathcal{C}_n$ denote the collection of all probability measures of the form 
\begin{align*} 
\bigg\{ \sum_{i=1}^{m} \alpha_{i} \delta_{a_i}  \bigg\}, 
\end{align*}
satisfying the following properties:
\begin{enumerate}[(i)]
\item $m > n$, each $0 < \alpha_i < 1/n$ and $\sum_{i=1}^{m} \alpha_i =1$, 
\item corresponding to the set $\{a_1,a_2,\ldots, a_m\}$, there exist $0 < r_1,\ldots,r_m < 1/n$ such that $a_j \notin \bar{B}(a_i,r_i)$ for $1 \leq i \neq j \leq m$ and there is an element $y \in \Omega$ which belongs to every ball $\bar{B}(a_i,r_i)$.
\end{enumerate}

Let $\mu_{i} \in \mathcal{C}_{n}$. Let the support of  $\mu_{i}$ be denoted by $F_i = \{ a_{i1},\ldots,a_{i m_i}\},  m_i >n$, corresponding to which the set of radii and the  common element are denoted by $\{r_{i1},\ldots, r_{i m_i}\}$ and $y_i$ respectively. Also, let $\{\alpha_{i1},\ldots, \alpha_{i m_i}\}$ denote the set of coefficients. 

Let $\mathcal{A}_n$ be the collection of all probability measures of the form 
 \begin{align*} 
\bigg\{ \sum_{i=1}^{l} \lambda_i \mu_i \colon \mu_{i} \in \mathcal{C}_{n}, \lambda_i \in [0,1],  \sum_{i=1}^{l} \lambda_i =1 \bigg\}, 
\end{align*}
 such that the set $\{F_i\}$, where $F_i$ is the support of $\mu_i$, produces an unconnected family of balls, that is, no closed ball at $a_{ik} \in F_i$ of radius $r_{ik} < 1/n$ intersects $F_j$ for all $1\leq i \neq j \leq l$ and $1\leq k\leq m_i$. 

\begin{lemma}
Every element of $\mathcal{A}_n$ is an element of $M_n$.
\end{lemma}
\begin{proof}
It is easy to check that every $\mu \in \mathcal{A}_n$ belongs to $M_n$. Indeed, let $\nu = \sum_{i=1}^{l} \lambda_i \delta_{y_i}$. 
Let $a_{ij}\in F = \cup_{i=1}^{l}F_i$. 
By the assumption on the family $\{F_i: 1 \leq i \leq l \}$, $\bar{B}(a_{ij},r_{ij}) \cap F = \{a_{ij}\}$. 
Therefore, $\mu(\bar{B}(a_{ij},r_{ij})) = \lambda_i \alpha_{ij}$. 
Also, $y_i \in \bar{B}(a_{ij},r_{ij})$, which implies $\nu(\bar{B}(a_{ij},r_{ij})) \geq \lambda_i$. 
Since $\mu_i \in \mathcal{C}_n$ for all $1\leq i\leq l$,  
\begin{align*}
\frac{\nu(\bar{B}(a_{ij},r_{ij}))}{\mu(\bar{B}(a_{ij},r_{ij}))} \geq \frac{1}{\alpha_{ij}} >n.
\end{align*}
Since $\mu$ is supported on $F$, 
\begin{align*}
\mu \bigg\{x : \nu(\bar{B}(x,r)) \leq n \mu(\bar{B}(x,r))~\text{for all}~r < 1/n \bigg\} = 0. 
\end{align*}
Hence, $\mu \in M_n$.
\end{proof}

Furthermore, $\mu \in M_n^{\circ}$, which we show in the following lemma. 
\begin{lemma}
\label{lem:m_n_interior} 
Let $\mu \in \mathcal{A}_n$. Then $\mu$ is an element of $M_{n}^{\circ}$.
\end{lemma}
\begin{proof} 
Let $\omega$ be a probability measure such that $\pi(\mu,\omega) < \varepsilon$, where $\varepsilon >0$. Then from the definition of the metric $\pi$, there exists 
$\delta > 0$ such that $\pi(\mu,\omega) \leq \delta \leq \varepsilon$ and 
\begin{align*}
  \omega(B) \leq \mu(B^{\delta}) + \delta~\text{for all}~B \in \mathcal{B}(\Omega).
\end{align*}
 We will find the conditions on $\delta$, and hence on $\varepsilon$, such that whenever $\pi(\mu,\omega) < \varepsilon$, the measure $\omega $ belongs to $M_n$. 
This will imply that the measure $\mu$ belongs to $M_{n}^{\circ}$.

Let $\nu = \sum_{i=1}^{l} \lambda_i \delta_{y_i}$. Let $D = D_{\omega}$ denote the set of all $x$ in $\Omega$ such that $\nu(\bar{B}(x,r)) \leq n \omega(\bar{B}(x,r)) $ for all $r < 1/n$. It is evident that if $\omega(D) < 1/n$ then $\omega \in M_n$.

Set $F = \cup_{i=1}^{l}F_i$, which is the support of $\mu$. Let $A$ be the set of all $x \in \Omega$ such that $\bar{B}(x,\delta) \cap F $ is empty, then $A$ is a subset of $F^{c}$. We write $D$ as the disjoint union of intersection of $D$ with three sets $F, A$ and $F^c \setminus A$. Therefore,
 \begin{align}
\label{eqn:total_bound}
\omega(D) = \omega(D \cap F) + \omega( D \cap (F^{c} \setminus A)) + \omega(D \cap A).
\end{align}
The set $A^{\delta}$ does not intersect $F$, so $\mu(A^{\delta}) = 0$. Therefore, we have
\begin{align*}
\omega(D \cap A)  \leq  \omega(A) \leq \mu(A^{\delta}) + \delta = \delta.
\end{align*}
We now show that the sets $D \cap F$ and $D \cap (F^{c} \setminus A)$ can be made empty by choosing an appropriate $\delta >0$, denoted by $\delta_0$. 
Note that the choice of such $\delta$ could be made beforehand. It then follows that for $\varepsilon < \min\{1/n, \delta_0\}$, the open ball centered at $\mu$ of radius $\varepsilon$ with respect to the metric $\pi$ lies in $M_n$, and hence, $\mu$ is an element of $M_n^{\circ}$.
\begin{itemize}
\item Suppose $z \in D \cap F$, then $z$ is equal to some $a_{ij} \in F_i$. 
This implies $\omega(\bar{B}(a_{ij},r_{ij})) \leq \mu(\bar{B}(a_{ij},r_{ij} + \delta)) + \delta \leq \lambda_i \alpha_{ij} + \delta$, whenever 
\begin{align*}
r_{ij} + \delta < \min\{\rho(a_{ij},a): a \in F, a \neq a_{ij}\}. 
\end{align*}
Let $p_{ij}=\min\{\rho(a_{ij},a): a \in F, a \neq a_{ij}\}$. Note that there always exist such $\delta$ satisfying the above inequality since by the assumption on the family $\{F_i\}$, $r_{ij} < p_{ij}$. 
By choosing $\delta > 0$ such that 
\begin{align*}
\frac{\lambda_i}{\lambda_i \alpha_{ij} + \delta} > n, 
\end{align*}
which is always possible since $\alpha_{ij} < 1/n$, the fraction $\nu(\bar{B}(a_{ij},r_{ij}))/\allowbreak \omega(\bar{B}(a_{ij},r_{ij})) \geq (\lambda_i) / (\lambda_i \alpha_{ij} + \delta)$ is strictly greater than $n$, which implies $a_{ij}$ does not belong to $D \cap F$.
Thus, by taking $\delta < \min\{t_1, t_2\}$, where 
\begin{align*}
t_1 = \min_{ij}\{p_{ij} - r_{ij}\} ~\text{and}~ t_2= \min_{ij}\left\{\lambda_i\left(\frac{1}{n}-\alpha_{ij}\right)\right\}, 
\end {align*}
we conclude that the set $D \cap F$ is empty. 
\item Now, let $z$ be an element of $D \cap (F^c \setminus A)$ then there is an element $a_{ij}$ in $\bar{B}(z,\delta)$. Let $t_{ij} = \rho(\bar{B}(a_{ij},\delta),y_i)$. If we choose $\delta$ such that 
\begin{align*}
4 \delta + t_{ij} < p_{ij},
\end{align*}
which is always possible since $t_{ij} < p_{ij}$, $\bar{B}(z,2\delta + t_{ij})$ will contain the common element $y_i$. It follows from the triangle’s inequality, $\rho(z,y_i) \leq \rho(z,a_{ij}) + \rho(a_{ij},y_i) \leq 2\delta + t_{ij} $. Note that  the ball $\bar{B}(z,2\delta + t_{ij})$ may also contain some $y_j$, $j\neq i$.

On the other hand, $\bar{B}(z, 3\delta + t_{ij})$ will not contain any element except $a_{ij}$ from $F$. 
To see this, suppose $a \in F$ such that $a \neq a_{ij}$ is in the closed ball $\bar{B}(z, 3\delta + t_{ij})$. Then 
$\rho(a, a_{ij}) \leq \rho(a, z) + \rho(z, a_{ij}) < 4\delta + t_{ij} < p_{ij}$, which is a contradiction.

This implies that $\omega(\bar{B}(z,2\delta + t_{ij})) \leq \lambda_i\alpha_{ij} + \delta$ and
\begin{align*}
 \frac{\nu(\bar{B}(z,2\delta + t_i))}{\omega(\bar{B}(z,2\delta + t_i))}  \geq \frac{\lambda_i}{\lambda_i \alpha_{ij} + \delta}.
\end{align*}
The left hand side of the above inequality is strictly greater than $n$ since we will choose $\delta < \min\{t_1, t_2\}$. In addition, if we take $\delta < t_3$, where 
\begin{align*}
 t_3 = \min_{ij}\{p_{ij} - t_{ij}\},
\end{align*}
the set $D \cap (F^c \setminus A)$ will be empty.
\end{itemize}
Hence, whenever $\delta_0 < \min\{t_1, t_2, t_3\}$,  
we will have $\omega(D) \leq \delta_0 < 1/n$, implying $\omega \in M_n$.
  \end{proof}

From the Lemma~\ref{lem:m_n_interior}, it is sufficient to prove that $\mathcal{A}_n$ is dense in $\mathcal{N}$ to prove the denseness of $M_n^{\circ}$ in $\mathcal{P}_{\Omega}$, which is 
shown under the assumption that no nonempty open set of $\Omega$ is metrically finite dimensional. 
\begin{lemma}
\label{lem:m_n_dense}
Suppose that no nonempty open subset of $\Omega$ is metrically finite dimensional in $\Omega$. Then for each $n$, the set $M_{n}^{\circ}$ is dense in $\mathcal{P}_{\Omega}$.
\end{lemma}
\begin{proof}
Let $\omega = \sum_{i=1}^{l}\lambda_i \delta_{d_i}$ be an element of $\mathcal{N}$, where $d_i \in S$. Let $t = \min\{\rho(d_i,d_j): 1 \leq i \neq j \leq l\}$ and let $0 <\varepsilon < t$. Note that $B(d_i,\varepsilon ) \cap B(d_j,\varepsilon) = \emptyset$, for all $1 \leq i \neq j \leq l$.

By the assumption, $B(d_i,\varepsilon /2), 1 \leq i \leq l$ is not metrically finite dimensional on scale $\varepsilon/2$. Therefore, we have a set of measures $\{\mu_i: 1 \leq i \leq l\}$ in $\mathcal{C}_n$, where each $\mu_i$ has support $F_i \subseteq B(d_i,\varepsilon/2)$. This implies that for $a \in F_i$, the closed ball $\bar{B}(a, r), r< \varepsilon/2$ is contained in $B(d_i,\varepsilon)$ and does not contain any other element of $F_j$. So, $F = \cup_{i=1}^{l} F_i$ will form an unconnected family of closed balls and hence, from the Lemma \ref{lem:m_n_interior},  $ \mu = \sum_{i=1}^{l} \lambda_i \mu_i$ is in $M_n^{\circ}$. 

Let $A$ be any Borel measurable subset of $\Omega$. We will show that $\omega(A)\leq \mu (A^{\varepsilon/2}) + \varepsilon/2$, which then completes the proof.
It is trivial if $A$ does not contain any $d_i$. Suppose $d_i$ belongs to $A$. Then $F_i \subseteq A^{\varepsilon/2}$ since $F_i$ is contained in $ B(d_i,\varepsilon/2)$,
and hence, $\mu(A^{\varepsilon/2}) \geq \lambda_i = \omega(A)$, if no $d_j$, $j \neq i$ is contained in $A$. 
Therefore,
\begin{align*}
  \omega(A)\leq \mu (A^{\varepsilon/2}).
\end{align*}
\end{proof}

Now we are ready to give the proof of the necessity condition in the Theorem~\ref{thm:Priess_theorem}. We state it as a separate lemma as follows.
\begin{proof}[Proof of Necessary part of Theorem~\ref{thm:Priess_theorem}]
We want to prove the following:	Let $(\Omega, \rho)$ be a complete separable metric space. Suppose the strong Lebesgue-Besicovitch differentiation property for $\Omega$. Then $\Omega$ is metrically sigma-finite dimensional.

Suppose $\Omega$ is not metrically sigma-finite dimensional. From the Lemma \ref{lem:p_result_1}, without loss of generality we can assume that there is no non-empty open subset of $\Omega$ which is metrically finite dimensional in $\Omega$. Let $S$ be a dense countable subset of $\Omega$.
	
From the Lemma \ref{lem:m_n_dense}, we have that each $M_n^{\circ}$ is a dense open subset of $\mathcal{P}_{\Omega}$. It follows from the Baire Category Theorem that $\cap_{n \in \mathbb{N}} M_n^{\circ}$ is dense and hence $\cap_{n \in \mathbb{N}} M_n$ is non-empty.

Let $\mu \in \cap_{n \in \mathbb{N}} M_{n} $, then for each $n$ there is a sequence of probability measures $\nu_{n}$ such that 
\begin{align*}
\mu \bigg( A_{n} = \{ x \in \Omega:  \nu_n(\bar{B}(x,r)) > n \mu(\bar{B}(x,r)) \text{ for some } r < 1/n \} \bigg)  \geq  1 - \frac{1}{n}. 
\end{align*}
Since $\limsup_{n} \mu(A_n) \leq  \mu(\limsup_{n} A_n)$ (See Theorem 4.1 \cite{Billingsley_2012}) and that $\mu$ is a probability measure, we have 
\begin{align*}
     \mu\left(\limsup_{n} A_n\right) = 1.
\end{align*}
Let $\nu = \sum_{n=1}^{\infty} \alpha_n \nu_n$, where  $\alpha_{n} = \frac{1}{n(n+1)}$. Define the following set.
\begin{align*}
A = \bigg\{x: \limsup_{r \rightarrow 0} \frac{\nu(\bar{B}(x,r))}{\mu(\bar{B}(x,r))} = \infty \bigg\}
\end{align*}
We now show that $\limsup_{n}A_n \subseteq A$. Given any $x \in \limsup_{n} A_n$, $x$ belongs to $A_n$ for infinitely many $n$ and so there is an increasing sequence $m_1,m_2,\ldots$ in $n$ such that $x \in A_{m_t}$ for every $t \in \mathbb{N}$. 

For each $m_t$, we have a $0< r_t  < \frac{1}{m_t}$ such that 
\begin{align*}
\frac{\alpha_{m_t} \nu_{m_t}(\bar{B}(x, r_t))}{\mu(\bar{B}(x,r_t))}  > \alpha_{m_t} m_t.
\end{align*}
We can assume that $(r_t)$ is a decreasing sequence converging to zero. Let fix $t \geq 2$. 
For all $1 \leq j \leq t-1$, we have
\begin{align*}
\frac{\alpha_{m_j} \nu_{m_j}(\bar{B}(x, r_t))}{\mu(\bar{B}(x,r_t))} > \alpha_{m_j} m_j \frac{\mu(\bar{B}(x, r_j))}{\mu(\bar{B}(x, r_t))} \geq \alpha_{m_j} m_j,
\end{align*}
while for all $t+1 \leq j$,
\begin{align*}
   \frac{\alpha_{m_j} \nu_{m_j}(\bar{B}(x, r_t))}{\mu(\bar{B}(x,r_t))} > \alpha_{m_j} m_j \frac{\nu_{m_j}(\bar{B}(x, r_t))}{\nu_{m_j}(\bar{B}(x, r_j))}.
\end{align*}
Let $s_t = \sum_{j=1}^{t} \alpha_{m_j} m_j$. Then for all $t \geq 1$,
\begin{align*}
       \frac{\nu(\bar{B}(x,r_t))}{\mu(\bar{B}(x,r_t))} =  \sum_{j=1}^{\infty}\frac{ \alpha_{m_j} \nu_{m_j}(\bar{B}(x,r_t))}{\mu(\bar{B}(x,r_t))} > s_t+   \sum_{j=t+1}^{\infty} \alpha_{m_j} m_j \frac{\nu_{m_j}(\bar{B}(x, r_t))}{\nu_{m_j}(\bar{B}(x, r_j))} > s_t.
\end{align*}
Since $s_t$ tends to infinity as $t$ tends to infinity, it implies $x \in A$. Hence, $\limsup_{n}A_n \subseteq A$. 
Since $\mu\left(\limsup_{n} A_n\right) = 1$,  we have then $\mu(A)=1$. In other words,
\begin{align*}
 \limsup_{r \rightarrow 0} \frac{\nu(\bar{B}(x,r))}{\mu(\bar{B}(x,r))} = \infty \text{ for } \mu-\text{a.e.} 
\end{align*}
This means that
\begin{align}\label{contradiction}
      \limsup_{r \rightarrow 0} \frac{\mu(\bar{B}(x,r))}{(\mu+ \nu)(\bar{B}(x,r))} = 0 \text{ for } \mu-\text{a.e.}.
\end{align}
As $\mu$ is absolutely continuous with respect to $\mu + \nu$. By the Radon-Nikodym theorem, there is a measurable function $f: \Omega\rightarrow [0,\infty)$ such that for any measurable $A$, we have $\mu(A) = \int_{A} f(y) (\mu+ \nu)(dy)$.
Then we have, 
\begin{align*}
\limsup_{r \rightarrow 0} \frac{\mu(B(x,r))}{(\mu+ \nu)(B(x,r))}  & = \limsup_{r \rightarrow 0} \frac{1}{(\mu+ \nu)(B(x,r))} \int_{B(x,r)} f(y) (\mu+ \nu) (dy)
\end{align*}
Suppose that $\mu+ \nu$ satisfies the strong differentiation property, then the right-hand side of the above equation is equals to $f(x)$ for $(\mu+\nu)$-almost everywhere and hence also for $\mu$-almost everywhere, while the left-hand side is equals to 0 for $\mu$-almost everywhere.
This gives that $f(x) = 0$ for $\mu$-almost everywhere and therefore contradicts the fact that $\mu$ is a probability measure. This completes the proof. 
\end{proof}

%%%%%%%%%%%%%%%%%%%%%%%%%%%%%%%%%%%%%%%%%%%%%%%%%%%%%%%%%%%%%%%%%%%%%%%%%%%%%%%%%%%%%%%
\chapter{Statistical Machine Learning}
\label{chap:Statistical machine learning}
%%%%%%%%%%%%%%%%%%%%%%%%%%%%%%%%%%%%%%%%%%%%%%%%%%%%%%%%%%%%%%%%%%%%%%%%%%%%%%%%%%%%%%%
In this chapter, we introduce the fundamentals of statistical machine learning. We start with the binary classification problem and then discuss about the learning rules, error of a learning rule and consistency. 

%%%%%%%%%%%%%%%%%%%%%%%%%%%%%%%%%%%%%%%%%%%%%%%%%%%%%
\section{Binary classification problem}
\label{sec:Binary Classification problem}
%%%%%%%%%%%%%%%%%%%%%%%%%%%%%%%%%%%%%%%%%%%%%%%%%%%%%
A classification problem is categorizing a set of data, for example sorting the clothes based on its colors like blue, white, red, yellow etc. These categories are referred as labels for a data. In general, any classification problem can be understood as a binary classification problem that is, with two labels. Almost every machine learning concepts can be modeled mathematically.

 Let $\Omega$ be a non-empty set and let $\{0,1\}$ be the set of labels. A \emph{labeled sample} $\sigma_n$ of size $n$ is an element of $(\Omega \times \{0,1\})^n$,
\begin{align*}
\sigma_n = (x_1,y_1), \ldots,(x_n,y_n),
\end{align*}
where $(x_i,y_i) \in \Omega \times \{0,1\}$ such that each data point $x_i$ has label $y_i$. Then, the binary classification problem (see figure \ref{fig:classification}) is defined as follows. 
\begin{definition}[Binary classification problem \cite{Devroye_Gyorfi_Lugosi_1996}]
Given $\sigma_n$, a binary classification problem is to construct a Borel measurable function $g : \Omega \rightarrow \{0,1\}$ such that $g(x_i) = y_i$ for every $1 \leq i \leq n$ and that $g$ assigns a label $0$ or $1$ to every element $x$ of $\Omega$. The function $g$ is called a \emph{classifier}. 
\end{definition}
%%%%%%%%%%%%%%%%%%%%%%%%%%%%%%%%%%%%%%%%%%%%%%%%%%%%%
%%\vspace{0.6cm}
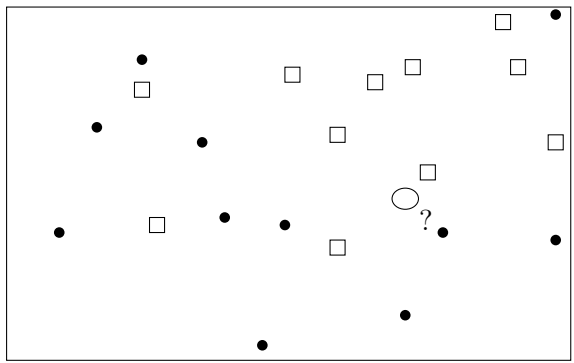
\begin{figure}
\centering
\begin{tikzpicture}
%\hspace{1cm} 
\draw (0.3,0.3) rectangle (7.8,5);
\fill (1,2) circle (2pt);
\fill (1.5,3.4) circle (2pt);
\fill (2.9,3.2) circle (2pt);
\fill (2.1,4.3) circle (2pt);
\fill (3.2,2.2) circle (2pt);
\fill (3.7,0.5) circle (2pt);
\fill (4,2.1) circle (2pt);
\fill (5.6,0.9) circle (2pt);
\fill (6.1,2) circle (2pt);
\fill (7.6,4.9) circle (2pt);
\fill (7.6,1.9) circle (2pt);
\draw (2.2,2) rectangle (2.4,2.2);
\draw (2,3.8) rectangle (2.2,4);
\draw (4.6,1.7) rectangle (4.8,1.9);
\draw (4,4) rectangle (4.2,4.2);
\draw (4.6,3.2) rectangle (4.8,3.4);
\draw (5.6,4.1) rectangle (5.8,4.3);
\draw (5.8,2.7) rectangle (6,2.9);
\draw (7,4.1) rectangle (7.2,4.3);
\draw (6.8,4.7) rectangle (7,4.9);
\draw (5.1,3.9) rectangle (5.3,4.1);
\draw (7.5,3.1) rectangle (7.7,3.3);
\draw (5.6,2.45) node[sloped,below] { \ \ \ \ \small ?} ellipse (5pt and 4pt);
\end{tikzpicture}
\caption{A binary classification problem: given a set of labels `rectangle' and `black dot', predict the label of new data point `ellipse'?} 
\label{fig:classification}
\end{figure}
%%%%%%%%%%%%%%%%%%%%%%%%%%%%%%%%%%%%%%%%%%%%%%%%%%%%%
Let's see an example. We want to classify the emails in our email account as spam and non-spam emails. We would like to construct a machine to do this work. We take a set of emails called training data and based on it we set a hypothesis that if subject of an email contain `credit or win' then it is a spam. Now, the machine has to classify the new emails based on this hypothesis. This is a binary classification problem with labels `spam' or `non-spam'.

Suppose we have a way to classify emails, is it what we want? A human can pick spam emails without an error by seeing the email content but a machine cannot. There is always some uncertainty in classifying emails such as emails having no subject, and hence there a possibility of error. In principle, we prefer those machines which give less error and hence are more accurate. Due to such uncertainties, the probabilistic settings are the best.
 
Let $\mu$ be a probability measure on $\Omega \times \{0,1\}$ and $(X,Y)$ be a $\Omega \times \{0,1\}$-valued random variable having distribution $\mu$. Here, $X$ is a random element having label $Y$. 
\begin{definition}[Misclassification error] 
\label{misclass_error}
The \emph{misclassification error} for a classifier $g$ is the measure of set of all labeled data points whose predicted label and actual label are different,
\begin{align*}
\ell_{\mu}(g) \ & = \ \mathbb{P}(g(X) \neq Y) \ \\
\ & = \ \mu \{(x,y) \in \Omega \times \{0,1\} : g(x) \neq y \}. 
\end{align*}
\end{definition}
The prime aim of a classifier is to predict label for a new data point. The misclassification error gives the probability that we will predict a wrong label. Like in our example of emails, the machine can classify an email from a friend as spam based on the hypothesis, whereas the actual label is `non-spam'. The Definition~\ref{misclass_error} gives the probability of such cases. It is evident that a classifier with low misclassification error will be preferred. 
%%%%%%%%%%%%%%%%%%%%%%%%%%%%%%%%%%%%%%%%%%%%%%%%%%%%%
\section{Learning rule and consistency}
\label{sec:Learning rule and consistency}
%%%%%%%%%%%%%%%%%%%%%%%%%%%%%%%%%%%%%%%%%%%%%%%%%%%%%
Here, we discuss the Bayes error and constructing good classifiers based on labeled samples to attain minimum possible error.  

Given a probability measure $\mu$ on $\Omega \times \{0,1\}$, it is possible to define the minimum possible misclassification error for $\mu$.
\begin{definition}[Bayes Error \cite{Devroye_Gyorfi_Lugosi_1996}]
The \emph{Bayes error} is the infimum of misclassification error for $\mu$, 
\begin{align*}
 {\ell}^{*}_{\mu} & = \inf\{\ell_{\mu}(g): g\text{ is a classifier on } \Omega\}   \
\end{align*}
\end{definition}
The set of classifiers is non-empty as we can always define a function as $g : \Omega \rightarrow \{1 \}$ and also, the misclassification error is bounded between $0$ and $1$. This implies that the infimum always exists and indeed is attained by the Bayes classifier (defined later). So, Bayes classifier can be a solution to the classification problem, but Bayes error depends on $\mu$ which is unknown. The only thing we have are labeled samples.  

We know that $(X,Y)$ is distributed according to $\mu$. We define two measures $\nu$ and $\nu_1$ on $\Omega$. For any measurable $A \subseteq \Omega$, let
\begin{align*} 
\nu (A) =  \mu(A \times \{0\}) +\mu(A\times\{1\}),  \ \nu_1 (A) = \mu(A\times\{1\}).
\end{align*}
As $\nu_1 \leq \nu$, so $\nu_1$ is absolutely continuous with respect to $\nu$. By the Radon-Nikodym theorem \cite{Billingsley_2012}, there exists a measurable function $\eta$ such that, 
\begin{align*}
\nu_1(A)& =\int_{A} \eta(x) \nu(dx).
\end{align*}
Here, $\eta$ is the Radon-Nikodym derivative of $\nu_1$ with respect to $\nu$. Probabilistically, $\eta$ is equal to the conditional probability of getting label $1$, given $X=x$,
\begin{align*}
\eta(x)& = \mathbb{P}(Y=1|X=x).  
\end{align*}
In statistics, $\eta$ is called \emph{regression function}. Note that, $\eta$ is function on $\Omega$ and takes values in $[0,1]$. We show below that the distribution of $(X,Y)$ can be completely described by the pair $(\nu,\eta)$, where $\nu$ is a probability measure and $\eta$ a regression function  on $\Omega$ obtained from the underlying probability measure $\mu$ on $\Omega \times \{0,1\}$. We can write any measurable set $A  \subseteq \Omega \times \{0,1\}$ as,
\begin{align*}
A & = \{ A_0 \times \{0\} \} \cup \{A_1 \times \{1\} \}.  
\end{align*}
Therefore, we have
\begin{align*}
 \mathbb{P}((X,Y) \in A) & = \mathbb{P}(X \in A_0, Y = 0) +  \mathbb{P}(X \in A_1, Y = 1) \ \\
& =\int_{A_0} (1-\eta(x)) d\nu(x) + \int_{A_1} \eta(x) d\nu(x).
\end{align*}
In the above equation, the left hand-side is in terms of $\mu$, while the right hand-side is defined by $\nu$ and $\eta$. So, the distribution of $(X,Y)$ is completely determined by $\nu$ and $\eta$. The distribution of $(X,Y)$ is described by $(\nu,\eta)$ means that the random element $X$ is distributed according to $\nu$ with a random label $Y$ following Bernoulli distribution with probability of success $\eta(x) =  \mathbb{P}(Y=1|X=x)$. 
\begin{remark}
We will intermittently describe the distribution of $(X,Y)$ by $\mu$ or $(\nu,\eta)$. In case of $(\nu,\eta)$, there is always an underlying probability measure $\mu$ on $\Omega \times \{0,1\}$.
\demo \end{remark}
With the help of the regression function, we can also define the important notion of the Bayes classifier.
\begin{definition}[See p. 10 of \cite{Devroye_Gyorfi_Lugosi_1996}] 
The \emph{Bayes classifier} is defined as:
\begin{align}
g^*(x) = & 
\begin{cases}
1  & \text{if \ } \eta(x) \geq \frac{1}{2}, \ \\
0  & \text{otherwise} 
\end{cases}
\end{align}
\end{definition}
The above definition is well defined and the error of the Bayes classifier can be defined as $ \mathbb{P}(g^{*}(X) \neq Y)$. We can infer from the following theorem that the Bayes error is indeed attained by the Bayes classifier. 
\begin{theorem}[Optimality of Bayes classifier, see Theorem 2.1 in \cite{Devroye_Gyorfi_Lugosi_1996}]
\label{thm:optimal_bayes}
Let $g$ be a classifier on $\Omega$, then we have
\begin{align*}
\mathbb{P}(g(X) \neq Y) - \mathbb{P}(g^{*}(X) \neq Y) \geq  0.
\end{align*}
\end{theorem}
\begin{proof}
We first find the probability of no error for $g$ given an element $x$, 
\begin{align*}
&\mathbb{P}(g(X)=Y|X=x) \ \\ 
&=\mathbb{P}(g(X) =1,Y=1|X=x)+\mathbb{P}(g(X)=0,Y=0|X=x) \\
& = \mathbb{P}(Y =1 | X =x ) \mathbb{I}_{\{g(x) =1\}} + \mathbb{P}(Y=0 | X =x ) \mathbb{I}_{\{g(x) =0\}} \\
& = \eta(x) \mathbb{I}_{\{g(x) =1\}} + (1 - \eta(x))\mathbb{I}_{\{g(x) =0\}}.
\end{align*}
Similarly, we have the probability of zero error for $g^*$. 
\begin{align*}
\mathbb{P}(g^*(X) = Y | X =x ) & = \eta(x) \mathbb{I}_{\{g^*(x) =1\}} + (1 - \eta(x))\mathbb{I}_{\{g^*(x) =0\}}.
\end{align*}
Then the difference of error probabilities of $g$ and $g^*$ is, 
\begin{align*}
& \mathbb{P}(g(X) \neq Y | X =x ) - \mathbb{P}(g^*(X) \neq Y | X =x ) \ \\
& = \mathbb{P}(g^*(X) = Y | X =x ) - \mathbb{P}(g(X) = Y | X =x )  \ \\
& = \eta(x)( \mathbb{I}_{\{g^*(x) =1\}} -  \mathbb{I}_{\{g(x) =1\}} ) +  (1 - \eta(x))(\mathbb{I}_{\{g^*(x) =0\}} - \mathbb{I}_{\{g(x) =0\}}) \ \\
& = 2\bigg|\eta(x) - \frac{1}{2}\bigg|\mathbb{I}_{\{g(x)\neq g^*(x)\}},
\end{align*}
which is equal to 0 if $g = g^*$, $(\eta(x)-1/2)$ if $g^*(x) =1,g(x)=0$ and $(1/2 - \eta(x))$ if $g^*(x) =0,g(x)=1$. By the definition of $g^*$, $2\eta(x)-1$ is non-negative if and only if $g^*(x) = 1$. So,
\begin{align}
\mathbb{P}(g(X) \neq Y | X =x ) - \mathbb{P}(g^*(X) \neq Y | X =x ) & \geq 0,
\end{align}  
taking the expectation over all $x \in \Omega$,
\begin{align*}
&\mathbb{P}(g(X) \neq Y )  - \mathbb{P}(g^*(X) \neq Y) \ \\ & = \mathbb{E}\{ \mathbb{P}(g(X) \neq Y | X =x ) \} - \mathbb{E}\{\mathbb{P}(g^*(X) \neq Y | X =x )\} \geq 0.
\end{align*}
\end{proof}
From the above theorem we have $\ell^{*}_{\mu} = \mathbb{P}(g^{*}(X) \neq Y)$. The Bayes classifier depends on the underlying distribution $\mu$ of $(X,Y)$, which is unknown, thus $g^{*}$ is unknown. We assume the existence of $\mu$ to make the theoretical study possible.
We have labeled samples, in our hand, we try to construct a classifier based on labeled samples. We cannot in general expect misclassification error of a classifier to be zero, but we can strive for error of a classifier to be closer to the minimum possible error, that is, Bayes error. To achieve this, we construct a family of classifiers based on labeled samples, which are known as learning rule.
\begin{definition}[Learning rule]
\label{def:learning rule}
A \emph{learning rule} of size $n$ is a mapping defined on all labeled samples of size $n$ which assigns a label to a data point given a labeled sample, 
\begin{align*}
g_n : ({\Omega} \times {\{0,1\}} )^{n} \times \Omega \rightarrow  \{0,1\}  
\end{align*}
\end{definition} 
In simpler words, a learning rule take a labeled sample $\sigma_n$ and assigns a classifier $g_n(\sigma_n)$ to it. This classifier $g_n(\sigma_n)$ then finds the label $g_{n}(\sigma_n)(x) = g(x,\sigma_n)$ for $x \in \Omega$. A \emph{learning rule} is a sequence  of maps $(g_n),n \in \mathbb{N}$ for labeled samples of all sizes. We sometimes write $g_n(x) = g_{n}(x,\sigma_n)$, with an understanding that  a learning rule is also a function of labeled samples.

A learning rule is entirely a deterministic function, but further analysis to measure the error of a learning rule require randomness and probabilistic settings. Let $D^{\infty}$ denote the infinite sequence $ (X_1,Y_1), (X_2,Y_2), \ldots $ of independently and identically distributed random variables according to a probability measure $\mu$. Then, $D^{\infty}$ is called a \emph{random sample path} and the product measure $\mu^{\infty} = \prod_{i=1}^{\infty} \mu$ is the distribution of $D^{\infty}$. The first $n$ pairs from $D^{\infty}$, denoted by $D_n = (X_1,Y_1),\ldots, \allowbreak (X_n,Y_n)$ is called a \emph{random labeled sample} of size $n$ and follows distribution $\mu^n = \prod_{i=1}^{n} \mu$. Let $\sigma_n, \sigma_{\infty}$ denote a realization of $D_n$ and $D^{\infty}$, respectively. We have an underlying assumption that each $(X_i,Y_i)$ from the random sample path is distributed according to $\mu$ and is independent of $(X,Y)$.

A random sample of size $n$, $W = (W_1,W_2,\allowbreak \ldots,W_n)$, is a vector of i.i.d. random variables, while a sample viewed as an instance is one possible realization of the sample $W$. In other words, if $W_i(p) = w_i$ for $p\in \Omega$ then, $w = (w_1,w_2,\dots,w_n)$ is one realization of $W = (W_1,W_2,\dots,W_n)$, and $w$ is considered as an instance of the random sample $W$. For example, let $X_1,X_2$ is the result of throw of two dices respectively. Then, $(X_1,X_2)$ is a random sample of size $2$ and $(1,4)$ is one instance of this sample. 
Now, we define the error of a learning rule.
\begin{definition}[Error probability of a rule \cite{Devroye_Gyorfi_Lugosi_1996}] 
\label{def:error_prob_rule}
The \emph{error probability} of $(g_n)$ is the conditional probability, 
\begin{align}
\label{eqn:error_rule}
{\ell}_{\mu}(g_n) = \mathbb{P} ( g_n(X) \neq Y | D_n ), 
\end{align}
and the \emph{expected error probability} is given by,
\begin{align}
\label{eqn:expected_error_rule}
\mathbb{E}\{{\ell}_{\mu}(g_n)\} = \mathbb{P}\left(g_n(X) \neq Y\right),
\end{align}
where the average is over all labeled samples of size $n$.
\end{definition} 
Note that, the equation \eqref{eqn:error_rule} is a function of random data and hence ${\ell}_{\mu}(g_n)$ is a random variable. In simpler words, $\ell_{\mu}(g_n)$ is a function from set of all labeled $n$-samples $(\Omega \times \{0,1\})^{n}$ to $[0,1]$ such that $\ell_{\mu}(g_n)(\sigma_n) = \mathbb{P} ( g_n(X) \neq Y | \sigma_n ) = \mu\{ (x,y): g_{n}(x)(\sigma_n) \neq y\}$. While the expectation in the equation \eqref{eqn:expected_error_rule} is with respect to $\mu^{n}$ and hence the value is a real number. 

The accuracy of a learning rule is measured by the convergence of its error probability to Bayes error.
\begin{definition}[Consistent rule \cite{Devroye_Gyorfi_Lugosi_1996}] 
\label{def:consistent_rule}
A learning rule $(g_n)$ is called \emph{weakly consistent} for a probability measure $\mu$, if the error probability converges to Bayes error in probability, that is, 
\begin{align*}
\mathbb{E}\{{\ell}_{\mu}(g_n)\} \rightarrow {\ell}^*_{\mu}, \text{ \  as } n \rightarrow \infty, 
\end{align*}
while, $(g_n)$ is said to be \emph{strongly consistent} if
\begin{align*}
{\ell}_{\mu}(g_n) \rightarrow {\ell}^*_{\mu} \ \text{ almost surely,  as } n \rightarrow \infty, 
\end{align*}
\end{definition}
%%%%%%%%%%%%%%%%%%%%%%%%%%%%%%%%%%%%%%%%%%%%%%%%%%%%%
%\vspace{0.6cm}
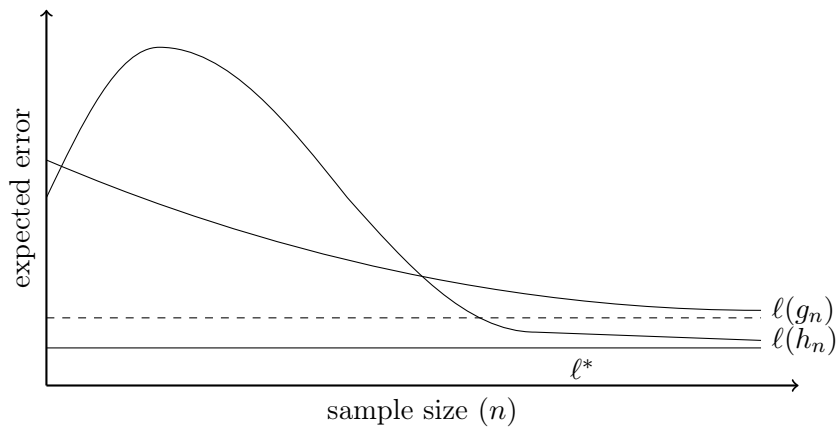
\begin{figure}
\centering
\begin{tikzpicture}
%\hspace{1cm} 
\draw[->,thick] (0,0) --  node[sloped,below] {\small {sample size $(n)$}} (10,0);
\draw[->,thick] (0,0) --  node[sloped,above] {\small {expected error} } (0,5);
\draw (0,0.5) -- node[sloped,below, near end] {\small$\ell^*$} (9.5,0.5);
\draw (9.5,1)  node[sloped,right] {\small{$\ell(g_n)$}} parabola (0,3) ;
\draw[dashed] (0,0.9) -- (9.5,0.9);
%\draw  (0,3) parabola bend(1.5,4.5)  (9.5,0.55);
\draw (0,2.5) sin (1.5,4.5) cos (4,2.5) sin (6.5,0.71) .. controls (6.5,0.71) and (9.5,0.6) .. (9.5,0.6)  node[sloped,right] {\small{$\ell(h_n)$}};
  %\draw  plot[smooth,domain=0:9.5] (\x, {4-(\x)^2});
\end{tikzpicture}
\caption{Consistent rule: $(h_n)$ is weakly consistent because its expected error converges to Bayes error ($\ell^*_{\mu}$), whereas $(g_n)$ is not weakly consistent as its expected error does not converges to Bayes error. The expected error of $(g_n)$ is monotonically non-increasing, so $(g_n)$ is a smart learning rule. } 
\label{fig:consistent rule} 
\end{figure}
%\vspace{0.6cm}
%%%%%%%%%%%%%%%%%%%%%%%%%%%%%%%%%%%%%%%%%%%%%%%%%%%%%

Note that, the almost sure convergence in the definition of strong consistency is with respect to random sample path $D^{\infty}$. That is,	the set of infinite labeled samples $\sigma_{\infty}$ for which the error probability $\ell_{\mu}(g_n)(\sigma_n)$ converges to Bayes error has measure one with respect to $\mu^{\infty} $, in other words, 
$$\mu^{\infty}\{ \sigma^{\infty}: \lim_{n \rightarrow \infty} \ell_{\mu}(g_n)(\sigma_n) = \ell^{*}_{\mu} \} = 1. $$ 
The notion of weak consistency demonstrates that if we increase the data size then with high probability we have the average error over all labeled samples of size $n$ to achieve Bayes error. While, for a strongly consistent rule, the error probability converges to Bayes error for almost every infinite labeled sample. In addition, a learning rule may be consistent for a particular distribution and may not be consistent for another distribution. It is preferable to construct a learning rule which is consistent for every distribution without having the need to know the unknown distribution.
\begin{definition}[Universally consistent rule \cite{Devroye_Gyorfi_Lugosi_1996}]
\label{def:universal_consistency} 
A learning rule $(g_n)$ is said to be {\emph{universally weakly consistent}} if it is weakly consistent for every probability measure $\mu$ on $\Omega \times \{0,1\}$. Similarly, a \emph{universally strongly consistent} rule is strongly consistent for every probability measure on $\Omega \times \{0,1\}$. 
\end{definition}
The $k$-nearest neighbor rule is an example of a universally consistent learning rule. In fact, the $k$-nearest neighbor rule is also universally strongly consistent in Euclidean spaces. We will explore more about consistency of the $k$-nearest neighbor rule in Chapter~\ref{chap:The $k$-nearest neighbor rule} and Chapter~\ref{chap:Consistency in a metrically finite dimensional spaces}. From a theoretical perspective, we prefer universally consistent rules but learning rules like Random forests rule which is not universally consistent are also employed in practical applications due to their ~high accuracy \cite{Biau_Devroye_Lugosi_2008}.  A learning rule whose expected error decreases monotonically with increasing $n$ is called a \emph{smart learning rule}. It has been conjectured that a universally consistent rule is not a smart rule (see Problem 6.16 of \cite{Devroye_Gyorfi_Lugosi_1996}).

Recently, a mutual notion of consistency has been introduced \cite{Zakai_Ritov_2009}, which measures the closeness between two learning rules. 
\begin{definition}[Mutually consistent \cite{Zakai_Ritov_2009}] 
\label{def:mutual_consistent_rule}
Two learning rules $(g_n)$ and $(h_n)$ are called \emph{mutually weakly consistent} if for every distribution $\mu$ on $\Omega \times \{0,1\}$,
\begin{align*}
\mathbb{E}_{\mu}\{ |g_n(X) - h_n(X)|\} \rightarrow 0 \text{ \ as \ } n \rightarrow \infty.
\end{align*} 
\end{definition}
A learning rule is universally weakly consistent if and only if it is mutually weakly consistent with Bayes rule. The notion of mutual strong consistency can be defined similarly.
%%%%%%%%%%%%%%%%%%%%%%%%%%%%%%%%%%%%%%%%%%%%%%%%%%%%%
\section{How to construct a learning rule?}
\label{sec:How to construct a learning rule?}
%%%%%%%%%%%%%%%%%%%%%%%%%%%%%%%%%%%%%%%%%%%%%%%%%%%%%
We know that a possible way to find a good classifier with low misclassification error is to construct a sequence of learning rules whose error can be made as small as possible. However, the real question is what is the form of such a learning rule when the only thing being available are the labeled samples. A formal way to construct a learning rule is elucidated in \cite{Devroye_Gyorfi_Lugosi_1996}. The basic idea is to devise a function $\eta_n$ with the help of labeled samples and try to approximate the regression function $\eta$. The most common way is to assign weights to the labeled sample. 

Given a labeled sample $\sigma_n = (x_1,y_1), \ldots,(x_n,y_n)$, let us define a function, 
\vspace{-0.1cm}
\begin{align}
\label{eqn:eta_n_app}
{\eta}_{n}(x) & =\sum_{i=1}^{n} y_i W_{i}^{n}(x)
\end{align}
where $W_{i}^{n} = \ W_{i}^{n}(x,\sigma_n)$ are non-negative weights and $\sum_{i=1}^{n}W_{i}^{n}(x) = 1$. To be precise, $\eta_{n}(x)$ is actually $\eta_{n}(x,\sigma_n)$ in the equation \eqref{eqn:eta_n_app}, it is a function of labeled samples and $x$.
Then a learning rule is defined as,
\begin{align}
\label{eqn:induce_classifier}
g_n(x) = & 
\begin{cases}
1  & \text{if \ } {\eta}_{n}(x) \geq \frac{1}{2}, \ \\
0  & \text{otherwise} 
\end{cases}
\end{align}
%%%%%%%%%%%%%%%%%%%%%%%%%%%%%%%%%%%%%%%%%%%%%%%%%%%%%
%\vspace{0.6cm}
\begin{figure}
\centering
\begin{tikzpicture}
\draw[->,thick] (0,0) --  node[sloped,below] {\small {$\Omega$}} (10,0);
\draw[->,thick] (0,0) --  node[sloped,above,yshift = 0.6cm] {\small {regression function} } (0,5);
\draw[snake=snake,segment aspect=100,segment amplitude=20pt,segment length =150pt] (0,2) --  node[sloped,very near end, left, below] {\small {$\eta$}} (9.5,0.7);
\draw[dashed]  (0,1.5) -- (9.5,1.5) node[very near start,xshift = -1.5cm] {\small {$1/2$}};
\draw[dashed] (3,0) -- (3,2.7);
\draw[<->] (0,2.6) -- (2.99,2.6);
\draw[very thick] (0,3) -- node[very near start,xshift =-.7cm] {\small {$1$}} (3,3) node[very near end,xshift = .5cm] {\ \ \small {$g^*$}};
\draw[dashed] (6,0) -- (6,2.1);
\draw[dashed] (7.5,0) -- (7.5,2.1);
\draw[<->] (6,1.8) -- (7.5,1.8);
\draw[very thick] (6,3) -- (7.5,3);
\draw[very thick] (3,0) -- (6,0);
\draw (3,0) circle (2pt);
\draw (6,0) circle (2pt);
\draw (7.5,0) circle (2pt);
\draw[very thick] (7.5,0) -- (9.5,0);
\end{tikzpicture}
\caption{If regression function $\eta$ is greater than or equals to $1/2$, then Bayes rule $g^*$ (thick black line) is equal to one. If $\eta$ is strictly less than $1/2$ then $g^*$ is equal to zero.} 
\label{fig:bayes_rule} 
\end{figure}
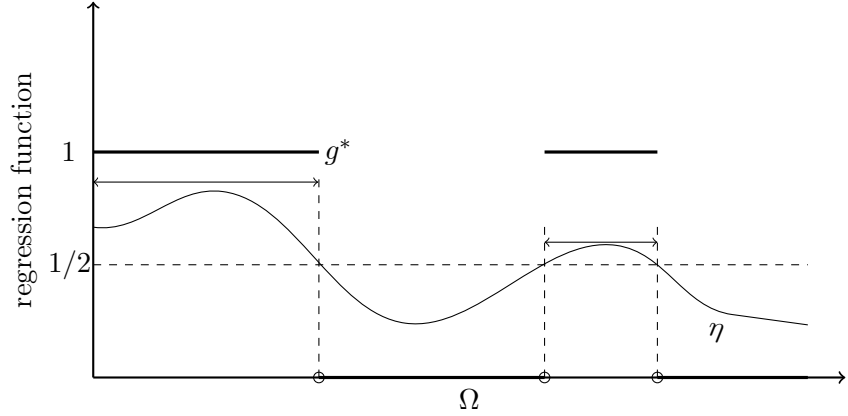
%\vspace{0.6cm}
%%%%%%%%%%%%%%%%%%%%%%%%%%%%%%%%%%%%%%%%%%%%%%%%%%%%%

The above defined learning rule $(g_n)$ is also called as \emph{plug-in rule} \cite{Devroye_Gyorfi_Lugosi_1996} (see figure \ref{fig:bayes_rule}). We do not state explicitly every time but it is important to understand that $g_{n}(x) = g_{n}(x,\sigma_n)$ always. The error probability of $(g_n)$ is stated in Definition~\ref{def:error_prob_rule}.  The following theorem conveys that the expected error probability of $(g_n)$ is more than the error probability of Bayes rule but cannot increase the error probability of Bayes rule by more than twice the average difference between $\eta$ and its approximation $\eta_n$. The proof of the following theorem has been adopted from \cite{Devroye_Gyorfi_Lugosi_1996}.
\begin{theorem}
\label{lem:diff_err}
Let $(\Omega,\rho)$ be a separable metric space and let $g_n$ be a learning rule as in the equation \eqref{eqn:induce_classifier}, then 
\begin{align*}
& \mathbb{P}(g_{n}(X) \neq Y ) - \mathbb{P}(g^*(X) \neq Y)  \leq 2 \mathbb{E}\{ | \eta(X) - \eta_n(X) | \},
\end{align*}
and,
\begin{align*}
& \mathbb{P}(g_{n}(X) \neq Y ) - \mathbb{P}(g^*(X) \neq Y)  \leq 2 \sqrt{\mathbb{E}\{ ( \eta(X) - \eta_n(X) )^2 \}}.
\end{align*}
\end{theorem}
\begin{proof}
In the proof of the Theorem \ref{thm:optimal_bayes}, we have deduced the following difference between  error probabilities,
\begin{align*}
\mathbb{P}(g_{n}(X) \neq Y | X =x ) - \mathbb{P}(g^*(X) \neq Y | X =x ) = 2\bigg|\eta(x) - \frac{1}{2}\bigg|\mathbb{I}_{\{g_{n}(x)\neq g^*(x)\}}.
\end{align*}
We see that if $g_{n}(x)=0, g^*(x)=1$, then $\eta_n(x) < 1/2, \eta(x) \geq 1/2 $. In the other case, if $g_{n}(x)=1, g^*(x)=0$, then $\eta_n(x) \geq 1/2, \eta(x) < 1/2 $. So, $|\eta(x) < 1/2| \leq |\eta(x) - \eta_{n}(x)|$.

Now we take the average of difference between conditional error probabilities, over all $x \in \Omega$,
\begin{align*}
& \mathbb{P}(g_{n}(X) \neq Y ) - \mathbb{P}(g^*(X) \neq Y) \ \\
& = \mathbb{E}\{ \mathbb{P}(g_{n}(X) \neq Y |X=x) -  \mathbb{P}(g^*(X) \neq Y | X =x ) \} \ \\
& = 2 \mathbb{E}\bigg\{ \bigg|\eta(X) - \frac{1}{2}\bigg|\mathbb{I}_{\{g_{n}(X)\neq g^*(X)\}} \bigg\} \\
& \leq 2\mathbb{E}\{ |\eta(X) - \eta_{n}(X)|\} 
\end{align*}
By the Cauchy-Schwarz inequality on the above inequality, we get
\begin{align*}
 \mathbb{P}(g_{n}(X) \neq Y ) - \mathbb{P}(g^*(X) \neq Y) \leq 2\sqrt{\mathbb{E}\{ (\eta(X) - \eta_{n}(X))^2\}}.
\end{align*}
\end{proof}
The Theorem \ref{lem:diff_err} is important because it gives sufficient condition to prove weak consistency, that is, if $\eta_n$ is asymptotically close to $\eta$ then the average error converges to Bayes error. Indeed, we can even deduce a sufficient condition from the Theorem \ref{lem:diff_err} to have strong consistency. A simple corollary to the Theorem \ref{lem:diff_err} is as follows.
\begin{corollary}
The difference between the error probability of learning rule $(g_{n})$ and Bayes rule $g^*$ is bounded by, 
\begin{align*}
 \ell_{\mu}(g_n) - \ell^*_{\mu} \leq 2 \mathbb{E}\{ | \eta(X) - \eta_n(X)| | D_{n} \}, 
\end{align*}
where $D_n$ is a random labeled sample. 
\end{corollary}
%%%%%%%%%%%%%%%%%%%%%%%%%%%%%%%%%%%%%%%%%%%%%%%%%%%%%%%%%%%%%%%%%%%%%%%%%%%%%%%%%%%%%%%
\chapter{The $k$-Nearest Neighbor Rule}
\label{chap:The $k$-nearest neighbor rule}
%%%%%%%%%%%%%%%%%%%%%%%%%%%%%%%%%%%%%%%%%%%%%%%%%%%%%%%%%%%%%%%%%%%%%%%%%%%%%%%%%%%%%%%
Here, we introduce the simplest learning rule called the $k$-nearest neighbor rule. We discuss about some of important results such as Stone's lemma, Stone's theorem and Cover-Hart lemma, together they establish the universal weak consistency of $k$-nearest neighbor rule in finite dimensional normed spaces. We also prove the inconsistency of the $k$-nearest neighbor rule on Davies' example. 
%%%%%%%%%%%%%%%%%%%%%%%%%%%%%%%%%%%%%%%%%%%%%%%%%%%%%
\section{The $k$-nearest neighbor rule}
\label{sec:What is the $k$-nearest neighbor rule?}
%%%%%%%%%%%%%%%%%%%%%%%%%%%%%%%%%%%%%%%%%%%%%%%%%%%%%
The origin of $k$-nearest neighbor rule can be dated back to the work of Fix and Hodges \cite{Fix_Hodges_1951} in 1951. Since then, the $k$-nearest neighbor rule has become a hub of statistical machine learning. 

The $k$-nearest neighbor rule is very simple. The `$k$' in the $k$-nearest neighbor rule is a positive integer and is less than or equal to $n$, the number of data points in a  sample. We explain in the following, the major steps of applying the $k$-nearest neighbor rule algorithmically:
\begin{itemize}
\item Suppose we have a set of $n$ data points, $\{x_1,\ldots,x_n\}$ with their labels $\{y_1,\ldots,y_n\}$. Let $x$ be a new data point and our task is to predict the label of $x$ given the labeled sample.
\item Let $\rho$ be a distance function, not necessarily a metric. Arrange the distances of $x$ to $x_i$ in increasing order,
\begin{align*}
\rho(x_{(1)},x) \leq \rho( x_{(2)},x) \leq \ldots \leq \rho(x_{(n)},x).  
\end{align*}
\item Then the first $k$ data points, $\{x_{(1)},\ldots,x_{(k)}\}$ are called the $k$-nearest neighbors of $x$ with corresponding labels $\{y_{(1)},\ldots,y_{(k)}\}$.
\item The label of $x$ is the most frequent label among $\{y_{(1)},\ldots,y_{(k)}\}$. In other words, we take a majority vote among $\{y_{(1)},\ldots,y_{(k)}\}$ and assign this as the label of $x$.
\end{itemize}

%%%%%%%%%%%%%%%%%%%%%%%%%%%%%%%%%%%%%%%%%%%%%%%%%%%%%
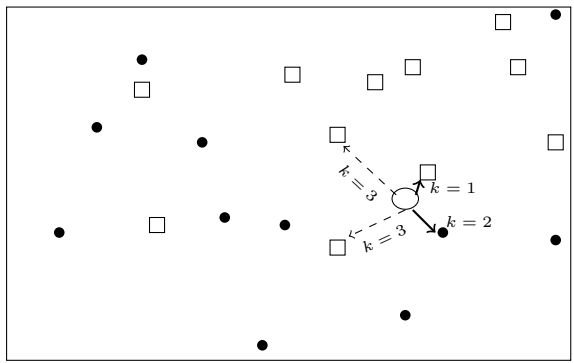
\begin{figure}
\centering
\begin{tikzpicture}
%\hspace{1cm} 
\draw (0.3,0.3) rectangle (7.8,5);
\fill (1,2) circle (2pt);
\fill (1.5,3.4) circle (2pt);
\fill (2.9,3.2) circle (2pt);
\fill (2.1,4.3) circle (2pt);
\fill (3.2,2.2) circle (2pt);
\fill (3.7,0.5) circle (2pt);
\fill (4,2.1) circle (2pt);
\fill (5.6,0.9) circle (2pt);
\fill (6.1,2) circle (2pt);
\fill (7.6,4.9) circle (2pt);
\fill (7.6,1.9) circle (2pt);
\draw (2.2,2) rectangle (2.4,2.2);
\draw (2,3.8) rectangle (2.2,4);
\draw (4.6,1.7) rectangle (4.8,1.9);
\draw (4,4) rectangle (4.2,4.2);
\draw (4.6,3.2) rectangle (4.8,3.4);
\draw (5.6,4.1) rectangle (5.8,4.3);
\draw (5.8,2.7) rectangle (6,2.9);
\draw (7,4.1) rectangle (7.2,4.3);
\draw (6.8,4.7) rectangle (7,4.9);
\draw (5.1,3.9) rectangle (5.3,4.1);
\draw (7.5,3.1) rectangle (7.7,3.3);
\draw (5.6,2.45) ellipse (5pt and 4pt);
\draw[->, thick] (5.74,2.5) -- node[right] {\tiny$k=1$} (5.8,2.7);
\draw[->, thick] (5.7,2.3) -- node[right] {\ \tiny$k=2$} (6,2);
\draw[->, dashed] (5.48,2.5) -- node[sloped,below] {\tiny$k=3$} (4.78,3.15);
\draw[->, dashed] (5.6,2.3) -- node[sloped,below] {\tiny$k=3$} (4.85,1.95);
\end{tikzpicture}
\caption{For $k=1$, `ellipse' has label `rectangle'; for $k=2$, there is a voting tie among `rectangle' and `black dot'; for $k=3$, there is a distance tie (represented by dashed lines) as two data points with label `rectangle' are at equal distance to `ellipse', so we cannot decide which one to choose as the $3$-rd nearest neighbor for `ellipse'. } 
\label{fig:knn_classification} 
\end{figure}
%%%%%%%%%%%%%%%%%%%%%%%%%%%%%%%%%%%%%%%%%%%%%%%%%%%%%

There are two major issues that hinder the implementation of the $k$-nearest neighbor rule (see figure \ref{fig:knn_classification}): Voting ties and distance ties. Voting ties is the difficulty in finding the majority vote among the picked `$k$' labels. If $k$ is an even integer and suppose exactly $k/2$ of $\{y_{(1)},\ldots,y_{(k)}\}$ are 0 and rest are 1, then there is no clear majority vote. Voting ties are usually avoided by taking $k$ to be odd. We pick label 1 as the majority vote in case of a voting tie as stated in the formal definition of the $k$-nearest neighbor rule in the later part of this section.

Distance ties occur when two or more data points are at the same distance to $x$, that is $\rho(x_{i},x) = \rho(x_{j},x)$. This is a problem because there may be many data points at same distance to $x$ and hence it is difficult to choose exactly $k$ nearest neighbors for $x$. The solution to distance ties are complicated and often the consistency is derived under the assumption of no distance ties. To obtain universal consistency, we need a tie-breaker to overcome the problems due to distance ties. 

There are several methods of breaking distance ties, however in this thesis, we discuss only two methods of breaking distance ties. The simplest one is index-based tie-breaking method or in simpler words, breaking distance ties by comparing indices. Given a sample of ordered $n+1$ data points $(x,x_1,\ldots,x_n)$, the $k$-nearest neighbors of $x_i$ are picked from the sample $(x_1,\ldots,x_{i-1},x,x_{i+1},\ldots,x_n)$. Suppose there is a distance tie between $x_j$ and $x$ for $x_i$, that is, $\rho(x_{i},x) = \rho(x_{i},x_{j})$, then we choose $x_{j}$ to be closer to $x_i$ if $x_j \in \{x_1,\ldots,x_{i-1}\}$, otherwise we choose $x$. In Euclidean spaces, tie-breaking by comparing indices is sufficient to avoid any bad situation but the same is not true for general metric spaces. Some issues related to distance ties  for metric spaces with finite Nagata dimension are discussed in section \ref{sec:Consistency with distance ties}. 
The second method is to break distance ties randomly and uniformly. A distance tie basically appears on the sphere, so in case of ties, a point is chosen uniformly on the sphere. Suppose $\rho(X_{i},X) = \rho(X_{i},X_{j})$, then $X$ and $X_{j}$ are chosen with equal probability, that is, $1/(\sharp\{S(X_{i},\rho(X_{i},X))\})$.

Now, we present a formal and mathematical definition of the $k$-nearest neighbor rule. According to \cite{Devroye_Gyorfi_Lugosi_1996}, the $k$-nearest neighbor classification rule belong to the family of plug-in rules, which are defined in the equation \eqref{eqn:induce_classifier}. Intuitively, it is clear that data points which are closer to $x$ will have more influence on $x$ rather than the data points lying far from $x$. This is the fundamental idea of the $k$-nearest neighbor rule, so it is convincing to assign high weights to the data points closer to $x$. 

Given a labeled sample, $\sigma_n = (x_1,y_1), \ldots,(x_n,y_n)$, let $\mathcal{N}_{k}(x)$ denote the set of $k$-nearest neighbors of $x$. Note that, $\sharp \mathcal{N}_{k}(x) = k$. Each data point in $\mathcal{N}_{k}(x)$ is assigned equal and non-zero weight, that is $1/k$. The $k$-nearest neighbor approximation for $\eta$ is,
\begin{align}
\label{eqn:knn_estimate}
\eta_{n}(x) = \frac{1}{k} \sum_{i=1}^{n} \mathbb{I}_{\{x_i \in \mathcal{N}_{k}(x)\}} y_i.
\end{align}
Then, the $k$-nearest neighbor rule is defined as :
\begin{align}
\label{eq:kNN}
g_n(x) =
\begin{cases}
1  & \text{if } \ \eta_{n}(x) \geq  1/2, \\
0  & \text{otherwise} 
\end{cases}
\end{align}
In the equation \eqref{eq:kNN}, the $k$-nearest neighbor rule $g_n$ assigns label $1$ to a data point $x$, if the average weights of the $k$-nearest neighbors of $x$ having label 1 is greater than the average weights of the $k$-nearest neighbors with label 0.

The $k$-nearest neighbor rule is the earliest example of a universally weakly consistent rule. There are two known methods to prove the universal weak consistency of the $k$-nearest neighbor rule: Stone's theorem \cite{Stone_1977} and using the weak Lebesgue-Besicovitch differentiation property \cite{Cerou_Guyader_2006, Devroye_1981}. We discuss the Stone's theorem in detail in the following section as the Stone's argument is our center of focus.

%%%%%%%%%%%%%%%%%%%%%%%%%%%%%%%%%%%%%%%%%%%%%%%%%%%%%
\section{Universal consistency}
\label{sec:Universal consistency}
%%%%%%%%%%%%%%%%%%%%%%%%%%%%%%%%%%%%%%%%%%%%%%%%%%%%%
In this section, we make some mathematical preparations for the proof of Stone's theorem. We start by proving some results that holds in any separable metric space such as Cover-Hart lemma and then using the argument of cones we prove the Stone's lemma and Stone's theorem  in Euclidean spaces. 
%%%%%%%%%%%%%%%%%%%%%%%%%%%%%%%%%%%%%%%%%%%%%%%%%%%%%
\subsection{Cover-Hart lemma and other results for general metric spaces}
\label{subsec:Cover-Hart lemma and other results for general metric spaces}
%%%%%%%%%%%%%%%%%%%%%%%%%%%%%%%%%%%%%%%%%%%%%%%%%%%%%
A separable metric space has some nice properties such as the support of a probability measure in a separable metric space has full measure. A set has full measure if its complement has zero measure. Let $S_{\nu}$ denotes the support of probability measure $\nu$.
\begin{lemma}[\cite{Cover_Hart_1967}]
\label{lem:full_support}
Let $(\Omega,\rho)$ be a separable metric space and let $X$ be distributed according to a probability measure $\nu$ on $\Omega$. Then $\mathbb{P}(X \in S_{\nu}) = \nu(S_{\nu}) = 1$. 
\end{lemma}
\begin{proof}
Let $D$ be a countable dense subset of $\Omega$. For each $x$ in $S_{\nu}^{c}	$, there exists $r >0$ such that $\nu(B(x,r)) =0$. Due to the denseness of $D$, there is an element $a$ in $D$ such that $\rho(x,a)< r/3$. We show that every element $z \in B(a,r/2)$ belongs to $B(x,r)$. By triangle's inequality $\rho(x,z) \leq \rho(x,a) + \rho(a,z) < r/3 + r/2 = 5r/6 <r $. So, $B(a,r/2) \subseteq B(x,r)$, and as $\nu(B(x,r))$ is zero so $\nu(B(a,r/2)) =0$. Observe that $\rho(x,a) <r/3 < r/2$, so $x \in B(a,r/2)$. 

So, for every $x \in S_{\nu}^c$, there is an element $a$ such that $x$ belongs to the open ball $ B(a,r/2)$. We can cover $S_{\nu}^c$ by the countable union of $B(a,r/2),a \in D$ which have measure zero. The countable sub-additivity of $\nu$ implies $S_{\nu}^c$ has zero measure.
\end{proof}
In a separable metric space, the distance of a data point to its $k$-th nearest neighbor can be made small under appropriate values of $k,n$. The Lemma \ref{lem:cover_hart} stated below was originally proved by Cover and Hart for the $k$-nearest neighbor rule in a separable metric space with fixed $k$ (See pages 23, 26 of \cite{Cover_Hart_1967}), moreover, the result is true even if $k$ increases with $n$ but slower than $n$ such that $k/n$ converges to zero (See Lemma 5.1 in \cite{Devroye_Gyorfi_Lugosi_1996} for Euclidean spaces). However, the proof remains same for any separable metric space. The proof of Cover-Hart lemma for any separable metric space presented here has been adapted from Hatko's masters thesis (see Lemma 2.3.4 of \cite{Hatko_2015}), where the proof has been done in separable C-inframetric space \footnote{a inframetric space is a semimetric space satisfying weak-triangle's inequality, $\rho(x,y) \leq C \max\{ \rho(x,z),\rho(z,y)\}$}.
\begin{lemma}[Cover-Hart lemma \cite{Cover_Hart_1967}]
\label{lem:cover_hart}
Let $(\Omega,\rho)$ be a separable metric space. Let $X,X_1,\ldots,X_n$ be an i.i.d. random sample distributed according to $\nu$.  Let $X_{(k)}(X)$ denote the $k$-th nearest neighbor of $X$ among a sample of $n$ points. If $(k_n)$ is a sequence of values such that $  \lim_{n \rightarrow \infty}k_n/n \rightarrow 0$, then
\begin{align*}
\mathbb{P}\bigg(\lim_{n \rightarrow \infty} \rho(X_{(k_n)}(X),X) & = 0 \bigg) = 1. 
\end{align*}
\end{lemma} 
\begin{proof}
 If $x$ is in the support of the measure $\nu$, then for all $\varepsilon >0$, $\nu(B(x,\varepsilon)) > 0$. We note that the distance $\rho(X_{(k_n)}(x),x) > \varepsilon$ if and only if $ \sum_{i=1}^{n} \mathbb{I}_{\{X_i \in B(x,\varepsilon)\}} \allowbreak < k_n$, which is equivalent to
\begin{align}
\label{eqn:k_n_zero}
& \frac{1}{n} \sum_{i=1}^{n} \mathbb{I}_{\{X_i \in B(x,\varepsilon)\}} < \frac{k_n}{n}.
\end{align}
We see that the right side of the equation \eqref{eqn:k_n_zero} goes to 0 as $k_n/n \rightarrow 0$, whereas the left side of the equation \eqref{eqn:k_n_zero} converges to $\nu(B(x,\varepsilon))$ almost surely by the strong law of large numbers. But $\nu(B(x,\varepsilon))$ is strictly positive as $x$ is in the support of $\nu$, therefore, $\rho(X_{(k_n)}(x),x)$ converges to 0 almost surely whenever $x \in S_{\nu}$ and $k_n/n \rightarrow 0$.

If $(k_n)$ is a constant sequence, then $\rho(X_{(k_n)}(x),x)$ is a monotone non-increasing sequence in $n$. We will show that the sequence $\rho(X_{(k_n)}(X),X)$ converges in probability to 0, and hence will converge almost surely. Let $\varepsilon >0$. From the Lemma \ref{lem:full_support}, we have $\mathbb{P}(X \in S_{\nu}) =1$, then 
\begin{align*}
\mathbb{P}( \rho(X_{(k_n)}(X),X) > \varepsilon) & = \mathbb{P}( X \in S_{\nu}) \mathbb{P}( \rho(X_{(k_n)}(X),X) > \varepsilon | X \in S_{\nu}) + \\ & \ \ \ \ \mathbb{P}( X \notin S_{\nu}) \mathbb{P}( \rho(X_{(k_n)}(X),X) > \varepsilon | X \notin S_{\nu}) \ \\
& = \mathbb{P}(\rho(X_{(k_n)}(X),X) > \varepsilon | X \in S_{\nu}) \ \\
& = \mathbb{E}\{ \mathbb{I}_{\{ \rho(X_{(k_n)}(X),X) > \varepsilon \}} | X \in S_{\nu}\},
\end{align*}
which converges to 0, by the Monotone Convergence Theorem. So, $\allowbreak \rho(X_{(k_n)}(X), \allowbreak X)$ converges to 0 in probability and therefore, converges to 0 almost surely. 

Now, suppose $(k_n)$ is a sequence increasing with $n$ but $k_n/n \rightarrow 0$. Let $X_{(k_n,n)}(X)$ denote the $k$-th nearest neighbor of $X$ among the sample $X_1,\ldots, \allowbreak X_n$. Since, $\sup_{m \geq n} \rho(X_{(k_m,m)}(x),x) \rightarrow 0$ almost surely whenever $x$ is in $S_{\nu}$ (as proved above), as a consequence we have $\sup_{m \geq n} \rho(X_{(k_m,m)}(x),x) \rightarrow 0$ almost surely as $n \rightarrow \infty$. 

Let $\varepsilon > 0$. As, $\rho(X_{(k_n,n)}(X),X) \leq \sup_{m \geq n}  \rho(X_{(k_m,m)}(X),X) $, we have
\begin{align*}
\mathbb{P}( \rho(X_{(k_n,n)}(X),X) > \varepsilon) & \leq \mathbb{P}\bigg(\sup_{m \geq n}  \rho(X_{(k_m,m)}(X),X) > \varepsilon \bigg).
\end{align*}
So, it is enough to show that the sequence $\sup_{m \geq n} \rho(X_{(k_m,m)}(X),X)$ converges to 0 almost surely. We follow the similar argument as above by showing that $\sup_{m \geq n} \rho(X_{(k_m,m)}(X),X)$ converges to 0 in probability. As the sequence $\sup_{m \geq n}  \rho(X_{(k_m,m)}(x),x)$ is monotonically non-increasing, this implies that $\sup_{m \geq n} \rho(X_{(k_m,m)}(X),X)$ converges to 0 almost surely as $n \rightarrow \infty$.  We know that $\mathbb{P}( X \in S_{\nu}) =1$ from the Lemma \ref{lem:full_support}, as a result  
\begin{align*}
\mathbb{P}\bigg( \sup_{m \geq n} \rho(X_{(k_m,m)}(X),X) > \varepsilon \bigg) & = \mathbb{P}\bigg(\sup_{m \geq n} \rho(X_{(k_m,m)}(X),X) > \varepsilon | X \in S_{\nu} \bigg) \ \\
& = \mathbb{E}\bigg\{ \mathbb{I}_{\{ \sup_{m \geq n} \rho(X_{(k_m,m)}(X),X) > \varepsilon \}} | X \in S_{\nu}\bigg\},
\end{align*} 
The expectation of indicator functions of events that is, $\mathbb{E}\{ \mathbb{I}_{\{\sup_{m \geq n}  \rho(X_{k_m},X) > \varepsilon \}} \allowbreak | X \in \text{supp}(\nu)\}$ goes to 0  by the Monotone Convergence Theorem. 
\end{proof}

It has been shown in Theorem 5.2 of \cite{Devroye_Gyorfi_Lugosi_1996} that if $k$ is fixed but $n \rightarrow \infty$, then the expected error of the $k$-nearest neighbor rule converges to some constant which is greater than the Bayes error. So, for finite values of $k$, the $k$-nearest neighbor fails to be universally weakly consistent.  Also, two of the conditions of Stone's theorem are satisfied whenever $n,k \rightarrow \infty$ and $k/n \rightarrow 0$ in finite dimensional normed spaces. Therefore, to obtain universal consistency we always consider the limit that $k$ increases with $n$ but slowly, that is, $k/n \rightarrow 0$ as $k,n \rightarrow \infty$. 

The proof of the following result is based on \cite{Devroye_Gyorfi_Lugosi_1996}, where it was proven in Euclidean settings, but the same proof works for every separable metric space. This result shows that the expected difference between the $k$-nearest neighbor approximation $\eta_n$ in the equation \eqref{eqn:knn_estimate} and another approximation $\tilde{\eta}_n$ in the equation \eqref{eqn:sec_app_knn} decreases for large values of $k$.
\begin{lemma}
\label{lem:third_prop_all_metric_sp}
Let $(\Omega,\rho)$ be a separable metric space and let $\nu$ be a probability measure on $\Omega$. Given a labeled sample $\sigma_n$, which takes values in $(\Omega \times \{0,1\})^n$, define a function $\tilde{\eta}_{n}$,
 \begin{align}
\label{eqn:sec_app_knn}
\tilde{\eta}_{n}(x) = \frac{1}{k}\sum_{i=1}^{n} \mathbb{I}_{\{x_i \in \mathcal{N}_{k}(x)\}}\eta(x_i).
\end{align} 
If $k \rightarrow \infty$, then $\mathbb{E}\{ (\eta_{n}(X) - \tilde{\eta}_n(X) )^2 \} \rightarrow 0$.
\end{lemma}
\begin{proof} 
By the definition of $\eta_n$ and $\tilde{\eta}_n$,
\begin{align}
 & (\eta_{n}(X) - \tilde{\eta}_n(X))^2 \\ & =  \bigg(\frac{1}{k} \sum_{i=1}^{n} \mathbb{I}_{\{X_i \in \mathcal{N}_{k}(X)\}} Y_i - \frac{1}{k} \sum_{i=1}^{n} \mathbb{I}_{\{X_i \in \mathcal{N}_{k}(X)\}} \eta(X_i) \bigg)^2 \nonumber \\
& = \bigg( \frac{1}{k} \sum_{i=1}^{n} \mathbb{I}_{\{X_i \in \mathcal{N}_{k}(X)\}} (Y_i - \eta(X_i) ) \bigg)^2 \nonumber \\
& = \frac{1}{k^2} \sum_{i =1}^{n} \sum_{j=1}^{n} \mathbb{I}_{\{X_i \in \mathcal{N}_{k}(X)\}} \mathbb{I}_{\{X_j \in \mathcal{N}_{k}(X)\}} (Y_i - \eta(X_i) )(Y_j - \eta(X_j) )
\label{eqn:ind_expec}
\end{align}
We know that $\mathbb{E}\{ \eta(X_i)\} = \mathbb{E}\{ \mathbb{E}\{Y_i| X = X_i\}\} = \mathbb{E}\{Y_i\}$. And if $i \neq j$, then $(X_i,Y_i)$ and $(X_j,Y_j)$ are independent of each other.  So,  
\begin{align*}
& \mathbb{E}\bigg\{ \frac{1}{k^2} \sum_{i,j =1, i \neq j}^{n}\mathbb{I}_{\{X_i \in \mathcal{N}_{k}(X)\}} \mathbb{I}_{\{X_j \in \mathcal{N}_{k}(X)\}} (Y_i - \eta(X_i) )(Y_j - \eta(X_j) )\bigg\}=  0.
\end{align*}
Since $(Y_i - \eta(X_i))^2$ is bounded above by one, we have 
\begin{align*}
 \mathbb{E}\{(\eta_{n}(X) - \tilde{\eta}_n(X))^2\} & = \mathbb{E}\bigg\{\frac{1}{k^2} \sum_{i =1}^{n} \mathbb{I}_{\{X_i \in \mathcal{N}_{k}(X)\}} (Y_i - \eta(X_i) )^2 \bigg\} \ \\
& \leq \mathbb{E}\bigg\{ \frac{1}{k^2}  \sum_{i =1}^{n}\mathbb{I}_{\{X_i \in \mathcal{N}_{k}(X)\}} \bigg\} \ \\
& = \mathbb{E}\bigg\{ \frac{1}{k} \sum_{i =1}^{n} \frac{1}{k} \mathbb{I}_{\{X_i \in \mathcal{N}_{k}(X)\}} \bigg\} \ \\
& = 1/k,
\end{align*}
where the last equality is true because $\mathcal{N}_{k}(X)$ contains exactly $k$ data points from the sample. In case of distance ties, we break ties and choose $k$ data points for $\mathcal{N}_k$.
\end{proof}
In fact, we prove in the Lemma \ref{lem:eta_tilde_reg} that the expected difference between $\tilde{\eta}_n$ and $\eta$ can be bounded above by the expected difference of $\eta$ and an uniformly continuous function with the help of Luzin's theorem (see Theorem \ref{thm:Luzin_theorem}). A initial part of the following proof is based on \cite{Devroye_Gyorfi_Lugosi_1996}, where it was drafted for Euclidean spaces. 
\begin{lemma}
\label{lem:eta_tilde_reg}
Let $\nu$ be a probability measure on $\Omega$, where $(\Omega,\rho)$ is a separable metric space.
Let $\tilde{\eta}_n$ be same as defined in the equation \eqref{eqn:sec_app_knn} of Lemma \ref{lem:third_prop_all_metric_sp}. Let $\varepsilon > 0$, then there exists a set $K \subseteq \Omega$ and a uniformly continuous function $\eta^{*}$ on $\Omega$ such that, $\mathbb{E}\{( \tilde{\eta}_n(X) - \eta(X))^2\}$ is less than or equal to 
\begin{align*}
\mathbb{E}\bigg\{ \frac{1}{k} \sum_{i =1}^{n} \mathbb{I}_{\{X_i \in \mathcal{N}_k(X)\}} ( \eta^*(X_i) - \eta(X_i) )^{2} \bigg| X\in K, X_i \in U \bigg\}  + 12\varepsilon, 
\end{align*}
where $U = \Omega \setminus K$.
\end{lemma}
\begin{proof}
By the Jensen's inequality we have,
\begin{align*}
( \tilde{\eta}_{n}(X) - \eta(X) )^2 & = \bigg( \frac{1}{k} \sum_{i=1}^{n} \mathbb{I}_{\{X_i \in \mathcal{N}_{k}(X)\}}\eta(X_i) - \eta(X) \bigg)^2 \ \\
& = \bigg( \frac{1}{k} \sum_{i=1}^{n} \mathbb{I}_{\{X_i \in \mathcal{N}_{k}(X)\}}( \eta(X_i) - \eta(X) ) \bigg)^2 \ \\
& \leq \frac{1}{k} \sum_{i=1}^{n} \mathbb{I}_{\{X_i \in \mathcal{N}_{k}(X)\}} (\eta(X_i) - \eta(X) )^2 .
\end{align*}
Then, we have 
\begin{align*}
\mathbb{E}\bigg\{ ( \tilde{\eta}_{n}(X) - \eta(X) )^2 \bigg\} & \leq \mathbb{E}\bigg\{ \frac{1}{k} \sum_{i=1}^{n} \mathbb{I}_{\{X_i \in \mathcal{N}_{k}(X)\}} (\eta(X_i) - \eta(X) )^2 \bigg\}.
\end{align*}
Given $\varepsilon >0$, the Luzin's theorem (see chapter 7 of \cite{Folland_1999} or see Theorem \ref{thm:Luzin_theorem}) implies that there exists a compact set $K \subseteq \Omega$ such that $\eta|_{K}$ is continuous and $\nu(\Omega \setminus K) < \varepsilon$. Let $U = \Omega \setminus K$. Due to Lemma~\ref{lem:uni_cts_lip_cts} and Lemma~\ref{lem:lip_cts_lip_cts},  we can extend $\eta|_{K}$ to a uniformly continuous function $\eta^{*} : \Omega \rightarrow [0,1]$ such that $\eta^*(X)  =  \eta(X) $ for $X \in K$ and $(\eta^*(X) - \eta(X))^2 \leq 1 $ whenever $X \notin K$.

Using the inequality $(a+b+c)^2 \leq 3 (a^2 + b^2 + c^2)$, where $a,b,c$ are real numbers, we have

\begin{align}
\label{eqn:three_bounds}
& \mathbb{E}\bigg\{ \frac{1}{k} \sum_{i=1}^{n} \mathbb{I}_{\{X_i \in \mathcal{N}_{k}(X)\}} (\eta(X_i) - \eta(X) )^2 \bigg\} \nonumber \\
& = \mathbb{E}\bigg\{ \frac{1}{k} \sum_{i=1}^{n} \mathbb{I}_{\{X_i \in \mathcal{N}_{k}(X)\}}(\eta(X_i)-\eta^*(X_i) + \eta^*(X_i)-\eta^*(X) + \eta^*(X)-\eta(X) )^2 \bigg\} \nonumber \\
& = 3 \mathbb{E}\bigg\{ \frac{1}{k} \sum_{i=1}^{n} \mathbb{I}_{\{X_i \in \mathcal{N}_{k}(X)\}}(\eta(X_i) - \eta^*(X_i) )^2 \bigg\} + 3 \mathbb{E}\bigg\{ \frac{1}{k} \sum_{i=1}^{n} \mathbb{I}_{\{X_i \in \mathcal{N}_{k}(X)\}} \nonumber \\ & \ \ \ \ (\eta^*(X_i) - \allowbreak \eta^*(X))^2 \bigg\}  +  3 \mathbb{E}\bigg\{ \frac{1}{k} \sum_{i=1}^{n} \mathbb{I}_{\{X_i \in \mathcal{N}_{k}(X)\}} (\eta^*(X) - \eta(X))^2 \bigg\}.
\end{align}
We bound the three expressions in the right-hand side of the above inequality in the following way,
\begin{itemize} 
\item Third term of the equation \eqref{eqn:three_bounds}: 
As $\eta^*$ and $\eta$ are equal on $K$ and  $\frac{1}{k} \sum_{i=1}^{n} \mathbb{I}_{\{X_i \in \mathcal{N}_{k}(X)\}} = 1$ after breaking the distance ties, we have
\begin{align*}
& \mathbb{E}\bigg\{ \frac{1}{k} \sum_{i=1}^{n} \mathbb{I}_{\{X_i \in \mathcal{N}_{k}(X)\}} (\eta^*(X) - \eta(X))^2 \bigg\} \\ & \leq  \mathbb{E}\{(\eta^*(X) - \eta(X))^2 | X \in K\} +  \mathbb{E}\{(\eta^*(X) - \eta(X))^2| X \in U \} \\
& \leq \nu(U)  < \varepsilon.
\end{align*}
\item Second term of the equation \eqref{eqn:three_bounds}: We first divide the expectation in two disjoint cases: $\rho(X_{i},X) > \delta$ and $\rho(X_{i},X)\leq \delta$. As $\eta^{*}$ is a uniformly continuous function, given $\varepsilon >0$ there exists $\delta >0$ such that $ (\eta^*(X_i) - \eta^*(X))^2 \leq \varepsilon$  whenever $\rho(X,X_i) \leq \delta$. For the second case, we use Cover-Hart lemma (Lemma \ref{lem:cover_hart}). As $k/n$ goes to zero, by Cover-Hart lemma the distance between $X$ and its $k$-th nearest neighbor $X_{(k)}$ will tend to zero almost surely. That is, for all $\varepsilon > 0$, $\mathbb{P}(\rho(X_{(k)},X)> \delta) \rightarrow \varepsilon$. Note that, all $k-1$-nearest neighbors are closer to $X$ than $X_{k}$. This implies that $\mathbb{E}\{(1/k)\sum_{i=1}^{n} \mathbb{I}_{\{\rho(X_{i},X) > \delta)\}} \} < \varepsilon$.

So, we have 
\begin{align*}
&  \mathbb{E}\bigg\{ \frac{1}{k} \sum_{i=1}^{n} \mathbb{I}_{\{X_i \in \mathcal{N}_{k}(X)\}} (\eta^*(X_i) - \eta^*(X))^2 \bigg\}  \ \\ & =  \mathbb{E}\bigg\{ \frac{1}{k} \sum_{i=1}^{n} \mathbb{I}_{\{X_i \in \mathcal{N}_{k}(X)\}}\mathbb{I}_{\{\rho(X,X_i) > \delta \}} (\eta^*(X_i) - \eta^*(X))^2 \bigg\}  +  
\end{align*}
\begin{align*}
& \  \ \ \  \mathbb{E}\bigg\{ \frac{1}{k} \sum_{i=1}^{n} \mathbb{I}_{\{X_i \in \mathcal{N}_{k}(X)\}} \allowbreak \mathbb{I}_{\{\rho(X,X_i) \leq \delta \}}(\eta^*(X_i) - \eta^*(X))^2 \bigg\}  \ \\
& \leq  2 \varepsilon.
\end{align*}

\end{itemize}
Now, we will analyze the first term of the equation \eqref{eqn:three_bounds}. Let $Z_i$ denote the expression $ \mathbb{I}_{\{X_i \in \mathcal{N}_k(X)\}} ( \eta^*(X_i) - \eta(X_i) )^{2}$. We use $Z_i$ here for easy calculations. Then, the first term in the equation \eqref{eqn:three_bounds} looks like,
\begin{align}
\label{eqn:ab2}
\mathbb{E}\bigg\{ \frac{1}{k} \sum_{i=1}^{n} \mathbb{I}_{\{X_i \in \mathcal{N}_{k}(X)\}} (\eta(X_i) - \eta^*(X_i) )^2 \bigg\} & =   \mathbb{E}\bigg\{ \frac{1}{k} \sum_{i =1}^{n} Z_{i} \bigg\}. 
\end{align} 
We further divide the equation \eqref{eqn:ab2} into two cases,
\begin{align}
\label{eqn:ab1}
\mathbb{E}\bigg\{ \frac{1}{k} \sum_{i =1}^{n} Z_i  \bigg\} & \leq   \ \mathbb{E}\bigg\{ \frac{1}{k} \sum_{i =1}^{n} Z_i \bigg| X \in U \bigg\} + \mathbb{E}\bigg\{ \frac{1}{k} \sum_{i =1}^{n} Z_i \bigg| X \in K \bigg\}.
\end{align} 
The value of $(1/k)\sum_{i =1}^{n} Z_i$ is at most one, so the first term on the right hand side of the equation \eqref{eqn:ab1} is bounded above by $\nu(U) < \varepsilon$. For the second term of the equation \eqref{eqn:ab1}, we again consider two disjoint cases,
\begin{align*}
\mathbb{E}\bigg\{ \frac{1}{k} \sum_{i =1}^{n} Z_i \bigg| X \in K \bigg\}  & = \mathbb{E}\bigg\{ \frac{1}{k} \sum_{i =1}^{n} Z_i \bigg| X,X_i \in K \bigg\} + \\ & \ \ \ \ \mathbb{E}\bigg\{ \frac{1}{k} \sum_{i =1}^{n} Z_i \bigg| X \in K, X_i \in U \bigg\}. \ 
\end{align*}
If $X_i \in K$, then $\eta(X_i) = \eta^*(X_i)$ and so we have
\begin{align*}
\mathbb{E}\bigg\{ \frac{1}{k} \sum_{i =1}^{n} Z_i \bigg| X,X_i \in K \bigg\}  = 0. 
\end{align*}
Now summing all the bounds calculated above, we obtain
\begin{align*}
& \mathbb{E}\{( \tilde{\eta}_n(X) - \eta(X))^2\} \\ & \leq  \mathbb{E}\bigg\{ \frac{1}{k} \sum_{i =1}^{n} Z_i \bigg| X \in K, X_i \in U \bigg\} + 3(\varepsilon +  2\varepsilon + \varepsilon) \ \\
& = \mathbb{E}\bigg\{ \frac{1}{k} \sum_{i =1}^{n} \mathbb{I}_{\{X_i \in \mathcal{N}_k(X)\}} ( \eta^*(X_i) - \eta(X_i) )^{2} \bigg| X \in K, X_i \in U \bigg\} + 12\varepsilon.
\end{align*}
\end{proof}
It is important to recall again that all the results in Subsection~\ref{subsec:Cover-Hart lemma and other results for general metric spaces} holds for any separable metric space. 
%%%%%%%%%%%%%%%%%%%%%%%%%%%%%%%%%%%%%%%%%%%%%%%%%%%%%
\subsection{Stone's theorem}
\label{sec:Consistency using Stone's theorem}
%%%%%%%%%%%%%%%%%%%%%%%%%%%%%%%%%%%%%%%%%%%%%%%%%%%%%
Charles Stone proved that the $k$-nearest neighbor rule is universally consistent in an Euclidean space. The result can be extended to finite dimensional normed spaces without much difficulty, see for example Duan's thesis \cite{Duan_2014}. Here, we discuss the proof of Stone's theorem in Euclidean spaces using the cones argument adopted from section 5.3 of \cite{Devroye_Gyorfi_Lugosi_1996}.

Let $\theta \in (0,\pi/2)$. A \emph{cone} $C(x,\theta)$, around an element $x \in \mathbb{R}^d$ of angle $\theta$, is the set of all $y$ from $\mathbb{R}^d$ such that the angle between $x$ and $y$ is less than or equal to $\theta$, that is, 
$$C(x,\theta) = \bigg\{y \in \mathbb{R}^{d}:\frac{\langle x,y \rangle}{||x||\ ||y||} \geq \cos(\theta)\bigg\}, $$ 
where $\langle x,y \rangle = x^t.y$ is the dot product of $x$ and $y$. A cone of angle $\pi/6$ is shown in figure \ref{fig:cones}.
\begin{lemma}
If $\theta \in (0,\pi/6]$, then the cone $C(x,\theta)$ has the following geometrical property: for $x_1,x_2 \in C(x,\theta)$,
\begin{align*}
||x_1|| < ||x_2|| \Rightarrow ||x_1-x_2|| < ||x_2||.
\end{align*}
\end{lemma}
\begin{proof}
If $x_1,x_2$ is in $C(x,\theta)$, then the angle of $x_1$ and $x_2$ with $x$, respectively, is at most $\theta$. From the figure \ref{fig:cones}, we see that the angle between $x_1$ and $x_2$ is at most $2\theta$,
\begin{align*}
\cos(2\theta) & = 2 \cos^2(\theta) -1 \ \\
& \leq 2 \frac{\langle x,x_1 \rangle}{||x||\  ||x_1||} \frac{\langle x,x_2 \rangle}{||x||\  ||x_2||} -1 \ \\
& =  \frac{x_1^t.x_2 \ x^t.x}{||x||^2\  ||x_1||\  ||x_2||} + \frac{x_1^t.x_2 \ x^t.x}{||x||^2\  ||x_1||\  ||x_2||} -1 \ \\
& \leq \frac{\langle x_1,x_2 \rangle}{||x_1||\  ||x_2||},
\end{align*}
where the last inequality is due to Cauchy-Schwarz inequality. We see that if $||x_1|| < ||x_2||$, then $\frac{||x_1||}{||x_2||} < 1$, which gives $\frac{||x_1||^2}{||x_2||^2} +1  <  \frac{||x_1||}{||x_2||} +1$.
If $\theta \leq \pi/6$, then $\cos(2\theta) \geq 1/2$, so we have 
\begin{align*}
||x_1 -x_2||^2 & = \langle x_1-x_2,x_1-x_2\rangle \ \\
& = ||x_1||^2 + ||x_2||^2 - 2 \langle x_1,	x_2\rangle \ \\
& = ||x_1||^2 + ||x_2||^2 - 2 \frac{\langle x_1,x_2\rangle}{||x_1||\ ||x_2||} ||x_1||\ ||x_2|| \\
& \leq ||x_1||^2 + ||x_2||^2 - 2 \cos(2 \theta) ||x_1||\  ||x_2|| \ \\
& \leq ||x_1||^2 + ||x_2||^2 - ||x_1||\  ||x_2|| \ \\
& = ||x_2||^2 \bigg(  \frac{||x_1||^2}{||x_2||^2} + 1 - \frac{||x_1||}{||x_2||} \bigg) \ \\
& < ||x_2||^2.
\end{align*}
\end{proof}
%%%%%%%%%%%%%%%%%%%%%%%%%%%%%%%%%%%%%%%%%%%%%%%%%%%%%
\begin{figure}
\centering
\begin{tikzpicture}
%\hspace{1cm} 
\draw (4.5cm,0.9cm) arc (30:60:4.95cm);
\fill (4.1,1) circle (1.5pt) node[left] {\small {$x_1$}};
\draw (4.2cm,0.85cm) arc (30:60:4.6cm);
\fill (3.2,2.4) circle (1.5pt) node[above] {\small {$x_2$}};
\draw (1cm,0.199cm) arc (30:60:0.59cm) node[right] {\tiny {$\pi/6$}} ;
\draw (1cm,0.5cm) arc (30:60:0.69cm) node[right] {\tiny {$\pi/6$}};
\draw (0,0) -- (5,1);
\draw (0,0) -- (4,2);
\draw (0,0) -- (3,3);
\fill (2.2,1.1) circle (1.5pt) node[below] {\small {$x$}};

\end{tikzpicture}
\caption{A cone of angle $\pi/6$ at $x$ has the geometrical property: if $\lVert x_1 \rVert < \lVert x_2 \rVert$, then $\lVert x_1 - x_2 \rVert < \lVert x_2 \rVert$. } 
\label{fig:cones} 
\end{figure}
%%%%%%%%%%%%%%%%%%%%%%%%%%%%%%%%%%%%%%%%%%%%%%%%%%%%%

The following covering lemma (see pp. 67-68, Lemma 5.5 of \cite{Devroye_Gyorfi_Lugosi_1996}) for $\mathbb{R}^d$ is true for any fixed positive value of $\theta < \pi/2$.
\begin{lemma}[Covering lemma for $\mathbb{R}^d$ \cite{Devroye_Gyorfi_Lugosi_1996}]
\label{lem:covering_r}
Let $(\mathbb{R}^{d},||.||)$ be an Euclidean space. Let $\theta \in (0,\pi/2)$, then there exists a constant $\beta_d$, depending only on the dimension $d$ and norm, such that there is a finite subset $\{z_1,\ldots,z_{\beta_{d}} \}$ of $\mathbb{R}^{d}$ and the finite union of cones $C(z_i,\pi/6)$ covers $\mathbb{R}^d$. The constant $\beta_d$ is less than or equal to $ \bigg(1 + \frac{1}{\sin(\theta/2)}\bigg)^d -1$.
\end{lemma}
Now we present the proof of the important geometric Stone's lemma for Euclidean spaces using the beautiful argument of cones, as  given in \cite{Devroye_Gyorfi_Lugosi_1996}.  
\begin{lemma}[Geometric Stone's lemma \cite{Devroye_Gyorfi_Lugosi_1996}] 
\label{stone_lemma}
Let $x,x_1,\ldots,x_n$ be a sample of $n+1$ points in an Euclidean space $(\mathbb{R}^d,||.||)$. Suppose that $x_i \neq x_j$ for $i \neq j$. Then, $x$ can be the $k$-nearest neighbor for at most $k \beta_d$ number of data points $x_i$,
\begin{align*}
\sum_{i=1}^{n} \mathbb{I}_{\{ x \in \mathcal{N}_{k}(x_i)\}} \leq \ k \beta_{d},  
\end{align*}	
where $\beta_{d}$ is a constant as given in Lemma \ref{lem:covering_r}. 
\end{lemma}
\begin{proof} 
By the Lemma \ref{lem:covering_r}, we can cover $\mathbb{R}^{d}$ by $\beta_{d}$ numbers of cones at $z_i$ of angle $\theta \leq \pi/6$. Let $x + C(z_i,\theta)$ be the translation of $C(z_i,\theta)$ and it still covers $\mathbb{R}^d$ due to translation invariance property of norm. So, $\mathbb{R}^d =\cup_{i=1}^{\beta_d} (x + C(z_i,\theta))$. See figure \ref{fig:stone_lemma}. The data points $x_i$ are lying around $x$ belonging to some set $(x +C(z_i,\theta))$. In each set $(x + C(z_i,\theta))$, we mark $x_i$ which are $k$-nearest neighbors of $x$. If there are fewer than $k$ points in a particular set, then we mark all the points in that set. 
From the figure \ref{fig:stone_lemma}, we can see that marked points form an insulation belt around $x$ separating unmarked points and $x$. If a data point $x_j$ is unmarked then there are at least $k$ data points in that cone which are closer to $x_j$ than $x$, after breaking distance ties by comparing indices. So, we can assume that $||x-x_i|| < ||x-x_j||$ for every marked point $x_i$ in that particular cone. By the geometrical property of cones, we have $||x_i-x_j|| < ||x-x_j||$, which means that $x_i$ is closer to $x_j$ than $x$. This implies that if $x_j$ is not marked then $x$ cannot be the $k$-nearest neighbor of $x_j$. Hence, we need to count the marked points. There are $\beta_d$ sets and in each set there are at most $k$ marked points, so there are at most $k\beta_d$ marked points.  
\begin{align*}
\sum_{i=1}^{n} \mathbb{I}_{\{x \in \mathcal{N}_{k}(x_i)\}}  & \leq \sharp\{x_i: x_i \text{ is marked}\} \ \\
& \leq k \beta_d.
\end{align*}
\end{proof}
%%%%%%%%%%%%%%%%%%%%%%%%%%%%%%%%%%%%%%%%%%%%%%%%%%%%%
\begin{figure}[!ht]
\centering
\begin{tikzpicture}
\draw (0,0) -- (3,0.8);
\draw (0,0) -- (1,2.5);
\draw (0,0) -- (-3,0.8);
\draw (0,0) -- (-1,2.5);
\draw (0,0) -- (-2.9,-0.8);
\draw (0,0) -- (-1,-2.5);
\draw (0,0) -- (3,-0.8);
\draw (0,0) -- (1,-2.5);
\fill (0,0) circle (1.5pt) node[yshift=-0.05cm,below] {\small {$x$}};
\fill (0.5,0.7) circle (1.5pt);
\fill (1.1,0.6) circle (1.5pt);
\fill (1.9,0.9) circle (1.5pt);
\fill (1.3,1.7) circle (1.5pt);
\fill (1,1) circle (1.5pt);
\draw (1,1) ellipse (3pt and 2pt);
\draw (0.5,0.7) ellipse (3pt and 2pt);
\draw (1.1,0.6) ellipse (3pt and 2pt);

\fill (-2,1) circle (1.5pt);
\fill (-1.7,1.3) circle (1.5pt);
\fill (-0.5,0.5) circle (1.5pt);
\fill (-1,1.5) circle (1.5pt);
\fill (-1,0.7) circle (1.5pt);
\fill (-1.5,1) circle (1.5pt);
\fill (-1.5,0.7) circle (1.5pt);
\draw (-0.5,0.5) ellipse (3pt and 2pt);
\draw (-1,0.7) ellipse (3pt and 2pt);
\draw (-1.5,0.7) ellipse (3pt and 2pt);

\fill (0.7,-0.9) circle (1.5pt);
\fill (1.7,-0.8) circle (1.5pt);
\fill (1,-0.7) circle (1.5pt);
\fill (1.5,-1.5) circle (1.5pt);
\draw (0.7,-0.9) ellipse (3pt and 2pt);
\draw (	1,-0.7) ellipse (3pt and 2pt);
\draw (1.7,-0.8) ellipse (3pt and 2pt);

\fill (-0.2,0.9) circle (1.5pt);
\fill (0.1,0.8) circle (1.5pt);
\fill (-0.3,1.8) circle (1.5pt) node[below] {\small {$x_j$}};
\fill (0.3,1.8) circle (1.5pt) node[below] {\small {$x_i$}};
\draw (-0.2,0.9)  ellipse (3pt and 2pt);
\draw (0.1,0.8)  ellipse (3pt and 2pt);
\draw (0.3,1.8)  ellipse (3pt and 2pt);

\fill (2.2,-0.01) circle (1.5pt);
\fill (1,0) circle (1.5pt);
\fill (1.5,-0.3) circle (1.5pt);
\fill (1.9,-0.4) circle (1.5pt);
\fill (2.5,-0.2) circle (1.5pt);
\draw (1,0)  ellipse (3pt and 2pt);
\draw (1.5,-0.3)  ellipse (3pt and 2pt);
\draw (1.9,-0.4)  ellipse (3pt and 2pt);

\fill (0.1,-1.5) circle (1.5pt);
\fill (0.2,-1) circle (1.5pt);
\draw (0.1,-1.5)  ellipse (3pt and 2pt);
\draw (0.2,-1)  ellipse (3pt and 2pt);

\fill (-2.5,-0.4) circle (1.5pt);
\fill (-2,0) circle (1.5pt);
\fill (-1.5,-0.1) circle (1.5pt);
\fill (-1.1,-0.2) circle (1.5pt);
\draw (-1.1,-0.2)  ellipse (3pt and 2pt);
\draw (-2,0)  ellipse (3pt and 2pt);
\draw (-1.5,-0.1)  ellipse (3pt and 2pt);

\fill (-1.3,-1) circle (1.5pt);
\fill (-1,-0.5) circle (1.5pt);
\fill (-1,-1.5) circle (1.5pt);
\draw (-1.3,-1)  ellipse (3pt and 2pt);
\draw (-1,-0.5)  ellipse (3pt and 2pt);
\draw (-1,-1.5)  ellipse (3pt and 2pt);

\draw (0.35cm,0.1cm) arc (30:60:0.6cm) node[right] {\small {$\pi/4$}};

\end{tikzpicture}

\caption{Illustration of Stone's lemma: Cover $\mathbb{R}^d$ by cones of angle $\pi/8$ at $x$. In each cone, mark at most $(k=3)$ nearest neighbors of $x$ (eye-shaped data points). The points $x_i$ and $x_j$ are at same distance to $x$ and $i <j$ so by the index based tie-breaking we choose $x_i$ as the 3-rd nearest neighbor of $x$ in that particular cone.} 
\label{fig:stone_lemma} 
\end{figure}
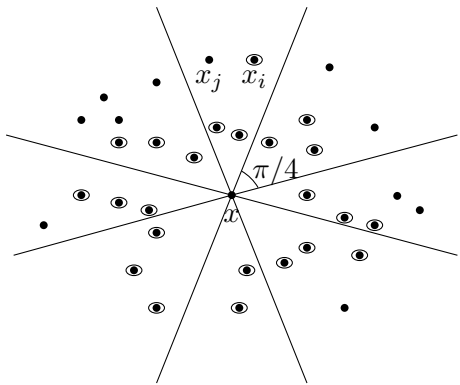
%%%%%%%%%%%%%%%%%%%%%%%%%%%%%%%%%%%%%%%%%%%%%%%%%%%%%

We are now ready to present the classical Stone's theorem. As a result of geometric Stone's lemma, the conditions of Stone's theorem are satisfied, which establishes the universal weak consistency of the $k$-nearest neighbor rule in Euclidean spaces. 
\begin{theorem}[Stone's theorem \cite{Stone_1977,Devroye_Gyorfi_Lugosi_1996}]
\label{thm:stone_euclidean}
Let $g_n$ be the $k$-nearest neighbor rule on Euclidean space $(\mathbb{R}^d,||.||)$. If $k/n \rightarrow 0$ as $n,k \rightarrow \infty$, then the expected error probability of $g_n$ converges to Bayes error. In other words, the $k$-nearest neighbor rule is universally weakly consistent.
\end{theorem}
\begin{proof}
From the Theorem \ref{lem:diff_err}, it is sufficient to show that 
\begin{align*}
\mathbb{E}\{(\eta(X) - \eta_n(X))^2\} \rightarrow 0.
\end{align*}
Using the inequality $(a+b)^2 \leq 2a^2 + 2 b^2$, where $a,b$ are real numbers, 
\begin{align*} 
\mathbb{E}\{(\eta_{n}(X) - \eta(X))^2\} & = \mathbb{E}\{(\eta_{n}(X) - \tilde{\eta}(X) + \tilde{\eta}(X) - \eta(X))^2\} \ \\
& \leq 2 \mathbb{E}\{(\eta_{n}(X) - \tilde{\eta}(X))^2\} + 2\mathbb{E}\{(\tilde{\eta}(X) - \eta(X))^2\}
\end{align*}
The Lemma \ref{lem:third_prop_all_metric_sp} implies that the first term in the above equation goes to zero when $k \rightarrow \infty$. From the Lemma \ref{lem:eta_tilde_reg}, we have an upper bound on the second term. Then we exchange $X$ and $X_i$ , as $X,X_i$ are i.i.d., and use the (Stone's) Lemma \ref{stone_lemma}. We also use the fact that $( \eta^*(X) - \eta(X) )^{2}$ is bounded above by one, where $\eta^*$ is a uniformly continuous function as stated in Lemma \ref{lem:eta_tilde_reg}. We have (by Lemma \ref{lem:eta_tilde_reg}),
\begin{align*}
& \mathbb{E}\{( \tilde{\eta}_n(X) - \eta(X))^2\} \\
& \leq \mathbb{E}\bigg\{ \frac{1}{k} \sum_{i =1}^{n} \mathbb{I}_{\{X_i \in \mathcal{N}_k(X)\}} ( \eta^*(X_i) - \eta(X_i) )^{2} \bigg| X\in K, X_i \in U \bigg\}  + 12\varepsilon \ \\
& = \mathbb{E}\bigg\{ \frac{1}{k} \sum_{i =1}^{n} \mathbb{I}_{\{X \in \mathcal{N}_k(X_i)\}} ( \eta^*(X) - \eta(X) )^{2} \bigg| X_i\in K, X \in U \bigg\}  + 12\varepsilon \ \\
& \leq  \beta_d \mathbb{E}\{( \eta^*(X) - \eta(X) )^{2} | X \in U \}  + 12\varepsilon \ \\
& \leq \beta_d \nu(U) + 12\varepsilon \ \\
& \leq \beta_d \varepsilon + 12\varepsilon.
\end{align*} 
\end{proof}
We observe that the Stone's lemma \ref{stone_lemma} is the heart of the Stone's theorem. If the Stone's lemma holds for any general metric space, then the Stone's theorem holds and hence we achieve universal weak consistency in any general metric space. But, the argument of cones which has been used to prove Stone's lemma is extremely restricted to finite dimensional Euclidean spaces. In general, Stone's lemma is known to be true for any finite dimensional normed space \cite{Duan_2014}. Indeed, the proof of Stone's lemma is limited to finite dimensional normed spaces. In the next chapter, we make an attempt to generalize Stone's lemma for spaces with finite Nagata dimension. 

There is another method worked out by C{\'e}rou and Guyader \cite{Cerou_Guyader_2006} to prove the universal weak consistency in more general metric spaces. They showed that the weak Lebesgue-Besicovitch differentiation property of a metric space is sufficient to guarantee the universal weak consistency.
\begin{theorem}[C{\'e}rou and Guyader \cite{Cerou_Guyader_2006}]
\label{bdt_knn}
Let $(\Omega,\rho)$ be a separable metric space. Suppose that $\Omega$ satisfies the weak Lebesgue-Besicovitch differentiation property. Then, the $k$-nearest neighbor rule is universally weakly consistent.
\end{theorem}
The above theorem by C{\'e}rou and Guyader, along with the result by Preiss (Theorem \ref{thm:Priess_theorem}) imply that the $k$-nearest neighbor rule is universally weakly consistent in a complete separable metric space having sigma-finite metric dimension.
The main aim of this thesis is to reprove this result directly, by using the means of statistical learning theory while trying to imitate the proof by Stone in as much as possible. We investigate to what extent the geometric Stone's lemma can be adapted in such metric spaces, and make a number of interesting observations.  

%%%%%%%%%%%%%%%%%%%%%%%%%%%%%%%%%%%%%%%%%%%%%%%%%%%%%
\section{An example of inconsistency}
\label{sec:An example of inconsistency}
%%%%%%%%%%%%%%%%%%%%%%%%%%%%%%%%%%%%%%%%%%%%%%%%%%%%%
In this section, we first discuss the example by Davies in detail and then prove the inconsistency of the $k$-nearest neighbor rule on this example. 
%%%%%%%%%%%%%%%%%%%%%%%%%%%%%%%%%%%%%%%%%%%%%%%%%%%%%
\subsection{Davies' example}
\label{subsec:Davies' example}
%%%%%%%%%%%%%%%%%%%%%%%%%%%%%%%%%%%%%%%%%%%%%%%%%%%%%
Roy Davies in his article \cite{Davies_1971} constructed an interesting example of a compact metric space and two different Borel measures, say $\mu_a, \mu_b$, whose values on all closed balls of radius strictly less than 1 are equal to each other, such that the Radon-Nikodym derivative $d\mu_a/d(\mu_a + \mu_b)$ fails the differentiation property. According to C{\'e}rou and Guyader \cite{Cerou_Guyader_2006}, the universal weak consistency is unachievable if the differentiation property fails. It would be nice to give a complete proof of the differentiation property using the consistency argument, mentioned as a future work in Chapter \ref{chap:Future Prospects}. Now, we present the construction of Davies' example. 

Let $n$ be a natural number and let $(p_n)$ be a sequence of natural numbers, which will chosen recursively later. For each $n$, define a set $M_n = \{(i_1,i_2) : 1 \leq i_1 \leq p_n, 0 \leq i_2 \leq p_n  \}$ consisting of $p_n^2+ p_n$ pair of elements. An element of type $(i_1,i_2), i_2 >0$ is called a peripheral element corresponding to its central element, $(i_1,0)$. Let $G_n = (M_n,E_n)$ be a graph with $M_n$ and $E_n$ being the set of vertices and edges, respectively. The edges between the vertices are defined as: every central element is joined to other central elements, that is there is an edge between $(i_1,0)$ and $(i_2,0)$ for $1 \leq i_1 \neq i_2 \leq p_n$, and every peripheral element $(i_1,i_2),i_2>0$ is joined to its corresponding central element $(i_1,0)$. Based on these edges, we will define the distance between any two elements. 

		Let $\Omega = \prod_{n \in \mathbb{N}} M_n = \{ (x_n) = (x_1,x_2,\ldots): x_n = (i_1,i_2) \in M_n\}$. Let $(x_n), (y_n)$ be two elements of $\Omega$, they are distinct if $x_n = y_n $ for every $n$. Let $m$ be the smallest index such that $x_m \neq y_m$. Define the distance between $x = (x_n)$ and $y = (y_n)$ as, 
		\begin{align*}
		\rho( x,y) = 
		\begin{cases}
		0   \ \ \ \ \ \ \ \ \ \ \ \text{ if } x_n = y_n, \forall n \ \\
		(1/2)^{m} \ \ \ \ \ \text{if $x_m \neq y_m$, $\exists$ edge between } x_m \text{ and }y_m  \ \\
 		(1/2)^{m-1} \ \ \text{otherwise}
		\end{cases}
		\end{align*}  
		The function $\rho$ is similar to the metric defined in the Lemma \ref{eg:non_archi_metric}. Following the similar argument as in Lemma \ref{eg:non_archi_metric}, we will obtain that $\rho$ is a metric. In fact, $(\Omega,\rho)$ is a compact metric space of diameter equals to $1$.    
		
Let $\alpha_0 = 2/3$ and $\beta_0 = 1/3$ and now we will define the values of $p_n,\alpha_n,\beta_n$, recursively. Given $\alpha_0 > \beta_0 $, choose $p_1 > (\alpha_0/\beta_0)$ large enough that for some positive real numbers $\alpha_1 > \beta_1$, we have $ p_1^2 \alpha_1 + p_1 \beta_1 = 2/3$ and $ p_1^2 \beta_1 + p_1 \alpha_1 = 1/3$.  Given $\alpha_{n-1} > \beta_{n-1}$, choose $p_n > (\alpha_{n-1}/\beta_{n-1})$ so large that there are positive real numbers that 
\begin{align}
\label{eqn:alpha_n_n-1}
p_n^2 \alpha_{n} + p_n \beta_{n}  =  \alpha_{n-1}, \ p_n^2 \beta_{n} + p_n \alpha_{n} = \beta_{n-1}.
\end{align}

Let $\Omega[x_1,\ldots,x_n] = \{x_1\} \times \ldots \times \{x_n\} \times M_{n+1} \times M_{n+2} \times \ldots $ and $\Omega[\phi] = \Omega$. We define two functions based on the number of central elements in the set $\{x_1,\ldots,x_n\}$. If there are even  number of central elements in $\{x_1,\ldots,x_n\}$, then $\mu_a$ assigns value $\beta_n$ and $\mu_b$ assigns value $\alpha_n$ and similarly, if there are odd number of central elements then the values are flipped. That is, 
		\begin{align}
		\mu_a (  \Omega[x_1,\ldots,x_n] ) & =
		\begin{cases} 		\label{eqn:mu_a}
		\beta_n \ \ \text{if } \sharp \{i: 1 \leq i \leq n, x_i \text{ is central} \} = \text{ even}, \ \\
		\alpha_n \ \ \text{otherwise}
		\end{cases}
		\end{align}
	In the similar way, the function $\mu_b$ is defined, 
		\begin{align}
			\mu_b (\Omega[x_1,\ldots,x_n] ) & =
			\begin{cases}
			\alpha_n \ \ \text{if } \sharp \{i: 1 \leq i \leq n, x_i \text{ is central} \} = \text{ even}, \ \\
			\beta_n \ \  \text{otherwise}
			\end{cases}
			\label{eqn:mu_b}
		\end{align}
It follows from the above definitions that $\mu_a(\Omega) = \beta_0 = 1/3$ and 	$\mu_b(\Omega) = \alpha_0 = 2/3$. By the Carath{\`e}odary's extension theorem and Dynkin's $\pi-\lambda$ theorem, such measures exist and are unique if they are finitely additive. 
\begin{lemma}
\label{lem:countable_add}
Let $l = \{a,b\}$, we have
\begin{align*}
\mu_l (\Omega[x_1,\ldots,x_{n-1}] ) = \sum_{x_n \in M_n} \mu_l (\Omega[x_1,\ldots,x_n] )
\end{align*}    
\end{lemma}
\begin{proof}
A set  $\Omega[x_1,\ldots,x_{n-1}]$ is the finite union of disjoint sets $\Omega[x_1,\ldots,\allowbreak x_n]$ over all $x_n \in M_n$. Suppose there are odd number of central elements in the set $\{x_1,\ldots,x_{n-1}\}$, then we have  $ \mu_a (\Omega[x_1,\ldots,x_{n-1}] ) = \alpha_{n-1}$. Also, $ \sum_{x_n \in M_n} \mu_l (\Omega[x_1,\ldots,x_n] )$ can be divided into two sums, when $x_n$ is a central element and when $x_n$ is a peripheral element. Observe that there are $p_n$ central and $p_n^2$ peripheral elements in $M_n$. So, we have $ \sum_{x_n \in M_n} \mu_l (\Omega[x_1,\ldots,x_n] ) \allowbreak = p_n \beta_{n} + p_n^2 \alpha_{n}$, which is equal to $\alpha_{n-1}$ by the equation \eqref{eqn:alpha_n_n-1}. In the same way, by the equation \eqref{eqn:alpha_n_n-1} we have that $\mu_b (\Omega[x_1,\ldots,x_{n-1}] ) = \beta_{n-1} =  p_n^2 \beta_{n} + p_n \alpha_{n} =  \sum_{x_n \in M_n} \mu_b (\Omega[x_1,\ldots,x_n] )$. 

The finite additivity of both functions, in the other case when there are even number of central elements in $\{x_1,\ldots,x_{n-1}\}$ follows likewise. 
\end{proof}
We show in the following lemma that both the measures agree on all closed balls of radius strictly less than 1. 
\begin{lemma}
The values of the measures $\mu_a$ and $\mu_b$ are equal on each closed ball in $\Omega$ of radius strictly less than 1.
\end{lemma}
\begin{proof}
If $r =1$, then any closed ball in $\Omega$ of radius one is equal to the whole space $\Omega$, and we know that $\mu_a(\Omega) = 1/3 \neq \mu_b(\Omega)$. Let $x= (x_n) \in \Omega$. Since all the distances between the points of $\Omega$ are of the form $(1/2)^n$. Suppose $r = (1/2)^t < 1$ for a fixed integer $t \geq 1$. The closed ball $\bar{B}(x,r)$ will contain all those $y = (y_n)$ which are at distance at most $(1/2)^t$ to $x$. So, $\bar{B}(x,r)$ contain two types of elements: 
\begin{itemize}
\item All $y \in \Omega$ such that $\rho(x,y)  = 2^{-t}$ belong to $\bar{B}(x,r)$. That is, $x_i = y_i $ for $1 \leq i \leq t-1$, $x_t \neq y_t$ and there is an edge between $x_t$ and $y_t$.
\item All $y \in \Omega$ such that $\rho(x,y) = 2^{-m} < 2^{-t}$ are also in $\bar{B}(x,r)$. This means that for every $m = t+1,t+2,\ldots$, we have $x_i = y_i $ for $1 \leq i \leq m-1$, $x_m \neq y_m$.
\end{itemize}
In simpler words, $\bar{B}(x,r)$	contains all those $(y_n)$ for which, either there is an edge between $x_t$ and $y_t$, or $x_t = y_t$. 
We consider the following cases to compute the measure of a closed ball,
\begin{enumerate}[(i)]
\item Let $x_t$ be a central element of $M_t$, say $x_t = (i_1,0)$. Then, the possible values of $y_t$ are denoted by $E = \{ (j,0), (i_1,j): 1 \leq j \leq p_t\}$, there are $p_t$ central and $p_t$ peripheral elements. Therefore,  the closed ball can be written as, $\bar{B}(x,r) = \cup_{y_t \in E} \Omega[x_1,\ldots,x_{t-1},y_t]$. To evaluate the measure of $\bar{B}(x,r)$, we further have two cases, either number of central elements in $\{x_1,\ldots,x_{t-1}\}$ is odd, or even. We treat only the case having odd number of central elements (the other case follows similarly). If the number of central elements in $\{x_1,\ldots,x_{t-1}\}$ is odd, then by the Lemma \ref{lem:countable_add} and definition of $\mu_a,\mu_b$, we have  $\mu_a(\bar{B}(x,r))  = p_t \alpha_t + p_t \beta_t = \mu_b(\bar{B}(x,r))$, because there are equal number of central and peripheral elements in $E$.
\item Let $x_t$ be a peripheral element of $M_t$, say $x_t = (i,j)$. Then the possible values for $y_t$ is $\{(i,j), (i,0)\}$. This implies that if the number of central elements in $\{x_1,\ldots,x_{t-1}\}$ is odd or even, then $\mu_a(\bar{B}(x,r))  = \alpha_t + \beta_t = \mu_b(\bar{B}(x,r))$.
\end{enumerate} 
In any case, the values of both measures are equal on every closed ball of radius $<1$. 
\end{proof}
Now, we will show that the differentiation property fails. 
\begin{lemma}
The differentiation property does not holds for $\mu_a + \mu_b$.
\end{lemma}
\begin{proof}
As, $\mu_a$ is absolutely continuous with respect to $\mu_a + \mu_b$, by the Radon-Nikodym theorem, there is a measurable function $f_a: \Omega \rightarrow [0,\infty)$ such that for any measurable set $A \subseteq \Omega$,
\begin{align*}
 \mu_a(A)  = \int_{A} f_a(x) (\mu_a+\mu_b)(dx). %%%  \mu_b(A)  = \int_{A} f_b(x) (\mu_a+\mu_b)(dx).
\end{align*}
Suppose that the differentiation theorem holds for $\mu_a+\mu_b$, then we have
	\begin{align*}
\lim_{r \rightarrow 0} \frac{1}{ (\mu_a + \mu_b) (\bar{B}(x,r) ) } \int_{\bar{B}(x,r)}  f_a(y) (\mu_a + \mu_b)(dy) & = f_a(x) \ \ \text{for $(\mu_a+\mu_b)$ a.e. } x 
	\end{align*}
As, $ \mu_a(\bar{B}(x,r))  =  \mu_b(\bar{B}(x,r)), r<1$, then $f_a(x) = 1/2$ for $(\mu_a+\mu_b)$-almost everywhere. Let $\theta = \{x \in \Omega: f_a(x) = 1/2 \}$, so  $(\mu_a + \mu_b)(\theta^c) = 0$.	Therefore, we have
	\begin{align*}
	\mu_a(\Omega) & = \int_{\theta} f_a(x) (\mu_a+\mu_b)(dx) \ \\
	& = \frac{1}{2}, 
	\end{align*}
which is a contradiction. 
\end{proof}
Davies extended $\Omega$ to $\hat{\Omega}$ and also $\mu_a$ and $\mu_b$ to become probability measures on $\hat{\Omega}$, to conclude that there are two distinct probability measures $\mu_a^{'}$ and $\mu_b^{'}$ such that they have equal values on every closed ball in $\hat{\Omega}$ of radius $<1$. 

Although, the original space $\Omega$ and measures $\mu_a$ and $\mu_b$ are enough to show the inconsistency of the $k$-nearest neighbor rule. We explain in brief the further argument by Davies in the following paragraph. 

We can define $\hat{\Omega}$ as the union of disjoint copies of $\Omega$, such as $\hat{\Omega} = \Omega \times \{a\} \cup \Omega \times \{b\}$. The Borel measures $\mu_a^{'},\mu_b^{'}$ are defined as, for $\hat{A} \subseteq \hat{\Omega}$,
		\begin{align*}
		\mu_a^{'}(\hat{A}) = \mu_a(A_1) + \mu_b(A_2), \  \mu_b^{'}(\hat{A}) = \mu_b(A_1) + \mu_a(A_2),
		\end{align*}
		where $\hat{A} = A_{1} \times \{a\} \cup A_2 \times \{b\}$. This implies that $\mu_a^{'}$ and $\mu_b^{'}$ are two distinct probability measures on $\hat{\Omega}$. The distance $\hat{\rho}$ between two elements, where each element is from $\Omega \times \{a\}$ and $\Omega \times \{b\}$, respectively, is equal 1. If both the points are from same space, say $\Omega \times \{a\}$, then the distance is given by the original metric $\rho$. The metric properties of $\rho$ implies that $\hat{\rho}$ is a metric and thus, $\hat{\Omega}$ is of diameter 1.  It follows from the properties of $\mu_a, \mu_b$ that the values of $\mu_a^{'}$ and $\mu_b^{'}$ are equal on all closed balls of radius strictly less than 1. 
		
%%%%%%%%%%%%%%%%%%%%%%%%%%%%%%%%%%%%%%%%%%%%%%%%%%%%%
\subsection{Inconsistency of the $k$-nearest neighbor rule on the Davies' example}
\label{subsec:inconsistent_Davies}
%%%%%%%%%%%%%%%%%%%%%%%%%%%%%%%%%%%%%%%%%%%%%%%%%%%%%
Here, we show that the $k$-nearest neighbor is not weakly consistent without using the differentiation argument. It also give a hope that the consistency can be studied directly without involving any differentiation argument. 

The following lemma suggest that if the values of two measures are equal on every closed ball, then the values of both measures will be equal on every open ball and every sphere.
\begin{lemma}
\label{lem:closed_open}
For every $x \in \Omega$, we have
\begin{align*}
\mu_a(B(x,r)) = \mu_b(B(x,r)) \text{, } \mu_a(S(x,r)) = \mu_b(S(x,r)),
\end{align*}
where $r \leq 1$ for open balls and $r <1$ for spheres.
\end{lemma}
\begin{proof}
Let $(r_n)$ be an increasing sequence converging to $r $ such that $r_1 \leq r_{2} \leq \ldots < r \leq 1$ and $\cup_{n=1}^{\infty}\bar{B}(x,r_n) = B(x,r)$. Then, by the $\sigma$-additivity of $\mu_a$ and $\mu_b$, we have
\begin{align*}
\mu_a(B(x,r)) & = \mu_a(\cup_{n=1}^{\infty}\bar{B}(x,r_n)) \ \\
& = \lim_{n \rightarrow \infty} \mu_{a}(\bar{B}(x,r_n)) \ \\
& = \lim_{n \rightarrow \infty} \mu_{b}(\bar{B}(x,r_n)) \ \\
& = \mu_b(\cup_{n=1}^{\infty}\bar{B}(x,r_n)) \ \\
& = \mu_b(B(x,r)).
\end{align*}
As, the values of measures are equal on every closed ball of radius $<1$ and on every open ball with radius $\leq 1$, it follows that, for $r <1$
\begin{align*}
\mu_a(S(x,r)) & = \mu_a(\bar{B}(x,r)) - \mu_a(B(x,r)) \ \\
& = \mu_b(\bar{B}(x,r)) - \mu_b(B(x,r)) \ \\
& = \mu_b(S(x,r)).
\end{align*}
\end{proof}
Let us define two measures $\mu_0$ and $\mu_1$ such that, 
\begin{align} 
\label{eqn:mu_0_mu_1}
\mu_0 = \frac{6}{5} \mu_a, \ \ \mu_1 = \frac{9}{10} \mu_b.
\end{align}
Then, $\mu_0 (\Omega ) = (6/5)(1/3) = 0.4 $ and $\mu_1(\Omega) = 0.6$. For $r <1$, we know that $\mu_{a}(\bar{B}(x,r)) = \mu_{b}(\bar{B}(x,r))$ and so, by the equation $\eqref{eqn:mu_0_mu_1}$ we have, $\mu_0 (\bar{B}(x,r)) \allowbreak = (4/3) \mu_{1}(\bar{B}(x,r))$. 

Similarly, from the Lemma \ref{lem:closed_open} and the equation $\eqref{eqn:mu_0_mu_1}$ it follows that $\mu_0 (B(x,r)) = (4/3) \mu_1(B(x,r))$ and $\mu_0 (S(x,r)) = (4/3) \mu_1(S(x,r))$, whenever $r <1$. 

Let $\mu = \mu_0 + \mu_1$, then $\mu(\Omega) =1$ is a probability measure on $\Omega$. We observe that, $\mu_1$ is absolutely continuous with respect to $\mu$, by the Radon-Nikodym theorem, there exists a function $\eta$, called the Radon-Nikodym derivative such that for measurable $A \subseteq \Omega$,
\begin{align}
\label{eqn:RND_mu_1}
\mu_1(A) = \int_{A} \eta(x) \mu(dx). 
\end{align}
The distribution of the pair of random variables $(X,Y)$, can be described by $\mu$ and $\eta$. Let $\mu_1$ denote the distribution of points having label 1, that is, $\mu_1(A) = \mathbb{P}( X \in A, Y =1)$ for measurable $A \subseteq \Omega$. Then, there exists a measurable set $M_1 \subseteq \Omega$, $\mu(M_1) >0$ such that $\eta(x) \geq 0.6$ for all $x \in M_1$. Suppose there is no such set $M_1$, which means that the value of $\eta$ is strictly less than 0.6 on every set of positive measure. This is a contradiction to the fact that $\mu_1(\Omega) = 0.6$, due to the above relation between $\eta$ and $\mu_1$. Let $M$ be a measurable subset of $\Omega$ such that $M_1 \subseteq M$ and for every $x \in M$, $\eta(x) \geq 0.5$. The Bayes rule $g^{*}$ assigns label 1 to a data point $x$ if $\eta(x) \geq 0.5$, otherwise assigns label 0. So, $g^{*}(x)$ is equal to 1 if $x \in M$ and equal to 0 if $x \notin M$. The Bayes error is given by,
\begin{align*} 
\ell^{*}_{\mu} & = \mathbb{P}(g^{*}(X) = 1, Y=0) + \mathbb{P}(g^{*}(X) = 0, Y=1) \ \\
& = \mathbb{P}(X \in M, Y=0)  +  \mathbb{P}(X \in M^{c}, Y=1) \ \\
& = \mu_0(M) + \mu_1(M^{c}).
\end{align*}
For $x \in M^{c}$, we have $\eta(x) < 0.5$. By the equation \eqref{eqn:RND_mu_1}, we have $\mu_1(M^c) =  \int_{M^{c}} \eta(x) \mu(dx) \leq 0.5\mu(M^c) = 0.5 ( \mu_0(M^c) + \mu_0(M^c) )$. So, $\mu_1(M^c) \leq \mu_0(M^c)$, this implies that $\ell^{*}_{\mu} \leq 0.4$.
We will show that the expected error of the $k$-nearest neighbor rule is at least 0.6 in the limit, and therefore, strictly greater than the Bayes error. 

Let $D_n = ( (X_1,Y_1),\ldots,(X_n,Y_n) )$ be a random labeled sample of independently and identically distributed random pairs.  Let $x \in X, r<1$ and let  $k', k''$ be two natural numbers such that $k' \leq k \leq k'' $. Let $T$ denote the following event, 
\begin{align*}
T = \bigg\{ \varepsilon_{kNN}(X) = r <1, \sharp\{ \bar{B}(x,r)\} = k'', \sharp\{ B(x,r)\} = k' \bigg\}, 
\end{align*}
where $\varepsilon_{kNN}(X)$ is defined in the equation \eqref{eqn:varknn_ball}. Let $Y_1,\ldots, Y_k$ denote the labels of the the $k$-nearest neighbors of $x$, denoted by $X_1,\ldots,X_k$. We claim that, the expectation of average of labels of $k$-nearest neighbors of $x$, given the event $T$, is equal to 3/7. This can be observed by considering the following two cases:
\begin{enumerate}[(I)]
\item If $\mu(S(x,r)) =0$, then there is only one data point $X_k$ on the sphere $S(x,r)$. In this case, $k'' = k$ and $k'=k-1$. For $i =1,\ldots,k$, given the event $T$, we know that $X_i$ are coming from the closed ball $\bar{B}(x,r)$, thus we have
\begin{align*}
\mathbb{E}\{ Y_{i}| T \} & =  \mathbb{P}(Y_i=1 | T) \ \\
& = \frac{\mathbb{P}(X_i \in \bar{B}(x,r),Y_i =1)}{\mathbb{P}(X_i \in \bar{B}(x,r))} \ \\
& =  \frac{\mu_1( \bar{B}(x,r) )}{\mu_0( \bar{B}(x,r) ) + \mu_1(\bar{B}(x,r))}\ \\
& = \frac{\mu_1( \bar{B}(x,r) )}{\frac{4}{3}\mu_1( \bar{B}(x,r) ) + \mu_1(\bar{B}(x,r))} \ \\
& = \frac{3}{7},
\end{align*}
where we used the equality, $\mu_0( \bar{B}(x,r) ) = (4/3) \mu_0( \bar{B}(x,r) )$. The expectation of average of labels of the $k$-nearest neighbors of $x$ is,
\begin{align*}
\mathbb{E}\bigg\{ \frac{Y_{1} + \ldots + Y_{k}}{k} \bigg| T \bigg\} & = \frac{1}{k} \sum_{i=1}^{k} \mathbb{E}\{ Y_{i}| T \} \ \\
&  = \frac{3}{7}.
\end{align*} 
\item If $\mu(S(x,r)) >0$, then there may be more than one data point on the sphere. Given the event $T$, out of $k$ nearest neighbors of $x$, the $k'$-nearest neighbors of $x$ which are $X_1,\ldots,X_{k'}$  belongs to the open ball $B(x,r)$ and the remaining ($k$-$k'$)-nearest neighbors, $X_{k'+1},\ldots,X_{k}$ are coming from the sphere with the distance ties being broken uniformly on the sphere. 

Therefore, for $i =1 ,\ldots, k'$
\begin{align*}
\mathbb{E}\{ Y_{i}| T \} & =  \mathbb{P}(Y_i=1 | T) \ \\
& = \frac{\mathbb{P}(X_i \in B(x,r),Y_i =1)}{\mathbb{P}(X_i \in B(x,r))} \ \\
& =  \frac{\mu_1( B(x,r) )}{\mu_0( B(x,r) ) + \mu_1(B(x,r))}\ \\
& = \frac{3}{7}.
\end{align*}
Since the distance ties are broken uniformly on the sphere, so for $i= (k'+1), \ldots, k$,
\begin{align*}
\mathbb{E}\{Y_{i}| T \} & =  \mathbb{P}(Y_i=1 | T) \ \\
& = \frac{\mathbb{P}(X_i \in S(x,r),Y_i =1)}{\mathbb{P}(X_i \in S(x,r))} \ \\
& =  \frac{\mu_1( S(x,r) )}{\mu_0( S(x,r) ) + \mu_1(S(x,r))}\ \\
& = \frac{3}{7}.
\end{align*}
We can write the expectation of the average of $k$-nearest neighbor labels as,
\begin{align*}
\mathbb{E}\bigg\{ \frac{Y_{1} + \ldots + Y_{k}}{k} \bigg| T \bigg\} &  =  \frac{k'}{k} \mathbb{E}\bigg\{ \frac{Y_{1} + \ldots + Y_{k'}}{k'} \bigg| T \bigg\} + \\ & \ \ \ \ \ \frac{k - k'}{k} \mathbb{E}\bigg\{ \frac{Y_{k'+1} + \ldots + Y_{k}}{k- k'} \bigg| T \bigg\} \ \\ 
 & =  \frac{k'}{k} \sum_{i=1}^{k'}\mathbb{E}\bigg\{  \frac{Y_{i}}{k'} \bigg| T \bigg\} + \frac{k-k'}{k} \sum_{i=k'+1}^{k}\mathbb{E}\bigg\{  \frac{Y_{i}}{k-k'} \bigg| T \bigg\} 
\end{align*}
\begin{align*}
& = \frac{k'}{k} \frac{3}{7} +  \frac{(k -k')}{k} \frac{3}{7}  \hspace{1cm} \\
& = \frac{3}{7}. 
\end{align*}
\end{enumerate}
Therefore, in any case, given the event $T$ the conditional expectation of the average of the labels of the $k$-nearest neighbors of $x$ is $ \frac{3}{7} $.

According to the Cover-Hart lemma, $\varepsilon_{kNN} \rightarrow 0$ almost surely, whenever $n,k \rightarrow \infty$ and $k/n \rightarrow 0$. As a result, we have
\begin{align*} 
 \mathbb{E}\bigg\{ \frac{Y_{1} + \ldots + Y_{k}}{k} \bigg\} & = \mathbb{E}\bigg\{ \mathbb{E}\bigg\{ \frac{Y_{1} + \ldots + Y_{k}}{k} \bigg| T \bigg\}  \bigg\} + \mathbb{E}\bigg\{ \mathbb{E}\bigg\{ \frac{Y_{1} + \ldots + Y_{k}}{k} \bigg| T^c \bigg\}  \bigg\} \\
& = \frac{3}{7}  + \mathbb{E}\bigg\{ \mathbb{E}\bigg\{ \frac{Y_{1} + \ldots + Y_{k}}{k} \bigg| T^c \bigg\}  \bigg\} \ \\
& \leq \frac{3}{7} + \mathbb{P}(T^c),
\end{align*}
where $T^{c}$ is the event $\{ \varepsilon_{kNN} = 1 \}$ and so $\mathbb{P}(T^{c})$ converges to 0 in the limit. This implies that less than half of the $k$-nearest neighbors have label 1 in the limit. Thus, the $k$-nearest neighbor rule will predict label 0 in the limit $n,k \rightarrow \infty$ and $k/n \rightarrow 0$. The error would be the set of all points with label 1, that is,
\begin{align*}
\lim_{n,k \rightarrow \infty, k/n \rightarrow 0 } \mathbb{E}\{\ell_{\mu}(g_n) \} & = \mathbb{P}(X \in \Omega, Y =1) \ \\
& = \mu_1(\Omega) \ \\
& = 0.6 > \ell^{*}_{\mu}.
\end{align*}
Hence, the $k$-nearest neighbor rule is not consistent.
%%%%%%%%%%%%%%%%%%%%%%%%%%%%%%%%%%%%%%%%%%%%%%%%%%%%%%%%%%%%%%%%%%%%%%%%%%%%%%%%%%%%%%%
\chapter{Consistency and Metric Dimension}
\label{chap:Consistency in a metrically finite dimensional spaces}
%%%%%%%%%%%%%%%%%%%%%%%%%%%%%%%%%%%%%%%%%%%%%%%%%%%%%%%%%%%%%%%%%%%%%%%%%%%%%%%%%%%%%%%
Here, we present our analysis on the consistency of the $k$-nearest neighbor rule in various metric spaces with finite Nagata dimension and finite metric dimension. In this chapter, we divide our work into two main sections, consistency with zero distance ties and consistency with distance ties, to illustrate how the solution differ in the two cases. Starting with a no distance ties assumption, we prove a generalized version of geometric Stone's lemma for spaces with finite Nagata dimension. Further, we present some counter-examples to understand the difficulty in generalizing Stone's lemma in presence of distance ties. We then prove a different lemma to handle distance ties and finally, we reprove the universal weak consistency of the $k$-nearest neighbor rule in a metrically sigma-finite dimensional space. 
We also establish the strong consistency in metrically finite dimensional spaces under the assumption of no distance ties.  
%%%%%%%%%%%%%%%%%%%%%%%%%%%%%%%%%%%%%%%%%%%%%%%%%%%%%
\section{Consistency without distance ties}
\label{sec:Consistency without distance ties}
%%%%%%%%%%%%%%%%%%%%%%%%%%%%%%%%%%%%%%%%%%%%%%%%%%%%%
The simplest case is to work in distance ties-free settings. Assuming that there are no distance ties means the probability of a data point $x_j$ belonging to a sphere $S(x,\rho(x,x_i)), i \neq j$ is zero. Therefore, the $\nu$-measure of sphere $S(x,\rho(x,x_i))$ is zero. As, $x_i$ can be any data point in the space, so in principle, we assume that the measure of every sphere is zero. We also sometimes say that $\nu$ has zero probability of ties.

Given a sample of $n$ data points $\Sigma_n = \{x_1,\ldots,x_n\}$, define a function $\varepsilon_{kNN}:\Omega \rightarrow \mathbb{R}$ such that, 
\begin{align}
\label{eqn:varknn_ball}
\varepsilon_{kNN}(x) = \inf\{r >0: \sharp\{ \bar{B}(x,r) \}\geq k+1\},
\end{align}
where the ball $\bar{B}(x,r)$ is a treated as a ball in the finite set $\{x,x_1,\ldots,x_n\}$. 
The value $\varepsilon_{kNN}(x)$ is the minimum radius such that $\bar{B}(x,\varepsilon_{kNN}(x))$ contain at least $k+1$ sample points including the center $x$. 

Note that, the open ball $B(x,\varepsilon_{kNN}(x))$ contain at most $k$ points and $B(x,\varepsilon_{kNN}(x)) \subseteq \mathcal{N}_{k}(x) \cup \{x\}$. The problematic case of distance ties occurs on the sphere $S(x,\varepsilon_{kNN}(x))$. We defer the case of tie-breaking until the next section. The assumption of zero distance ties implies that $\bar{B}(x,\varepsilon_{kNN}(x)) = \mathcal{N}_{k}(x) \cap \{x\}$, containing $x$ and the $k$-nearest neighbors of $x$ from $\Sigma_n$. 

We prove a generalized version of Stone's lemma for metric spaces with finite Nagata dimension in the following lemma. 
\begin{lemma}[Generalized Stone's lemma]
\label{lem:gsl_sigma}
Let $(Q,\rho)$ be a separable metric space with Nagata dimension $\beta-1$ in $\Omega$. Let $\Sigma_n = \{x_1,x_2,\ldots,x_n \}$ be a finite sample in $\Omega$ and assume there are no distance ties. For $x \in \Omega$, we have
\begin{align*}
\sum_{x_i \in \Sigma_n \cap Q} \mathbb{I}_{\{ x \in \mathcal{N}_{k}(x_i)\}} \leq (k+1)\beta.
\end{align*}
\end{lemma}
\begin{proof}
Define a function $F : \Omega \rightarrow \mathbb{R}$ as,
\begin{align*}
F(x)=\sum_{x_i \in \Sigma_n \cap Q} \mathbb{I}_{\{x \in \mathcal{N}_{k}(x_i)\}}.
\end{align*}
Let $x_0 \in \Omega$ and suppose there are $m$ points from the sample $\Sigma_n \cap Q$ which have $x_0$ as one of their $k$-nearest neighbors. Let this set be $\tilde{\Sigma} = \{x_1,x_2,\ldots,x_m\}$. So, $F(x_0)=m$ and it is sufficient to show that $m \leq (k+1) \beta$.

Consider a family of closed balls $\mathcal{F} = \{\bar{B}(x_i,\varepsilon_{kNN}(x_i)):x_i \in \tilde{\Sigma} \}$, then every closed ball contains $x_0$. Note that, the ball $\bar{B}(x_i,\varepsilon_{kNN}(x_i))$ in $\mathcal{F}$ is considered as a ball in $\tilde{\Sigma}$. By the Lemma \ref{lem:wcp_nd},  $Q$  has ball-covering dimension $\beta$. There exists a subset $\mathcal{F}'$ of $\mathcal{F}$ such that the center of every ball in $\mathcal{F}$ belongs to some ball in $\mathcal{F}'$ and every $x$ in $\Omega$ belong to at most $\beta$ number of balls in $\mathcal{F}'$,
\begin{align*}
\sum_{\bar{B} \in \mathcal{F}'} \mathbb{I}_{\{x \in \bar{B}\}} \leq \beta.
\end{align*}
We know that every closed ball in $\mathcal{F}$ has at most $k+1$ data points out of $m$ sample points in $\tilde{\Sigma}$. Extracting $\mathcal{F}'$ means dividing $m$ points in $p$ number of boxes such that every box has at most $k+1$ points. The minimum number of such boxes would be $m/(k+1)$, so the cardinality of $\mathcal{F}'$ is at least $m/(k+1)$. As, $x_0$ belongs to every ball in $\mathcal{F}'$, we have
\begin{align*}
 m/(k+1)  \leq \sum_{\bar{B} \in \mathcal{F}'} \mathbb{I}_{\{ x_0 \in \bar{B}\}}  = \sharp \mathcal{F}' \leq \beta.
\end{align*}
Therefore, $m \leq (k+1) \beta$.
\end{proof}
As a consequence of the Lemma \ref{lem:gsl_sigma}, the $k$-nearest neighbor rule is weakly consistent for probability measures with zero probability of distance ties. 
\begin{theorem}
\label{thm:nagata_sc}
Let $(Q,\rho)$ be a separable metric space having Nagata dimension $\beta-1$ in $\Omega$. Let $\nu$ be a probability measure on $Q$ and assume that $\nu$ has zero probability of distance ties. Then, the expected error probability of the $k$-nearest neighbor rule converges to Bayes error with respect to $\nu$.
\end{theorem}
\begin{proof} 
The proof of this theorem is similar to the proof of the Theorem \ref{thm:stone_euclidean}, except that we use the Lemma \ref{lem:gsl_sigma} instead of the classical Stone's lemma. 

Let $\varepsilon > 0$ be any real number, by Luzin's theorem there is a set $K \subseteq Q$ such that $\nu(U = Q \setminus K) < \varepsilon$. By the Theorem \ref{lem:diff_err}, Lemma \ref{lem:third_prop_all_metric_sp} and Lemma \ref{lem:eta_tilde_reg}, everything boils down to showing that  
$$ \mathbb{E}\bigg\{ \frac{1}{k}\sum_{i =1}^{n} \mathbb{I}_{\{X_i \in \mathcal{N}_k(X)\}} ( \eta^*(X_i) - \eta(X_i) )^{2} \bigg| X\in K, X_i \in U \bigg\}, $$ 
is bounded above by some constant (which is independent of $n$ and $k$) times $\varepsilon$.

We first exchange $X$ and $X_i$ such that 
\begin{align*}
& \mathbb{E}\bigg\{ \frac{1}{k} \sum_{i =1}^{n} \mathbb{I}_{\{X_i \in \mathcal{N}_k(X)\}} ( \eta^*(X_i) - \eta(X_i) )^{2} \bigg| X\in K, X_i \in U \bigg\}  \\
& = \mathbb{E}\bigg\{ \frac{1}{k} \sum_{i =1}^{n} \mathbb{I}_{\{X \in \mathcal{N}_k(X_i)\}} ( \eta^*(X) - \eta(X) )^{2} \bigg| X_i\in K, X \in U \bigg\}  \ 
\end{align*} 
Then, we apply the generalized Stone's lemma \ref{lem:gsl_sigma} to bound the number of points having $X$ as their $k$-nearest neighbor. So, we have
 \begin{align*}
& \mathbb{E}\bigg\{ \frac{1}{k} \sum_{i =1}^{n} \mathbb{I}_{\{X \in \mathcal{N}_k(X_i)\}} ( \eta^*(X) - \eta(X) )^{2} \bigg| X_i\in K, X \in U \bigg\}  \ \\
& \leq  \frac{k+1}{k}\beta \mathbb{E}\{( \eta^*(X) - \eta(X) )^{2} | X \in U \} \ \\
& \leq 2\beta \nu(U) < 2\beta \varepsilon,
\end{align*} 
where we use the fact that $( \eta^*(X) - \eta(X) )^{2}$ is bounded above by one.
\end{proof}
In the above theorem, we proved the weak consistency for spaces with finite Nagata dimension. Indeed, the result is true for metric spaces having sigma-finite Nagata dimension. 
\begin{corollary}
Let $(\Omega,\rho)$ be a separable metric space which has sigma-finite Nagata dimension. Let $\nu$ be a probability measure on $\Omega$ and assume that $\nu$ has zero probability of distance ties. Then, the $k$-nearest neighbor rule is weakly consistent with respect to $\nu$.
\end{corollary}
\begin{proof}
If $\Omega$ has sigma-finite Nagata dimension, then $\Omega$ can be written as union of increasing chain of $Q_i$ which have finite Nagata dimension $\beta_i -1$. For given $\varepsilon$, we choose $Q_l$ such that $\nu(Q_l) > 1 -\varepsilon/2 $ (this is possible because $Q_i$ is an increasing chain and $\nu$ is $\sigma$-additive). By the Luzin's theorem, there is a set $K \subseteq Q_l$ such that $\nu(Q_l \setminus K) < \varepsilon/2$. Let $U = \Omega \setminus K$ and so $\nu(U) < \varepsilon$. 

As, the samples $X_i$ take values in $K \subseteq Q_l$, so we can apply the generalized Stone's lemma \ref{lem:gsl_sigma}. The rest of the argument is exactly same as in the proof of the Theorem \ref{thm:nagata_sc}.
\end{proof}
Although we have shown the consistency in metric spaces with sigma-finite Nagata dimension but the assumption of zero distance ties is not an ideal assumption. To obtain the $k$-nearest neighbors set and prove the universal consistency using the Stone's lemma, we need an appropriate tie-breaking method. The index-based tie-breaking method is a popular and simplest method to obtain the set $\mathcal{N}_k(x)$ for a data point $x$. 
%%%%%%%%%%%%%%%%%%%%%%%%%%%%%%%%%%%%%%%%%%%%%%%%%%%%%
\section{Consistency with distance ties}
\label{sec:Consistency with distance ties}
%%%%%%%%%%%%%%%%%%%%%%%%%%%%%%%%%%%%%%%%%%%%%%%%%%%%%
In the previous section, we established the consistency under the assumption of no distance ties but proving the consistency becomes much more complicated when the distance ties are considered. This is why in the literature, the consistency is first proved under the assumption of no ties and then the solutions are extended in the presence of distance ties. It is worth to examine the cases of distance ties and no distance ties separately. In this section, we start by showing that Stone's lemma fails in the presence of distance ties and so we prove a different geometric lemma to handle distance ties which will help in establishing the universal weak consistency in metrically finite dimensional spaces. 
%%%%%%%%%%%%%%%%%%%%%%%%%%%%%%%%%%%%%%%%%%%%%%%%%%%%%
\subsection{Stone's lemma fails with distance ties!}
\label{subsec:Stone's lemma fails with distance ties!}
%%%%%%%%%%%%%%%%%%%%%%%%%%%%%%%%%%%%%%%%%%%%%%%%%%%%%
Now, we will present few examples in order to conclude two important things that the Stone's lemma fails in the presence of distance ties and that the distance ties are unavoidable even in metric spaces with finite Nagata dimension.
\begin{example}
Let $(\Omega,\rho)$ be a separable metric space, where $\rho$ is the 0-1 metric. Suppose $n > k+1$ and let $\Sigma_n = \{x_1,x_2,\dots,x_n\}$ be a sample of $n$ data points in $\Omega$. Then,
\begin{enumerate}[(i)]
\item \textit{ $\Sigma_n$ has Nagata dimension $0$}.
	For $x_1,x_2 \in \Sigma_n$ and $a \in \Omega$, if $x_1 =x_2$ then $\rho(x_1,x_2)$ is 0 and $\max\{\rho(a,x_1), \rho(a,x_2)\}$ is either 0 or 1. In the case $x_1 \neq x_2$, the distance $\rho(x_1,x_2)$ is equal to 1 and $\max\{\rho(a,x_1), \rho(a,x_2)\}$ is 1. This is true for an two points from $\Sigma_n$. So, $\Sigma_n$ has Nagata dimension $0$ in $\Omega$.
\item \textit{Stone's lemma fails.}
Let $x \in \Omega$ such $x \neq x_i$ for all $x_i \in \Sigma_n$. Then $d(x,x_i) = 1$ for every $x_i \in \Sigma_n$, which means $x$ is the $1$-nearest neighbor of every $x_i$ in $\Sigma_n$.
\begin{align*} 
\sum_{i=1}^{n} \mathbb{I}_{\{x \in \bar{B}(x_i,\varepsilon_{1NN}(x_i))\}} = n \nleq k+1.  
\end{align*} \demo
\end{enumerate} 
\end{example}
The above example depicts the need for a tie-breaking method. Suppose $\Sigma_n$ is an ordered set of $n$ data points in the above example. The $1$-nearest neighbor of $x_i$ is chosen from the ordered set $\{x_1,\ldots,x_{i-1},x,x_{i+1},\ldots,x_n\}$. By the index-based tie-breaker, $x_1$ is the only data point having $x$ as its $1$-nearest neighbor after breaking distance ties. So, the Stone's lemma holds in this particular case. In Euclidean spaces, breaking distance ties by comparing indices is sufficient but this is not true for general metric spaces. In the following example, we show that the tie-breaking by comparing indices is not the right method to obtain a version of Stone's lemma. Indeed, the generalized Stone's lemma fails even if the distance ties are broken randomly and uniformly, which is more stable than index based tie-breaking method.
\begin{lemma}
\label{lem:nagata_fails_stone_lemma}
Let $\alpha >0$ be any real number and $x_1$ be a data point. Then, there is a finite sample $\Sigma_n = \{x_1,\ldots,x_n\}$ of size $n$ (depends on $\alpha$) with Nagata dimension 0, such that under the random uniform tie-breaking method,
$$  \mathbb{E}\bigg\{ \sum_{i=2}^{n} \mathbb{I}_{\{ x_1 \in \mathcal{N}_1(x_i)\}} \bigg\} > \alpha. $$
\end{lemma}
\begin{proof} Choose a positive integer $n$ large enough that $\sum_{i=1}^{n-1} 1/i > \alpha$. 
We will construct the sample $\Sigma_n$ recursively. Let $\Sigma_1 = \{ x_1\}$ and add $x_2$ to $\Sigma_1$ to form $\Sigma_2 = \Sigma_1 \cup \{x_2 \}$ such that $\rho(x_2,x_1) =1 $. Add $x_3$ to $\Sigma_2$ at a distance equal to 2 from $x_1,x_2$ and set $\Sigma_3 = \Sigma_2 \cup \{ x_3\}$. At $n-1$-th step, the set $\Sigma_{n-1}$  has already been defined, we add $x_n$ to $\Sigma_{n-1}$ to obtain $\Sigma_n= \{x_1,\ldots,x_n\}$ such that $\rho(x_i,x_n) = 2^{n-1}$ for $1 \leq i \leq n-1$. The construction is shown in the following figure. 

%%%%%%%%%%%%%%%%%%%%%%%%%%%%%%%%%%%%%%%%%%%%%%%%%%%%%
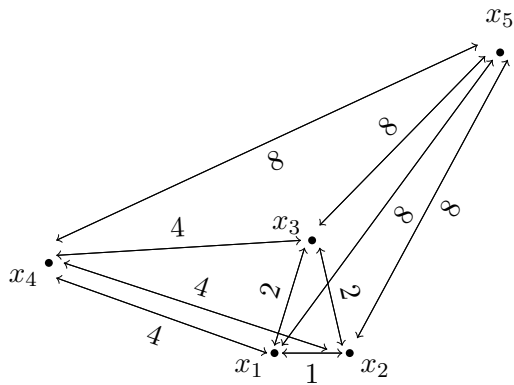
\begin{figure}[!ht]
\centering
\begin{tikzpicture}
\fill (0,0) circle (1.5pt) node[left, yshift=-0.2cm] {\small $x_1$};
\fill (1,0) circle (1.5pt) node[right, yshift=-0.2cm] {\small $x_2$};
\draw[->] (0.1,0) -- (0.9,0) node[pos=0.5,sloped,anchor = center,below] {{\small $1$}};
\draw[->] (0.9,0) -- (0.1,0);
\fill (0.5,1.5) circle (1.5pt) node[left, yshift=0.2cm] {\small $x_3$};
\draw[->] (0,0.1) -- (0.4,1.4) node[pos=0.5,sloped,anchor = center,above] {{\small $2$}};
\draw[->] (0.4,1.4) -- (0,0.1); 
\draw[->] (0.9,0.1) -- (0.6,1.4) node[pos=0.5,sloped,anchor = center,above] {{\small $2$}};
\draw[->] (0.6,1.4) -- (0.9,0.1); 
\fill (-3,1.2) circle (1.5pt) node[left, yshift=-0.2cm] {\small $x_4$};
\draw[->] (-0.1,0) -- (-2.9,1) node[pos=0.5,sloped,anchor = center,below] {{\small $4$}};
\draw[->] (-2.9,1) -- (-0.1,0);
\draw[->] (0.7,0.05) -- (-2.8,1.2) node[pos=0.5,sloped,anchor = center,above] {{\small $4$}};
\draw[->] (-2.8,1.2) -- (0.7,0.05);
\draw[->] (0.35,1.5) -- (-2.9,1.3) node[pos=0.5,sloped,anchor = center,above] {{\small $4$}};
\draw[->] (-2.9,1.3) -- (0.35,1.5);

\fill (3,4) circle (1.5pt) node[above, yshift=0.2cm] {\small $x_5$};
\draw[->] (0.1,0.1) -- (2.9,3.9) node[pos=0.5,sloped,anchor = center,below] {{\small $8$}};
\draw[->] (2.9,3.9) -- (0.1,0.1);
\draw[->] (1.1,0.2) -- (3.1,3.9) node[pos=0.5,sloped,anchor = center,below] {{\small $8$}};
\draw[->] (3.1,3.9) -- (1.1,0.2);
\draw[->] (0.6,1.7) -- (2.8,3.95) node[pos=0.5,sloped,anchor = center,above] {{\small $8$}};
\draw[->] (2.8,3.95) -- (0.6,1.7);
\draw[->] (-2.9,1.5) -- (2.7,4.1) node[pos=0.5,sloped,anchor = center,below] {{\small $8$}};
\draw[->] (2.7,4.1) -- (-2.9,1.5);
\end{tikzpicture}
\caption{Illustration of construction of $\Sigma_5 = \{ x_1,\ldots,x_5\}$}
\label{fig:stone_lemma_fails_construction} 
\end{figure}
%%%%%%%%%%%%%%%%%%%%%%%%%%%%%%%%%%%%%%%%%%%%%%%%%%%%%

We now show that $\Sigma_n$ has Nagata dimension 0, for every $n \in \mathbb{N}$, by using the induction argument. For $n=1$, we have $\Sigma_1 = \{x_1\}$, a singleton and hence Nagata dimension is 0. Suppose $\Sigma_{n-1}$ has Nagata dimension 0. Let $\mathcal{F}$ be a family of closed balls whose centers are in $\Sigma_{n}$. Now we have two cases:
\begin{enumerate}[(i)]
\item Suppose the point $x_{n}$ belongs to some ball $\bar{B}$ in $\mathcal{F}$. If $x_{n}$ is the center of $\bar{B}(x_n,r_n)$ and the radius $r_n \geq 2^{n-1}$, then the ball $\bar{B}(x_n,r_n)$ contain every point $x_i$ from $\Sigma_{n}$. Therefore, the subfamily $\mathcal{F}' \subseteq \mathcal{F}$ contains exactly one balls and covers the center of every ball in $\mathcal{F}$. This is equivalent to saying that the set $\Sigma_{n}$ has Nagata dimension 0. In the other case, when the radius of $\bar{B}(x_n,r_n)$ is $< 2^{n-1}$, the ball $\bar{B}(x_n,r_n)$ contains only $\{x_n\}$. If there is an data point $x_i$ such that $x_n$ is in $\bar{B}(x_i,r_i)$, then $r_i$ must be greater than or equals to $2^{n-1}$ and so the Nagata dimension is 0. Let $\mathcal{G} = \{ \bar{B}: \bar{B} \in \mathcal{F}, x_{n} \notin \bar{B} \}$ which is a family of balls in $\Sigma_{n-1}$. From the induction hypothesis, there is a subfamily $\mathcal{G}'$ of $\mathcal{G}$ with multiplicity one. So, the family $\mathcal{F}' = \mathcal{G}' \cup \bar{B}(x_n,r_n), r_n < 2^{n-1}$ contains every point of $\Sigma_{n}$ and has multiplicity one. So, $\Sigma_{n+1}$ has Nagata dimension 0.
The above argument covers the case when $x_n$ belongs to $\bar{B}$ but is not the center. 
\item Let $x_{n+1}$ does not belong to any ball in $\mathcal{F}$. This means that the radius of every ball in $\mathcal{F}$ is strictly less than $2^{n-1}$ and $\mathcal{F}$ is a family of balls in $\Sigma_{n-1}$. By the induction hypothesis, the Nagata dimension of $\Sigma_{n}$ is 0.
\end{enumerate}
We are interested in finding the expected numbers of data points from $\Sigma_n \setminus \{x_1\}$ having $x_1$ as their nearest neighbor. We break distance ties uniformly as following. The data point $x_1$ is the only nearest neighbor of $x_2$, so $x_1$ is chosen as the nearest neighbor of $x_2$ with probability 1. For $x_3$, there are two data points $x_1$ and $x_2$ closest to $x_3$ but at the same distance to $x_3$. So, the probability of $x_1$ being chosen as the nearest neighbor of $x_3$ is 1/2. Similarly for $x_n$, there are $n-1$ equidistant points $\{x_1,\ldots,x_{n-1}\}$ which are candidates for nearest neighbor of $x_n$. The probability of $x_1$ being chosen as the nearest neighbor of $x_n$ is $1/(n-1)$.
\begin{align*}
\mathbb{E}\bigg\{ \sum_{i=2}^{n} \mathbb{I}_{\{ x_1 \in \mathcal{N}_{1}(x_i) \} } \bigg\} = & \ \sum_{i=2}^{n}  \mathbb{E}\bigg\{ \mathbb{I}_{\{ x_1 \in \mathcal{N}_{1}(x_i)\}} \bigg\} \ \\
= & \ \sum_{i=1}^{n-1} \frac{1}{i} \ \\
> & \ \alpha.
\end{align*}
\end{proof}
The Lemma \ref{lem:nagata_fails_stone_lemma} shows that the Stone's lemma fails for spaces with finite Nagata dimension even if the distance ties are broken uniformly and randomly. So, there seems no hope for generalization of Stone's lemma in the presence of distance ties. Indeed, it is impossible to avoid distance ties. The following example demonstrates that a metric space with finite Nagata dimension can have many essential ties with high probability. A distance tie become an essential tie if it occurs at non-zero distance.
\begin{example}
\label{eg:ties_high_prob}
Let $0<\delta<1$ be any real number. There is a compact metric space with Nagata dimension zero (a Cantor set with a suitable metric) and a sequence $(n_k), n_{k} \rightarrow \infty, k/n_{k} \rightarrow 0$ such that for each $k$, the probability that a randomly chosen $n_{k +1}$-sample $X_1,X_2,\ldots,X_{n_k +1}$ has the property that $X_1$ has essential ties for $k$-nearest neighbors among $X_2,X_3,\ldots,X_{n_k}$ is $\geq 1- \delta$. In simpler words, $X_1$ is at the same distance to all its $(n_k-1)$-nearest neighbors $\{X_{2},\ldots, X_{n_k} \}$, with probability at least $1- \delta$.

\textit{Construction}: Let a sequence of positive reals $(\delta_i)_{i=1}^{\infty},\delta_i > 0 $ such that $2 \sum_{i=1}^{\infty} \delta_{i} <  \delta$. We construct two sequences  $(N_k)$ and $(n_k)$,where $N_k, n_k \in \mathbb{N}$, recursively. Let $\mu_{k}$ be the uniform measure on $[N_k]$ where $[N_k]$ denotes the set $\{ 1,2,\ldots,N_k\}$.

Step 1:  Let $n_1 > 1$ be any natural number. Choose $N_1$ so large that if we take a random $n_1$-sequence, whose elements are chosen independently and uniformly from $[N_1]$ then with probability $ > 1-\delta_1$ all elements of the $n_1$-sequence are pairwise different. 

Step 2: Choose $n_2$ so large that if we take a random $n_2$-sequence with elements independent and uniform in $[N_1]$, then the probability of every element of $[N_1]$ being chosen at least $n_1$ times is $> 1-\delta_1$.

Step 3: Next, we choose $N_2$ so large that if we take a random $n_2$-sequence having its elements chosen independently and uniformly in $[N_2]$, then with probability $ > 1-\delta_2$ all the elements of $n_2$-sequence are pairwise different.

Continuing the above steps gives us $(N_k)$ and $(n_k)$ such that $n_k, N_k \uparrow \infty$. Lets calculate the probability of having both the desired properties for every $k$. Let $A_k$ denote the property of having pairwise different elements in a random $n_k$-sequence with elements coming independently and uniformly from $[N_k]$. Let $B_{k}$ denote the property that in a random $n_{k+1}$-sequence whose elements comes uniformly from $[N_k]$,  every element of $[N_k]$ repeats at least $n_k$ times. From the construction, $\mathbb{P}(A_k) > 1 - \delta_k$ and $\mathbb{P}(B_k) > 1 - \delta_k$ Let $D_k = A_{k} \cap B_k$ be the event of both the properties being true for $k$-th recursive step. So, $\mathbb{P}(D_k) > 1 - 2 \delta_k$ and we find $\mathbb{P}( \cup_{k=1}^{\infty} D_k)$, which is the probability of both the properties holding simultaneously for every $k$. Using the union bound,
\begin{align*}
\mathbb{P}( \cup_{k=1}^{\infty} D_k) = & \ 1 - \mathbb{P}( \cap_{k=1}^{\infty} D_k^c) \ \\
\geq & \ 1 - \sum_{k=1}^{\infty} \mathbb{P}(  D_k^c) \ \\
\geq & \ 1 - 2 \sum_{k=1}^{\infty} \delta_k \ \\
> & \ 1  - \delta
\end{align*}
Set $\Omega  = \prod_{k=1}^{\infty} [N_k]$ and define a metric $\rho$ on $\Omega$, for any $\sigma,\tau \in \Omega$,
\begin{align*}
\rho(\sigma ,\tau) = \ 
\begin{cases}
0 \ \ \ \ \hspace{2.1cm} \text{if $\sigma = \tau$ } \ \\
2^{- \min\{i : \ \sigma_i \neq \tau_i \} } \ \ \ \text{otherwise }
\end{cases}
\end{align*}
By the Lemma \ref{eg:non_archi_metric}, $\rho$ is a non-Archimedean metric and hence the metric space $(\Omega,\rho)$ has Nagata dimension zero (from the Proposition \ref{cor:archi_Nagata}). Note that, the topology on $\Omega$ is the product topology and so $\Omega$ is a Cantor space. Let $\mu$ be the product measure of uniform measures $\mu_k$ on $[N_k]$ such that for any measurable set $S \subseteq \Omega, S = \prod_{i=1}^{\infty}S_i$, the measure $\mu(S) = \prod_{i=1}^{\infty} \mu_{i}(S_i)$. The measure $\mu$ is non-atomic and hence every distance tie will be essential.

Let $k$ be any natural number. We take a random $n_{k+1}$-sample $X_1,\ldots,X_{n_k} \\ ,\ldots,X_{n_{k +1}}$ using the distribution $\mu$ on $\Omega$.  The number $n_{k+1}$ is chosen so large that if we choose a word or a sample of length $n_{k+1}$ whose letter comes from $[N_k]$ then every element of $[N_k]$ should occur at least $n_k$ times. Suppose  $X_i = (x_{i1},x_{i2},\ldots)$ for $1 \leq i \leq n_{k+1}$, rearranging the terms we see that at least $n_{k}$ elements in the $k$-th coordinate, $\{ x_{1k},\ldots,x_{n_{k}k} \}$ are all equal with probability $> 1-\delta_k$. By the construction, $[N_1] \subseteq [N_2] \ldots \subseteq [N_{k}]$. Therefore, at least $n_k$ elements in each $i$-th coordinate, $\{ x_{1i},\ldots,x_{n_{k}i}$, for $1 \leq i \leq k$ \} are same. But we choose $N_{k+1}$ so large that if we choose randomly and uniformly letters from $[N_{k+1}]$ to make a word of length $n_{k+1}$ then all $n_{k+1}$ letters are different. 
The probability that, the $k$-th coordinate of $n_{k}$ elements $X_1,\ldots,X_{n_k}$ for $1 \leq i \leq k$ are same but the $k+1$-th coordinate are all different, is $> (1-\delta_k)^2 > 1- \delta$. The distance of $X_1$ to all other $n_{k} -1$ points is $\rho(X_1,X_i) = 2^{- (k+1)}$, so there are $n_{k}-1$ distance ties with positive probability.
\demo \end{example}
%%%%%%%%%%%%%%%%%%%%%%%%%%%%%%%%%%%%%%%%%%%%%%%%%%%%%
\subsection{Consistency in metrically sigma-finite dimensional spaces}
\label{subsec:Consistency in metrically sigma-finite dimensional space}
%%%%%%%%%%%%%%%%%%%%%%%%%%%%%%%%%%%%%%%%%%%%%%%%%%%%%
We can infer from the Lemma \ref{lem:nagata_fails_stone_lemma} and Example \ref{eg:ties_high_prob} that the existing tie-breaking method for Euclidean spaces may not yield similar results for general metric spaces with finite Nagata dimension and that it is impossible to find an analogue of Stone's lemma in such spaces in the presence of distance ties. Here, we present a key lemma that provides a way to deal with distance ties. Note that, we prove the results in this subsection for metrically finite dimensional spaces but the results also hold for metric spaces with finite Nagata dimension.
\begin{lemma}
\label{lem:assouad_type_sigma}
Let $(\Omega,\rho)$ be a metric space and let $Q \subseteq \Omega$ has metric dimension $\beta$ on scale $s$ in $\Omega$. Let $\Sigma_n = \{ x_1,\ldots,x_n\}$ be a finite sample in $\Omega$ and let $\tilde{\Sigma}$ be any sub-sample of $\Sigma_n$ with cardinality $m$. For $\alpha \in (0,1)$, let $T$ be the set of all $x_i$ in $\Sigma_n$ belonging to $Q$ whose $k$-nearest neighbor radius $\varepsilon_{kNN}(x_i)$ is strictly less than $s$ and the fraction of points in $\bar{B}(x_i,\varepsilon_{kNN}(x_i)) $ from $\tilde{\Sigma}$ is strictly greater than $\alpha$,
\begin{align}
T = \bigg\{ x_i \in \Sigma_n \cap Q: \varepsilon_{kNN}(x_i) < s, \frac{\sharp \{ \bar{B}(x_i,\varepsilon_{kNN}(x_i)) \cap \tilde{\Sigma} \}}{\sharp \bar{B}(x_i,\varepsilon_{kNN}(x_i))} > \alpha \bigg\}.
\end{align}
Then, the cardinality of $T$ is at most $\beta m/ \alpha$.
\end{lemma}
\begin{proof}
Let $\mathcal{F}$ be a family of closed balls with centers in $T$, 
$$ \mathcal{F} = \{ \bar{B}_i = \bar{B}(x_i,\varepsilon_{kNN}(x_i)) : x_i \in T \}. $$
 As, $Q$ has finite metric dimension, there exists a subfamily $\mathcal{F}'$ such that every $x_i$ in $T$ belongs to some ball in $\mathcal{F}'$ and any $x \in \Omega$ has multiplicity $\beta$ in $\mathcal{F}'$. By the definition of $T$, for every $x_i$ in $T$ we have,
\begin{align*}
 \sharp \bar{B}(x_i,\varepsilon_{kNN}(x_i)) & \leq \frac{1}{\alpha} \sharp \{\bar{B}(x_i,\varepsilon_{kNN}(x_i))\cap \tilde{\Sigma}\}  \
\end{align*}
Every point of $\tilde{\Sigma}$ can belong to at most $\beta$ balls in $\mathcal{F}'$ and so the total number of points from $\tilde{\Sigma}$ in $\mathcal{F}'$ can not be more than $\beta$ times the cardinality of $\tilde{\Sigma}$. Therefore, we have 
\begin{align*}
\sharp T & \ \leq \sum_{\bar{B}_i \in \mathcal{F}'}\sharp \bar{B}(x_i,\varepsilon_{kNN}(x_i)) \ \\ 
& \ \leq \frac{1}{\alpha} \sum_{\bar{B}_i \in \mathcal{F}'} \sharp \{\bar{B}(x_i,\varepsilon_{kNN}(x_i))\cap \tilde{\Sigma} \} \ \\
& \ \leq \frac{1}{\alpha} \beta m.
\end{align*}
\end{proof}
\begin{remark}
\label{rem:stone_ties_open}
As is seen from the proof, the Lemma \ref{lem:assouad_type_sigma} holds under more general assumptions:
\begin{enumerate}[(i)]
\item The result holds for closed balls of any radius strictly less than $s$, not necessarily only $\varepsilon_{kNN}$.  
\item The proof does not use the property of the balls being closed, so the result do hold for families of open balls with radius $< s$. \demo
\end{enumerate}
\end{remark}
We will need the following result to derive a stronger result from the Lemma \ref{lem:assouad_type_sigma} for the $k$-nearest neighbors sets.
\begin{lemma}
\label{lem:alpha_1}
Let $\alpha_1,\alpha_2,\alpha$ be non-negative real numbers. Let $t_1,t_2,t_3 \geq 0$ be such that $t_3 \leq t_2$, $t_1 + t_2 = 1$ and
$\alpha_1 t_1 + \alpha_2 t_2 \leq \alpha$.
Assume that $\alpha_1 \leq \alpha$, then
\begin{align*}
\frac{\alpha_1 t_1 + \alpha_2 t_3}{t_1 + t_3} \leq \alpha.
\end{align*}
\end{lemma}
\begin{proof}
If $\alpha_2 \leq \alpha$, then it is trivial. If $\alpha_2 > \alpha$,
\begin{align*}
\alpha_1 t_1 + \alpha_2 t_3 & \leq \alpha - \alpha_2 t_2 + \alpha_2 t_3 \ \\
& = \alpha - (1 - t_1) \alpha_2 + \alpha_2 t_3 \ \\
& \leq (t_1 + t_3) \alpha.
\end{align*}
\end{proof}
The following lemma shows that, if the fraction of points coming from a sub-sample in a closed as well as open ball at $x$ of radius $\varepsilon_{kNN}(x)$ is bounded above by some constant then, the fraction of $k$-nearest neighbors chosen from the sub-sample is also bounded by the same constant. 
\begin{lemma}
\label{lem:ball_to_NN}
Let $\Sigma_n = \{x_1,\ldots,x_n\}$ be a sample of $n$ points in any metric space $(\Omega,\rho)$. Let $\tilde{\Sigma}$ be a subset of $\Sigma_n$. Let $\alpha >0$ and let $x \in \Omega$. Suppose that the fraction of points from $\tilde{\Sigma}$, both in the closed ball $\bar{B}(x,\varepsilon_{kNN}(x))$ and in the open ball $B(x,\varepsilon_{kNN}(x))$ is at most $\alpha$,
\begin{align*}
\frac{\sharp \{ \bar{B}(x,\varepsilon_{kNN}(x)) \cap \tilde{\Sigma} \}}{\sharp \bar{B}(x,\varepsilon_{kNN}(x))}, 
\frac{\sharp \{ B(x,\varepsilon_{kNN}(x)) \cap \tilde{\Sigma} \}}{\sharp B(x,\varepsilon_{kNN}(x))} \leq \alpha.
\end{align*}  
The distance ties between the $k$-nearest neighbors of $x$ are broken randomly and uniformly. Then, the fraction of points from $\tilde{\Sigma}$ in the $k$-nearest neighbors set of $x$ is at most $\alpha$, that is,
\begin{align*}
\frac{ \sharp\{\mathcal{N}_{k}(x) \cap \tilde{\Sigma} \}}{\mathcal{N}_k(x)} \leq \alpha.
\end{align*}
\end{lemma}

\begin{proof}
In Lemma \ref{lem:alpha_1}, let $\alpha_1,\alpha_2$ be equal to the fraction of points from $\tilde{\Sigma}$ in the open ball and the sphere at $x$, respectively. Let $t_1$ be the fraction of points from the open ball at $x$ in the closed ball at $x$, that is,
\begin{align*}
& \alpha_1 = \frac{\sharp \{ B(x,\varepsilon_{kNN}(x)) \cap \tilde{\Sigma} \}}{\sharp B(x,\varepsilon_{kNN}(x))}, \alpha_2 = \frac{\sharp \{ S(x,\varepsilon_{kNN}(x)) \cap \tilde{\Sigma} \}}{\sharp S(x,\varepsilon_{kNN}(x))}, \\ & \ t_1 = \frac{\sharp \{ B(x,\varepsilon_{kNN}(x))\}}{\sharp \bar{B}(x,\varepsilon_{kNN}(x))}. 
\end{align*}
So, $t_2$ is the fraction of points from the sphere at $x$ in the closed ball $\bar{B}(x,\allowbreak \varepsilon_{kNN}(x))$. By our assumption, $\alpha_1 \leq \alpha$ and $\alpha_1 t_1 + \alpha_2 t_2$ (which is equal to the fraction of points from $\tilde{\Sigma}$ in the closed ball at $x$) is less than or equal to $\alpha$.

Since, $B(x,\varepsilon_{kNN}(x))$ contains at most $k$ points including $x$, so we chose the remaining $k$-nearest neighbors of $x$ uniformly from the sphere $ S(x, \allowbreak \varepsilon_{kNN}(x))$. Let $\nu_{S_{x}}$ be a uniform measure on $S_x = S(x,\varepsilon_{kNN}(x))$, then for any $A \subseteq \Sigma_n$,
\begin{align*}
\nu_{S_{x}}(A) = \frac{ \sharp\{ S(x,\varepsilon_{kNN}(x)) \cap A\}}{ \sharp\{S(x,\varepsilon_{kNN}(x))\}}.
\end{align*}
As, the event of choosing $k$-nearest neighbors of $x$ is independent of $\tilde{\Sigma}$, so we have
\begin{align*}
\nu_{S_{x}}(\mathcal{N}_{k}(x) \cap \tilde{\Sigma}) = \nu_{S_{x}}(\mathcal{N}_{k}(x)) \nu_{S_{x}}(\tilde{\Sigma}).
\end{align*}
Note that, $\alpha_2$ is the measure $\nu_{S_{x}}(\tilde{\Sigma})$ which is equal to,
\begin{align*}
\nu_{S_{x}}(\tilde{\Sigma}) & = \frac{\nu_{S_{x}}(\mathcal{N}_{k}(x) \cap \tilde{\Sigma})}{\nu_{S_{x}}(\mathcal{N}_{k}(x))} \ \\ 
& =  \frac{\sharp\{ S(x,\varepsilon_{kNN}(x)) \cap \mathcal{N}_{k(x)} \cap \tilde{\Sigma}\}}{\sharp\{S(x,\varepsilon_{kNN}(x)) \cap \mathcal{N}_{k}(x) \}}.
\end{align*}
Let $t_3$ be the fraction of $k$-nearest neighbors of $x$ chosen from the sphere at $s$ in the closed ball,  
\begin{align*}
t_3 = \frac{\sharp\{ S(x,\varepsilon_{kNN}(x)) \cap \mathcal{N}_{k}(x) \}}{\sharp\{\bar{B}(x,\varepsilon_{kNN}(x)) \}}.
\end{align*}
Now substituting all the values and using the above value for $\alpha_2$, we have 
\begin{align*}
\frac{t_1 \alpha_1 + t_3 \alpha_2}{t_1 +t_2} & = \frac{\sharp \{ \mathcal{N}_{k}(x) \cap \tilde{\Sigma}\} }{\sharp\{\mathcal{N}_{k}(x)\}} \\
& \leq \alpha \ \ \ \ \text{ (from the Lemma \ref{lem:alpha_1})}.
\end{align*}
\end{proof}
Now, we present our main result on the universal weak consistency of $k$-nearest neighbor rule in a metrically sigma-finite dimensional space where the distance ties are broken randomly and uniformly. 
\begin{theorem}
\label{stone_sigma_scales_ties}
Under the random and uniform tie-breaking method, the $k$-nearest neighbor rule is universally weakly consistent on a separable metrically sigma-finite dimensional space.
\end{theorem}
\begin{proof}
Let $(\Omega,\rho)$ be a separable metrically sigma-finite dimensional space. It follows from the Remark \ref{rem:fd_closed} that, $\Omega$ is an increasing union of closed and metrically finite dimensional sets $\{Q_i\}_{i=1}^{\infty}$. Each $Q_i$ is measurable because it is closed. Let $\nu$ and $\eta$ be a probability measure and a regression function on $\Omega$, respectively. The $\sigma$-additivity of $\nu$ implies that $\nu(Q_i)$ approaches $1$ as $i \rightarrow \infty$. Let $\varepsilon > 0$, then there exists $l \in \mathbb{N}$ sufficiently large such that 
\begin{align*}
\nu(Q_{l}) > 1 - \varepsilon /2 .
\end{align*}
Given $\varepsilon > 0$, the Luzin's theorem implies that there exists a compact subset $K \subseteq Q_l$ such that $\nu(K) > 1 - \varepsilon/2 $ and $\eta|_{K}$ is uniformly continuous. As,  $K$ is a subset of $Q_l$ so $K$ has metric dimension $\beta_l$ on the scale $s_l$ in $\Omega$ (by the Remark \ref{rem:subset_fd}). Let $U = \Omega \setminus K$ and hence $\nu(U) < \varepsilon$. 

From the Theorem \ref{lem:diff_err}, we know that the universal weak consistency follows if $\mathbb{E}\{(\eta_n(X) - \eta(X))^2\} \rightarrow 0$ whenever $n,k \rightarrow \infty$ and $k/n \rightarrow 0$. We use the inequality $(a+b)^2 \leq 2a^2 + 2 b^2$, where $a,b$ are real numbers, to obtain the following
\begin{align*} 
\mathbb{E}\{(\eta_{n}(X) - \eta(X))^2\} & = \mathbb{E}\{(\eta_{n}(X) - \tilde{\eta}(X) + \tilde{\eta}(X) - \eta(X))^2\} \ \\
& \leq 2 \mathbb{E}\{(\eta_{n}(X) - \tilde{\eta}(X))^2\} + 2\mathbb{E}\{(\tilde{\eta}(X) - \eta(X))^2\}
\end{align*}
The first term in the above equation goes to zero as $k$ increases to $\infty$ (by the Lemma \ref{lem:third_prop_all_metric_sp}). Now, we would show that the second term in the above equation also decreases to zero in the limit of $n$ and $k$. We see that from the Lemma \ref{lem:eta_tilde_reg}, we have the following bound on the second term, 
\begin{align}
\label{eqn:eta_star_bound}
& \mathbb{E}\{(\tilde{\eta}(X) - \eta(X))^2\} \nonumber \\
&\leq \mathbb{E}\bigg\{ \frac{1}{k} \sum_{i=1}^{n} \mathbb{I}_{\{X_i \in \mathcal{N}_{k}(X)\}} (\eta(X_i) - \eta^*(X_i) )^2 \bigg| X \in K, X_i \in U \bigg\} + 12\varepsilon.
\end{align}
Our aim is to bound from above the first term of right-hand side of the equation \eqref{eqn:eta_star_bound} by some constant (which is independent of $n$ and $k$) times $\varepsilon$. 

Given a random sample $(X_{0}, X_1, \allowbreak \ldots,X_n)$, let $R_{n+1}$ be the set of $X_j, 0 \leq j \leq n$ which belongs to $Q_l$ and have strictly greater than $k \sqrt{\varepsilon}$ of their $k$-nearest neighbors from $U$. That is, $R_{n+1}$ is the set of $X_j \in Q_l$ for which $\sharp\{i:X_i \in \mathcal{N}_{k}(X_j)\cap U \} > k \sqrt{\varepsilon}$. We first symmetrize the below expression using the normalized counting measure $\nu^{\sharp}$, defined on $\{0,1,\ldots,n\}$ and then divide into two cases: $X_j$ having $> k \sqrt{\varepsilon}$ of its $k$-nearest neighbors from $U$ and $X_j$ containing at most $k \sqrt{\varepsilon}$ of its $k$-nearest neighbors from $U$. Note that, $X_j$ take values in $K$ (which is a subset of $Q_l$) in the following expressions. So, we have
\begin{align*}
& \mathbb{E}\bigg\{ \frac{1}{k} \sum_{i=1}^{n} \mathbb{I}_{\{X_i \in \mathcal{N}_{k}(X)\}} (\eta(X_i) - \eta^*(X_i) )^2 | X \in K, X_i \in U \bigg\}  \nonumber \\
& = \mathbb{E}\bigg\{  \mathbb{E}_{j \sim \nu^{\sharp}}\bigg\{ \frac{1}{k} \sum_{i=0, i \neq j}^{n} \mathbb{I}_{\{X_i \in \mathcal{N}_{k}(X_j)\}} (\eta(X_i) - \eta^*(X_i) )^2 | X_j \in K, X_i \in U \bigg\} \bigg\} \nonumber 
\end{align*}
\begin{align}
& =  \mathbb{E}\bigg\{  \mathbb{E}_{j \sim \nu^{\sharp}}\bigg\{ \frac{1}{k} \sum_{\substack{i=0 \\ i\neq j}}^{n} \mathbb{I}_{\{X_i \in \mathcal{N}_{k}(X_j)\}} (\eta(X_i) - \eta^*(X_i) )^2 | X_j \in K \cap R_{n+1}, X_i \in U \bigg\} \bigg\}  \label{eqn:cases_d} \\   & + \mathbb{E}\bigg\{  \mathbb{E}_{j \sim \nu^{\sharp}}\bigg\{ \frac{1}{k} \sum_{\substack{i=0 \\ i\neq j}}^{n} \mathbb{I}_{\{X_i \in \mathcal{N}_{k}(X_j)\}} (\eta(X_i) - \eta^*(X_i) )^2 | X_j \in K\setminus R_{n+1}, X_i \in U \bigg\} \bigg\} \label{eqn:cases_d_not}
\end{align}

\begin{description}
\item[Equation \eqref{eqn:cases_d}:] 
Let $T_{n+1}$ denote the set of $X_j \in Q_l$ which contain $> \sqrt{\varepsilon}$ fraction of points from $U$ in its open ball $B(X_j,\allowbreak \varepsilon_{kNN}(X_j))$, and let $\tilde{T}_{n+1}$ denote the set of $X_j \in Q_l$ which contains $> \sqrt{\varepsilon}$ fraction of points from $U$ in its closed ball $\bar{B}(X_j,\varepsilon_{kNN}(X_j))$. 

If there is a distance tie, then the $k$-nearest neighbors of $X_j$ is chosen randomly and uniformly from the sphere $S(X_j,\varepsilon_{kNN}(X_j))$, so $\allowbreak \sharp\{\mathcal{N}_{k}(X_j)\}\allowbreak = k$. It follows from the Lemma \ref{lem:ball_to_NN} that for $X_j$, if the fraction of $k$-nearest neighbors of $X_j$ from $U$ is strictly greater than $\sqrt{\varepsilon}$, then either, the fraction of points from $U$ in the closed ball $\bar{B}(X_j,\varepsilon_{kNN}(X_j))$ is strictly greater than $\sqrt{\varepsilon}$ or, the fraction of points  from $U$ in the open ball $B(X_j,\varepsilon_{kNN}(X_j))$ is strictly greater than $\sqrt{\varepsilon}$. Numerically, if $\mathcal{N}_{k}(X_j) \cap U > k \sqrt{\varepsilon} = \mathcal{N}_{k}(X_j) \sqrt{\varepsilon}$, then either
\begin{align*}
\frac{\sharp\{\bar{B}(X_j,\varepsilon_{kNN}(X_j)) \cap U \}}{ \sharp\{\bar{B}(X_j,\varepsilon_{kNN}(X_j))\}} > \sqrt{\varepsilon} \text{ or, } \frac{\sharp\{B(X_j,\varepsilon_{kNN}(X_j)) \cap U \}}{ \sharp\{B(X_j,\varepsilon_{kNN}(X_j))\}} > \sqrt{\varepsilon}.
\end{align*}
So, the equation \eqref{eqn:cases_d} can be bounded as,
\begin{align*}
&  \mathbb{E}\bigg\{  \mathbb{E}_{j \sim \nu^{\sharp}}\bigg\{ \frac{1}{k} \sum_{i=0, i \neq j}^{n} \mathbb{I}_{\{X_i \in \mathcal{N}_{k}(X_j) \cap U\}} (\eta(X_i) - \eta^*(X_i) )^2 | X_j \in K \cap R_{n+1} \bigg\} \bigg\} \ \\
& \leq \mathbb{E}\bigg\{  \mathbb{E}_{j \sim \nu^{\sharp}}\bigg\{ \frac{1}{k} \sum_{i=0, i \neq j}^{n} \mathbb{I}_{\{X_i \in \mathcal{N}_{k}(X_j)\}} (\eta(X_i) - \eta^*(X_i) )^2 | X_j \in K \cap T_{n+1}\bigg\} \bigg\}   \ \\  
& + \mathbb{E}\bigg\{  \mathbb{E}_{j \sim \nu^{\sharp}}\bigg\{ \frac{1}{k} \sum_{i=0, i \neq j}^{n} \mathbb{I}_{\{X_i \in \mathcal{N}_{k}(X_j)\}} (\eta(X_i) - \eta^*(X_i) )^2 | X_j \in K\cap \tilde{T}_{n+1}\bigg\} \bigg\}   \ \\  
& \leq \mathbb{E}\bigg\{  \mathbb{E}_{j \sim \nu^{\sharp}}\{ \mathbb{I}_{\{X_j \in T_{n+1}\}} | X_j \in K\} \bigg\}+\mathbb{E}\bigg\{  \mathbb{E}_{j \sim \nu^{\sharp}}\{ \mathbb{I}_{\{X_j \in \tilde{T}_{n+1}\}} | X_j \in K \} \bigg\}  \ \\
& =  \mathbb{E}\bigg\{ \frac{\sharp T_{n+1}}{n+1} \bigg\} +\mathbb{E}\bigg\{ \frac{\sharp \tilde{T}_{n+1}}{n+1} \bigg\}.
\end{align*}
The Lemma \ref{lem:assouad_type_sigma} together with Remark \ref{rem:stone_ties_open} implies that, 
\begin{align*} 
  \mathbb{E}\bigg\{ \frac{\sharp T_{n+1}}{n+1} \bigg\} +  \mathbb{E}\bigg\{ \frac{\sharp \tilde{T}_{n+1}}{n+1} \bigg\} & \leq   2 \frac{\beta_l}{\sqrt{\varepsilon}(n+1)}\mathbb{E} \bigg\{ \sum_{i =0}^{n} \mathbb{I}_{\{ X_i  \in  U \}}  \bigg\}   \\
&  =  \frac{2\beta_l}{\sqrt{\varepsilon}} \nu(U)  \\ 
&  <  2\beta_l \sqrt{\varepsilon},
\end{align*}
where we used the law of large numbers. 
\item[Equation \eqref{eqn:cases_d_not}:] If $X_j$ is not in $R_{n+1}$, this means the there can be at most $k \sqrt{\varepsilon}$ of $k$-nearest neighbors of $X_j$ that belongs to $U$ after breaking distance ties. So, we have 
\begin{align*}
\text{Equation \eqref{eqn:cases_d_not}} & \leq \mathbb{E} \bigg\{ \frac{1}{k} \sum_{i=0, i \neq j}^{n} \mathbb{I}_{\{X_i \in \mathcal{N}_k(X)\}}\mathbb{I}_{ \{ X_i \in U\}}  \bigg| X_j \in K \setminus R_{n+1}  \bigg\}  \\
& \leq \frac{1}{k} k \sqrt{\varepsilon}  =  \sqrt{\varepsilon}.
\end{align*}
\end{description}
\end{proof}
Now that we have established the universal weak consistency of the $k$-nearest neighbor rule, we aim for the strong consistency in such metric spaces. This is an obvious direction because as shown in \cite{Devroye_Gyorfi_Lugosi_1996}, the weak consistency and strong consistency are equivalent in Euclidean spaces. The next section discusses the strong consistency in metrically finite dimensional spaces. 
%%%%%%%%%%%%%%%%%%%%%%%%%%%%%%%%%%%%%%%%%%%%%%%%%%%%%
\section{Strong consistency}
\label{sec:Strong consistency}
%%%%%%%%%%%%%%%%%%%%%%%%%%%%%%%%%%%%%%%%%%%%%%%%%%%%%
A learning rule is strongly consistent if for almost every infinite sample path, the conditional error probability given a finite set of first $n$ sample points from the infinite sample path, converges to Bayes error as the sample size $n$ increases. The strong consistency in Euclidean spaces was proved by Devroye et al. \cite{Devroye_Gyorfi_1985, Zhao_1987} under the assumption of no distance ties. The argument was based on cones in Euclidean spaces and hence the proof is limited to Euclidean spaces.  The strong consistency in the presence of distance ties was proved \cite{Devroye_Gyorfi_Krzyzak_Lugosi_1994} ten years later, as distance ties is a difficult hurdle to overcome.

Therefore, in this thesis we will only examine the strong consistency under the assumption of zero probability of distance ties. In particular, we establish the strong consistency of the $k$-nearest neighbor rule in metric spaces with finite metric dimension under the assumption that the distance ties occur with zero probability. Our proof is based on a similar argument as given in Theorem 11.1 on pp. 170-174 of \cite{Devroye_Gyorfi_Lugosi_1996}, but is based on a different geometry.

Let $0< \alpha \leq 1$ be a real number, define
\begin{align}
\label{eqn:r_alpha}
r_{\alpha}(x) = \inf\{r >0 : \nu(B(x,r)) \geq \alpha \}.
\end{align}	
A tie occurs with zero probability means the probability of a sphere is zero. We prove in the following lemma that the open ball at $x$ of radius $r_{\alpha}(x)$ has measure exactly equal to $\alpha$, if the measure of every sphere is zero.
\begin{lemma}
\label{lem:ties_zero}
Let $\nu$ be a probability measure with zero probability of ties. Then, $\nu(B(x,r_{\alpha}(x))) = \alpha$ for every $x$.
\end{lemma}
\begin{proof}
If $t < r_{\alpha}(x)$, then $\nu(\bar{B}(x,t)) < \alpha$. We can find a chain of subsets $\bar{B}(x,t)$ that increases to $B(x,r_{\alpha}(x))$. So, $\nu(B(x,r_{\alpha}(x))) \leq \alpha$.
Similarly if $t > r_{\alpha}(x)$, then $\nu(B(x,t)) \geq \alpha$. We can find a chain of open subsets $B(x,t)$ that decreases to $\bar{B}(x,r_{\alpha}(x))$. So, $\nu(\bar{B}(x,r_{\alpha}(x))) \geq \alpha$.
 The zero probability of distance ties means $\nu(S(x,r_{\alpha}(x))) = 0$, therefore $\nu(B(x,r_{\alpha}(x))) = \alpha$. 
\end{proof}
Turns out, the function $r_{\alpha}$ is 1-Lipschitz continuous and has a point-wise limit.
\begin{lemma}
\label{lem:r_alpha}
Let $r_{\alpha}$ be a real-valued function defined as in \eqref{eqn:r_alpha}, then $r_{\alpha}$ is a 1-Lipschitz continuous function. Also, $r_{\alpha}$ converges to $0$ as $\alpha \rightarrow 0$ at each point of the support of the measure.  
\end{lemma}
\begin{proof}
Let $\delta > 0$ be any real number. This means $\nu(B(x, r_{\alpha}(x) + \delta)) \geq \alpha$. 
This implies that $\nu(B(y, \rho(x,y)+r_{\alpha}(x) + \delta)) \geq \alpha$ and so, $r_{\alpha}(y) \leq \rho(x,y)+ r_{\alpha}(x) + \delta$. As $\delta $ is arbitrary, we have $r_{\alpha}(y) \leq \rho(x,y)+ r_{\alpha}(x)$. Therefore, $r_{\alpha}$ is a $1$-Lipschitz continuous function.

We will use the $(\epsilon,\delta)$-definition to show that $r_{\alpha} \rightarrow 0$ as $\alpha \rightarrow 0$ for every element in support of $\nu$, that is, for every $\epsilon > 0$, we will find a $\delta >0$ such that $r_{\alpha}(x) \leq \epsilon$ whenever $ \alpha \leq \delta$, $x \in S_{\nu}$. 

Let $\epsilon > 0$. We observe that $r_{\alpha}(x) \leq \varepsilon$ if and only if $\alpha \leq \nu(B(x,\epsilon))$. If $x \in S_{\nu}$, then $\nu(B(x,\epsilon)) >0$, which is the our $\delta$ corresponding to $\epsilon$. So, for every $x$ in the support of $\nu$, the sequence $r_{\alpha}(x) \rightarrow 0$ as $\alpha \rightarrow 0$. If $x \notin S_{\nu}$, then there exists a $\epsilon >0$ such that $\nu(B(x,\epsilon)) =0$. As, $\alpha > \nu(B(x,\epsilon)) =0 $, then $r_{\alpha}(x) > \epsilon$. Thus, for $x \notin S_{\nu}$, $r_{\alpha}$ does not converge to 0 as $\alpha \rightarrow 0$.
\end{proof}
Based on the properties of $r_{\alpha}$, we show in the following lemma that the measure of all elements from a metrically finite dimensional space containing a fixed point in its $r_{\alpha}$-ball is bounded above by the metric dimension times $\alpha$.
\begin{lemma}
\label{lem:stone_lemma_strong}
Let $Q$ be a separable metric space which has metric dimension $\beta$ on scale $s$. Assume that $\nu$ is a probability measure on $Q$ with zero probability of ties. 
For $y \in Q$, define 
\begin{align*}
A = \{ x \in Q : y \in B(x,r_{\alpha}(x)) \}.
\end{align*}
Then, we have $\nu(A) \leq \beta \alpha$ for $\alpha$ small enough. 
\end{lemma}
\begin{proof}
Let $\varepsilon >0$ be any real number. By Luzin's theorem, there is a compact set $K \subseteq A$ such that  $\nu(A \setminus K) < \varepsilon$. So, we need to estimate only the value of $\nu(K)$. 

It follows from the Lemma \ref{lem:r_alpha} that $r_{\alpha}$ is 1-Lipschitz continuous and $r_{\alpha}$ converges to 0 as $\alpha$ goes to 0, $\nu$-almost everywhere. Therefore, $r_{\alpha}$ converges to 0 uniformly on $K$, whenever $\alpha$ goes to 0. This means that there exists a $\alpha_0 >0$ such that for $0 < \alpha \leq \alpha_0$, we have $r_{\alpha}(x) < s$ for all $x \in K$.

Every open ball $B(x,r_{\alpha}(x))$ centered at $x \in K$ contains $y$, then we have for every $x \in K$ 
\begin{align}
\label{eqn:subset_r_alpha}
\bar{B}(x,\rho(x,y)) \subseteq B(x,r_{\alpha}(x)).
\end{align}
Let $D = \{ a_n: n \in \mathbb{N}\}$ be a countable dense subset of $K$. For each $n$, we select a family of closed balls $\bar{B}(a_i,\rho(a_i,y)), 1 \leq i \leq n$. Since, $Q$ has metric dimension $\beta$ on scale $s$, there exists a set of $\beta$ centers $\{x_1^{n},\ldots,x_{\beta}^{n}\} \subseteq \{a_1,\ldots,a_n\}$ such that $\cup_{i=1}^{\beta} \bar{B}(x_i^n,\rho(x_i^n,y))$ covers $\{a_1,\ldots,a_n\}$. As, $K$ is compact so every sequence has a sub-sequence which converges in $K$. For $i=1$, there is a sub-sequence $(n_1)$ of $(n)$ such that $(x_{1}^{n_1})$ converges to $x_1$. Similarly for $i=2$, there is a sub-sequence $(n_2)$ of $(n_1)$ such that $(x_{2}^{n_2})$ converges to $x_2$. Doing recursively until $i=\beta$, we have a sequence of indices $(n_{\beta})$ such that	$(x_1^{n_{\beta}}, \ldots, x_{\beta}^{n_{\beta}})$ converges to $(x_1,\ldots,x_{\beta})$ as $n_{\beta} \rightarrow \infty$.

We claim that the union of $\bar{B}(x_i, \rho(x_i,y) ), 1 \leq i \leq \beta$ covers $K$. As closure of finite union is the union of closures and since the balls are closed, it is enough to show that $D=\{a_m\}_{m \in \mathbb{N}}$ is contained in the union of $\bar{B}(x_i, \rho(x_i,y) ), \allowbreak 1 \leq i \leq \beta$. For $n_{\beta} \geq m$, $a_m$ belongs to at least one of the $\beta$ balls $\bar{B}(x_{i}^{n_{\beta}},\rho(x_{i}^{n_{\beta}},y))$. Then there is an $i_0$ such that  $a_m \in  \bar{B}(x_{i_0}^{n_{\beta}},\rho(x_{i_0}^{n_{\beta}},y))$ for infinitely many values of $n_{\beta} \geq m$. This means there is a sub-sequence $(n')$ such that $a_m \in \bar{B}(x_{i_0}^{n'},\rho(x_{i_0}^{n'},y))$, where $x_{i_0}^{n'} \rightarrow x_{i_0}$. Now, we will show that $a_m$ is closer to $x_{i_0}$ than $y$.

We have,
\begin{align*}
\rho(a_m,x_{i_0}) & = \rho(a_m, \lim_{n \rightarrow \infty} x_{i_0}^{n'} ) \ \\
& = \lim_{n' \rightarrow \infty} \rho(a_m,x_{i_0}^{n'}) \ \\
& \leq \lim_{n' \rightarrow \infty} \rho(x_{i_0}^{n'},y) \ \\
& = \rho(x_{i_0},y).
\end{align*} 
Therefore, $a_m$ is an element of $\bar{B}(x_{i_0},\rho(x_{i_0},y))$. It follows from the equation \eqref{eqn:subset_r_alpha} that the family $\{B(x_{i_0},r_{\alpha}(x_{i_0})): 1 \leq i_0 \leq \beta\}$ covers $K$.

By our assumption of zero probability of ties, we have $\nu(B(x,r_{\alpha}(x))) = \alpha$ (from the Lemma \ref{lem:ties_zero}). Further, the sub-additivity of $\nu$ implies that $\nu(K) \leq \beta \alpha$, and so
\begin{align*}
\nu(A) & = \nu( K) + \nu(A \setminus K) \ \\
& \leq \beta\alpha + \varepsilon,
\end{align*}
where $\alpha \leq \alpha_0$. As, $\varepsilon $ is arbitrary we have $\nu(A) \leq \beta \alpha$.
\end{proof}

As a consequence of Lemma \ref{lem:stone_lemma_strong}, we have exponential concentration on the probability of difference between conditional error probabilities of the $k$-nearest neighbor rule and the Bayes rule. The following theorem was proved in Euclidean spaces (Theorem 11.1 of \cite{Devroye_Gyorfi_Lugosi_1996}). However, the proof remains more or less same for metrically finite dimensional spaces except that we use the Lemma \ref{lem:stone_lemma_strong} instead of lemma based on Stone's idea with the cones.
\begin{theorem}
\label{thm:suc_finite_dim}
Let $(Q,\rho)$ be a separable metric space such that $Q$ has metric dimension $\beta$ on scale $s$. Let $\mu$ be a probability measure on $Q \times \{0,1\}$ and assume that $\nu$ on $Q$ obtained using $\mu$, has zero probability of ties. Let $g_n$ be the $k$-nearest neighbor rule. For $\varepsilon > 0$, there is a $n_0$ such that for $n > n_0$,
\begin{align*}
\mathbb{P}\bigg( \ell_{\mu}(g_n) - \ell^*_{\mu} > \varepsilon \bigg) \leq 2 e^{-\frac{n\varepsilon^2}{18 \beta^2}},
\end{align*}
whenever $k,n \rightarrow \infty$ and $k/n \rightarrow 0$.
\end{theorem}
\begin{proof}
Let $D_n$ be a random labeled sample, then $\ell_{\mu}(g_n) = \mathbb{P}(g_{n}(X) \neq Y | D_n)$ is a function of $D_n$ and hence a random variable. From the Theorem \ref{lem:diff_err}, we have that 
\begin{align*}
\ell_{\mu}(g_n) - \ell^*_{\mu} \leq 2 \mathbb{E}_{\nu}\bigg\{|\eta(X) - \eta_n(X)| \bigg| D_n \bigg\}.
\end{align*}
Therefore, it is sufficient to show that 
\begin{align*}
\mathbb{P}\bigg( \mathbb{E}_{\nu}\bigg\{|\eta(X) - \eta_n(X)| \bigg| D_n \bigg\} > \frac{\varepsilon}{2} \bigg) \leq 2 e^{-\frac{ n\varepsilon^2}{18 \beta^2}}, 
\end{align*} 
We shall omit writing the expectation conditional on $D_n$ to avoid unnecessary complicated notations with an understanding that the expectation of $|\eta(X) - \eta_n(X)|$ is still a random variable. Therefore, it is sufficient to show that  
\begin{align*}
\mathbb{P}\bigg( \mathbb{E}_{\nu}\{|\eta(X) - \eta_n(X)|\} > \frac{\varepsilon}{2} \bigg) \leq 2 e^{-\frac{ n\varepsilon^2}{18 \beta^2}}, 
\end{align*}
where $\eta_{n}(X) = \frac{1}{k}\sum_{i=1}^{n} \mathbb{I}_{\{X_i \in \mathcal{N}_{k}(X)\}}Y_i$. Let ${\eta}_{n}^*$  be another approximation of $\eta$,
\begin{align}
\label{eqn:eta_star}
{\eta}_{n}^*(X) = \frac{1}{k}\sum_{i=1}^{n} \mathbb{I}_{\{\rho(X_i,X) < r_{\alpha}(X)\}} Y_i.
\end{align} 
By the triangle's inequality, we have 
\begin{align*}
|\eta(X) - \eta_n(X)| \leq |\eta(X) - {\eta}_n^*(X)| + |{\eta}_n^*(X) - \eta_n(X)|. 
\end{align*}
For the second term on the right-hand side of above equation,  
\begin{align}
 |{\eta}_n^*(X) - \eta_n(X)| & = \frac{1}{k}\bigg|\sum_{i=1}^{n} \mathbb{I}_{\{\rho(X_i,X) < r_{\alpha}(X)\}} Y_i - \sum_{i=1}^{n} \mathbb{I}_{\{X_i \in \mathcal{N}_{k}(X)\}}Y_i \bigg| \nonumber \\
& = \frac{1}{k} \sum_{i=1}^{n} \bigg| \mathbb{I}_{\{\rho(X_i,X) < r_{\alpha}(X)\}}  -   \mathbb{I}_{\{X_i \in \mathcal{N}_{k}(X)\}} \bigg| \nonumber \ \\
& \leq \bigg|\frac{1}{k} \sum_{i=1}^{n}  \mathbb{I}_{\{\rho(X_i,X) < r_{\alpha}(X)\}}  -   1 \bigg|, \label{eqn:eta_bound}
\end{align}
where the last inequality is because $\mathcal{N}_k(X)$ contains at most $k$ points. Let $\hat{\eta}_{n}(X)$ be equal to $\frac{1}{k} \sum_{i=1}^{n}  \mathbb{I}_{\{\rho(X_i,X) < r_{\alpha}(X)\}}$ and let $\hat{\eta}(X)$ be equal to 1 always. Therefore, we have 
\begin{align*}
|\eta(X) - \eta_n(X)| \leq |\eta(X) - {\eta}_n^*(X)|  + |\hat{\eta}_n(X) - \hat{\eta}(X)|.
\end{align*}
The idea is to obtain the exponential concentration for the two terms of above equation, separately, using the McDiarmid's inequality (see Theorem \ref{thm:mcdiarmid}). So, we first show that the expected values of the integrals of the terms on the right-hand side of the above equation goes to zero. 
\begin{enumerate}[(i)]
\item From the equation \eqref{eqn:eta_bound} and using Cauchy-Schwarz inequality, we have
\begin{align*}
\mathbb{E}_{\mu^n}\{\mathbb{E}_{\nu}\{|{\eta}_n^*(X) - \eta_n(X)|\} \} & \leq \mathbb{E}_{\mu^n}\{\mathbb{E}_{\nu}\{|\hat{\eta}_n(X) - \hat{\eta}(X)|\} \} \  \\
& \leq \mathbb{E}_{\nu}\bigg\{ \sqrt{\mathbb{E}_{\mu^n}\{|\hat{\eta}_n(X) - \hat{\eta}(X)|\}} \bigg\} \  \\
& \leq \mathbb{E}_{\nu}\bigg\{ \sqrt{ \frac{n}{k^2}Var\{\mathbb{I}_{\{\rho(X_i,X) < r_{\alpha}(X)\}} \} } \bigg\} \  \\
& \leq \mathbb{E}_{\nu}\bigg\{ \sqrt{ \frac{n}{k^2}\nu(B(X,r_{\alpha}(X))) }  \bigg\} \  \\
& = \mathbb{E}_{\nu}\bigg\{ \sqrt{ \frac{n}{k^2}\alpha } \bigg\}. \ 
\end{align*}
As $\alpha$ is small, we can take $\alpha \leq k/n$ for large enough values of $n,k$. Substituting $\alpha \leq k/n$ in the above equation we have,  
\begin{align*}
\mathbb{E}_{\mu^n}\{\mathbb{E}_{\nu}\{|{\eta}_n^*(X) - \eta_n(X)|\} \} & \leq \mathbb{E}_{\nu}\bigg\{ \sqrt{ \frac{n}{k^2} \frac{k}{n} } \bigg\} \ \\
& = \frac{1}{\sqrt{k}}, 
\end{align*}
which goes to zero as $k \rightarrow \infty$.
\item We proved the following result while establishing the universal weak consistency in the Theorem \ref{stone_sigma_scales_ties}, 
\begin{align*}
\mathbb{E}_{\mu^n}\{|\eta(X) - \eta_n(X)| \} \rightarrow 0 \text{ as } n,k \rightarrow \infty, k/n \rightarrow 0.
\end{align*}
Using Fubini's theorem followed by the aforementioned result and case (i) implies that,
\begin{align*}
& \mathbb{E}_{\mu^n}\{\mathbb{E}_{\nu}\{|\eta(X) - {\eta}_n^*(X)|\} \} \ \\ & \leq  \mathbb{E}_{\mu^n}\{\mathbb{E}_{\nu}\{|\eta(X) - {\eta}_n(X)|\} \} + \mathbb{E}_{\mu^n}\{\mathbb{E}_{\nu}\{|\eta_n(X) - {\eta}_n^*(X)|\} \} \ \\
& \leq  \mathbb{E}_{\nu}\{\mathbb{E}_{\mu^n}\{|\eta(X) - {\eta}_n(X)|\} \} + \mathbb{E}_{\mu^n}\{\mathbb{E}_{\nu}\{|\eta_n(X) - {\eta}_n^*(X)|\} \} \ \\
& \rightarrow 0  \text{ as } n,k \rightarrow \infty, k/n \rightarrow 0.
\end{align*}
\end{enumerate}
So, we can choose $n,k$ so large that for a given $\varepsilon > 0$,
\begin{align}
\mathbb{E}_{\mu^n}\{\mathbb{E}_{\nu}\{|\eta(X) - {\eta}_n^*(X)|\} \} + \mathbb{E}_{\mu^n}\{\mathbb{E}_{\nu}\{|\hat{\eta}_n(X) - \hat{\eta}(X)|\} \} & < \frac{\varepsilon}{6}. \label{eqn:exp_1}
\end{align}
Therefore, we have
\begin{align}
&\mathbb{P}\bigg( \mathbb{E}_{\nu}\{|\eta(X) - \eta_n(X)|\} > \frac{\varepsilon}{2} \bigg) \nonumber \\ 
& \leq \mathbb{P}\bigg( \mathbb{E}_{\nu}\{|\eta(X) - {\eta}_n^*(X)|\} + \mathbb{E}_{\nu}\{|\hat{\eta}_n(X) - \hat{\eta}(X)| \} > \frac{\varepsilon}{2} \bigg)  \nonumber \\
& = \mathbb{P}\bigg( \mathbb{E}_{\nu}\{|\eta(X) - {\eta}_n^*(X)|\} - \mathbb{E}_{\mu^n}\{\mathbb{E}_{\nu}\{|\eta(X) - {\eta}_n^*(X)|\} \}  +  \nonumber \\ & \ \ \ \ \ \ \ \ \mathbb{E}_{\nu}\{|\hat{\eta}_n(X) - \hat{\eta}(X)|  \}   -  \mathbb{E}_{\mu^n}\{\mathbb{E}_{\nu}\{|\hat{\eta}_n(X) - \hat{\eta}(X)| \} \} > \frac{\varepsilon}{3} \bigg)  \nonumber \\
& \leq \mathbb{P}\bigg( \mathbb{E}_{\nu}\{|\eta(X) - {\eta}_n^*(X)|\} - \mathbb{E}_{\mu^n}\{\mathbb{E}_{\nu}\{|\eta(X) - {\eta}_n^*(X)|\} \} > \frac{\varepsilon}{6} \bigg) +  \nonumber \\ & \ \ \ \  \mathbb{P}\bigg(\mathbb{E}_{\nu}\{|\hat{\eta}_n(X) - \hat{\eta}(X)|  \}   - \mathbb{E}_{\mu^n}\{\mathbb{E}_{\nu}\{|\hat{\eta}_n(X) - \hat{\eta}(X)| \} \} > \frac{\varepsilon}{6} \bigg),  \label{eqn:bound_suc}
\end{align}
where the second equation in the above set of equations is obtained using the inequality \eqref{eqn:exp_1}.

Let $\theta$ be a function defined on labeled samples, $\theta: (Q \times \{0,1\})^n \rightarrow [0,\infty)$ as,
\begin{align*}
\theta(\sigma_n) = \mathbb{E}_{\nu}\{|{\eta}(X) - \eta_n^*(X)| \}
\end{align*}
Let a new sample $\sigma_n^{'}$ is formed by replacing $(x_i,y_i)$ by $(\hat{x}_i,\hat{y}_i)$. Let ${\eta}_{ni}^*(X)$ denote the changed value of $\eta_n^*$ as defined in \eqref{eqn:eta_star}, with respect to the new sample $\sigma_n^{'}$. Then, we have
\begin{align*}
| \theta(\sigma_n) - \theta(\sigma_n^{'} ) | & = \bigg|\mathbb{E}_{\nu}\{|{\eta}(X) - \eta_n^*(X)| \} - \mathbb{E}_{\nu}\{|\eta(X) - {\eta}_{ni}^*(X)| \} \bigg| \\
& \leq \mathbb{E}_{\nu}\{|{\eta}_n^*(X) - {\eta}_{ni}^*(X)| \}.
\end{align*}
Now, we calculate the value of 
\begin{align*}
|{\eta}_n^*(X) - {\eta}_{ni}^*(X)| & = \frac{1}{k} \bigg| \mathbb{I}_{\{\rho(X_i,X) < r_{\alpha}(X)\}} Y_i - \mathbb{I}_{\{\rho(\hat{X}_i,X) < r_{\alpha}(X)\}} \hat{Y}_i  \bigg| \ \\
& \leq \frac{1}{k} \mathbb{I}_{\{\rho(X_i,X) < r_{\alpha}(X)\}}
\end{align*}
So, we have
\begin{align*}
|\theta(\sigma_n) - \theta(\sigma_n^{'} ) |& \leq \frac{1}{k}\mathbb{E}_{\nu}\{\mathbb{I}_{\{\rho(X_i,X) < r_{\alpha}(X)\}} \} \ \\
& =  \frac{1}{k}\nu( B(x, r_{\alpha}(x))).
\end{align*}
It follows from the Lemma \ref{lem:stone_lemma_strong}, 
\begin{align*}
\sup_{x_1,y_1,\ldots,x_n,y_n,\hat{x}_i,\hat{y}_i}|\theta(\sigma_n) - \theta(\sigma_n^{'} ) |& \leq \frac{1}{k} \beta \alpha \ \\
 & \leq \frac{\beta}{n}.
\end{align*}
The above expression is true for all $1 \leq i \leq n$ and for every sample $\sigma_n$ and $\sigma'_{n}$ in $(Q \times \{0,1\})^n$. By the McDiarmid's  inequality (Theorem \ref{thm:mcdiarmid} in appendix), we get the following inequality 
\begin{align}
\label{eqn:one_bound}
\mathbb{P}\bigg(  \mathbb{E}_{\nu}\{|\eta(X) - {\eta}_n^*(X)| - \mathbb{E}_{\mu^n}\{\mathbb{E}_{\nu}\{|{\eta}(X) - {\eta}_n^*(X)|\} \}  > \frac{\varepsilon}{6}  \bigg) \leq  e^{-\frac{n \varepsilon^2}{18 \beta^2}}.
\end{align}
As, $\hat{\eta}_n$ is defined like $\eta^{*}_n$, we can define a new function $\tilde{\theta}(\sigma_n) = \mathbb{E}_{\nu}\{|{\hat{\eta}}_n(X) - \hat{\eta}(X)| \}$ and in a similar manner as presented above, we obtain that $|\tilde{\theta}(\sigma_n) - \tilde{\theta}(\sigma_n^{'} ) | \leq \beta /n $. Therefore, we have the following the exponential concentration (by the  McDiarmid's inequality),
\begin{align}
\label{eqn:second_bound}
\mathbb{P}\bigg(\mathbb{E}_{\nu}\{|\hat{\eta}_n(X) - \hat{\eta}(X)| \}   - \mathbb{E}_{\mu^n}\{\mathbb{E}_{\nu}\{|\hat{\eta}_n(X) - \hat{\eta}(X)|\} \}  > \frac{\varepsilon}{6} \bigg) \leq  e^{-\frac{n \varepsilon^2}{18 \beta^2}}.
\end{align}
Substituting the equations \eqref{eqn:one_bound} and \eqref{eqn:second_bound} in the equation \eqref{eqn:bound_suc}, we obtain
\begin{align*}
\mathbb{P}\bigg( \mathbb{E}_{\nu}\{|\eta(X) - \eta_n(X)|\} > \frac{\varepsilon}{2} \bigg) \leq 2 e^{-\frac{ n\varepsilon^2}{18 \beta^2}}. 
\end{align*}
\end{proof}
From the Theorem \ref{thm:suc_finite_dim}, it follows that the $k$-nearest neighbor rule is strongly consistent in any separable metrically finite dimensional space.
\begin{corollary}
\label{cor:sc_fd}
Under the assumption of zero probability of ties, the $k$-nearest neighbor rule is strongly consistent on a metrically finite dimensional separable space.
\end{corollary}
\begin{proof}
Let $(Q,\rho)$ be a separable metric space and suppose $Q$ has finite metric dimension $\beta$ on scale $s$. Let $\varepsilon >0$ be any real number. Let $F_n$ denote the event $\{ \ell(g_n) - \ell^* > \varepsilon\}$. From the Theorem \ref{thm:suc_finite_dim}, we have
\begin{align*}
\mathbb{P}( F_{n} ) \leq 2 e^{-\frac{n\varepsilon^2}{18 \beta^2}},
\end{align*}
Taking sum on the both sides, we get
\begin{align*}
\sum_{n=1}^{\infty} \mathbb{P}( F_n ) & \leq 2 \sum_{n=1}^{\infty} e^{-\frac{n\varepsilon^2}{18 \beta^2}} \ \\
& < + \infty.
\end{align*}
By Borel-Cantelli lemma, we have
\begin{align}
\label{eqn:error_strong}
\mathbb{P}\bigg(\limsup_{n \rightarrow \infty} F_n \bigg)= 0.
\end{align}
This means almost surely for any infinite sample path, the difference of error probabilities of the $k$-nearest neighbor rule and the Bayes rule converges to zero. That is,
\begin{align*}
\mu^{\infty}\bigg\{ \sigma^{\infty} \in (Q \times \{0,1\})^{\infty}: \limsup_{n \rightarrow \infty} \ell_{\mu}(g_n| \sigma_n) - \ell^*_{\mu} = 0 \bigg\} = 1.
\end{align*}
\end{proof}
%%%%%%%%%%%%%%%%%%%%%%%%%%%%%%%%%%%%%%%%%%%%%%%%%%%%%%%%%%%%%%%%%%%%%%%%%%%%%%%%%%%%%%%
\chapter{Future Prospects}
\label{chap:Future Prospects}
%%%%%%%%%%%%%%%%%%%%%%%%%%%%%%%%%%%%%%%%%%%%%%%%%%%%%%%%%%%%%%%%%%%%%%%%%%%%%%%%%%%%%%%
We examine the following diagram.

%%%%%%%%%%%%%%%%%%%%%%%%%%%%%%%%%%%%%%%%%%%%%%%%%%%%%
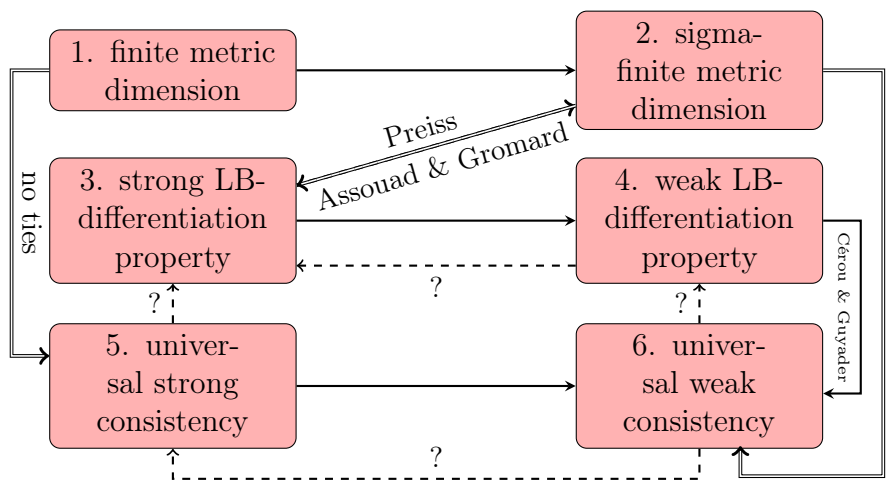
\vspace{0.6cm}
\begin{tikzpicture}[node distance=2cm]
\tikzstyle{arrow} = [thick,->,>=stealth]
\tikzstyle{startstop} = [rectangle, rounded corners, minimum width=3cm, minimum height=1cm,text centered, draw=black, fill=red!30]
\tikzstyle{io} = [rectangle, rounded corners, minimum width=3cm, minimum height=1cm,text centered,text width=3cm, draw=black, fill=red!30]

\tikzstyle{line} = [draw, -latex']

\node (start) [io] {2. sigma-finite metric dimension};
\node (in5) [io,xshift=-5cm, left of=start] {1. finite metric dimension};
\node (in1) [io, below of=in5] {3. strong LB-differentiation property};
\node (in2) [io, below of=in1,yshift=-0.2cm] {5. universal strong consistency};
\node (in3) [io, below of=start] {4. weak LB-differentiation property};
\node (in4) [io, below of=in3,yshift=-0.2cm] {6. universal weak consistency};

\draw[double,->] (in5.west) -- ++(-0.5cm,0) -- ++ (0,-3.8cm) node[pos=0.5,sloped,anchor = center,above] {{\small no ties}} -- ++(0.5cm,0) (in4);
\draw[arrow] (in5) -> (start);
\draw [double,->] (in1) -> node[pos=0.5,sloped,anchor=center,above]{{\small Preiss} \ \ }(start);
\draw[dashed,double,->] (start) ->  node[pos=0.5,sloped,anchor=center,below]{{\small Assouad \& Gromard}} (in1);
\draw [arrow] (in1) -> (in3);
\draw [dashed,thick,black,->] (in3.200) -> node[anchor=north]{{\small ?} }(in1.340);
\draw [dashed,thick,black,->] (in4.south)  -- ++(0,-0.4cm) -- ++(-7cm,0)node[pos=0.5,sloped,anchor = center,above] {{\small ?}} -- ++(0,0.4cm)  (in2.south);
\draw[arrow] (in2) -> (in4);
\draw [dashed,thick,->] (in4)  -- node[anchor=east] {{\small ?}}(in3);
\draw [dashed,thick,->] (in2)  -- node[anchor=east] {{\small ?}} (in1);
\draw [arrow] (in3.east) -- ++(0.5cm,0) -- ++(0,-2.3cm) node[pos=0.5,sloped,anchor = center,below] {{\tiny C{\'e}rou \& Guyader}}-- ++(-0.5cm,0)  (in4);
\draw[double,->] (start.east) -- ++(0.8cm,0) -- ++(0,-5.4cm) -- ++(-1.9cm,0) -- ++(0,0.4cm) (in4);
\end{tikzpicture}
\vspace{0.6cm}
%%%%%%%%%%%%%%%%%%%%%%%%%%%%%%%%%%%%%%%%%%%%%%%%%%%%%

Our main aim is to prove as many as implications as possible in the above flow diagram. The double lines in the above diagram represent our results.

In this dissertation, we have accomplished the following implications: $2 \Rightarrow 6$, $3 \Rightarrow 2$ and partially $1 \Rightarrow 5$ under the assumption of no ties. The implications $2 \Rightarrow 6$ can also be obtained by $2 \Rightarrow 3 \Rightarrow 4 \Rightarrow 6$, but we gave a direct proof without using any other implications. Apart from these, we have some other interesting results such as Lemma \ref{lem:nagata_fails_stone_lemma} and Example \ref{eg:ties_high_prob} which show that the solution for distance ties in Euclidean spaces does not extend to metric spaces with finite Nagata dimension. We also showed the inconsistency of the $k$-nearest neighbor rule on Davies's example in Subsection \ref{subsec:inconsistent_Davies}.

We outline a possible number of research directions (some are represented by question mark in the flow diagram) based on this thesis:
\begin{enumerate}[(I)]
\item $1 \Rightarrow 5$, $2 \Rightarrow 5$: We proved $1 \Rightarrow 5$ under the additional assumption of zero probability of ties. We would like to extend this result to a metrically sigma-finite dimensional space, under the assumption of zero probability of ties. The next step would be to forgo this assumption on distance ties and prove the universal strong consistency in metrically finite and sigma-finite dimensional spaces.  
\item $5 \Rightarrow 3$, $6 \Rightarrow 4$: We would like to prove  these two implications which seem parallel to each other. The proof of $6 \Rightarrow 4$ would be a converse of C{\'e}rou and Guyader's result on universal weak consistency and hence proves the equivalence between weak Lebesgue-Besicovitch differentiation property and universal weak consistency in a metrically sigma-finite dimensional space.  In \cite{Cerou_Guyader_2006}, a partial argument has been done for $6 \Rightarrow 3$. We would like to give a complete proof of this implication $6 \Rightarrow 3$, which will prove the implication $5 \Rightarrow 6 \Rightarrow 3 \Rightarrow 4$. There could be a possibility of proving $3 \Rightarrow 5$ directly, which is similar to $4 \Rightarrow 6$ (a result of C{\'e}rou and Guyader \cite{Cerou_Guyader_2006}) but with stronger form of convergence.
\item $6 \Rightarrow 5$: In Euclidean spaces, the strong and weak consistency of the $k$-nearest neighbor rule are equivalent. It would be interesting to find an example of a metric space such that the $k$-nearest neighbor rule is weakly consistent but not strongly consistent. The equivalence of universal weak and strong consistency in a metrically sigma-finite (or even finite) dimensional space is an advance question because most of the mathematical tools available now are limited to Euclidean spaces. We state $6 \Rightarrow 5$ as an open question. If $6 \Rightarrow 5$ and $5 \Rightarrow 3$ are true then, it answers the open question by Preiss ($4 \Rightarrow 2$) in affirmative, that is, the two notions of strong and weak Lebesgue-Besicovitch differentiation property  are equivalent in a metrically sigma-finite dimensional metric space. In general metric spaces, these are not equivalent \cite{Mattila_1971}.
\item Davies \cite{Davies_1971} constructed an interesting example of an infinite dimensional compact metric space (homeomorphic to a Cantor space) and two Borel measures which are equal on every closed balls, that fails the strong Lebesgue-Besicovitch differentiation property. Later in 1981, Preiss \cite{Preiss_1981} constructed an example of a Gaussian measure in a Banach space which fails the strong Lebesgue-Besicovitch density property.  In our knowledge, these are the only known explicit examples of infinite dimensional spaces where the differentiation property fails. The intuition fail drastically in infinite dimensional spaces. So, we would like to construct a much simpler example of an infinite dimensional metric space which fails the Lebesgue-Besicovitch density property and thus the $k$-nearest neighbor rule fails to be consistent.  

In particular, we believe that Hilbert cube may be a candidate for such an example. A Hilbert cube $W$ is the set of sequences $\{x = (x_1,x_2,\ldots): 0 \leq x_i \leq 1/i, i \in \mathbb{N} \}$. As $W$ is subspace of $\ell^2$ and so it inherits the metric, 
\begin{align*}
\rho(x,y) & = \sqrt{ \sum_{i=1}^{\infty} |x_i - y_i|^2 } \ \ \ \ \text{ for all } x,y \in W.
\end{align*}
Let $\lambda$ be the Lebesgue measure on $\mathbb{R}$ and let $\{ ([0,1/i],\mathcal{B}_i,i\lambda) \}_{i \in \mathbb{N}}$ be a family of measure spaces such that $i\lambda$ is a probability measure and $\mathcal{B}_i$ is a Borel $\sigma$-algebra on $[0,1/i]$. Then consider the product of measurable spaces $(W,\mathcal{B}) = \ ( \prod_{i\in \mathbb{N}} [0,1/i] , \prod_{i \in \mathbb{N}} \mathcal{B}_i)$ equipped with the  product measure $\mu$ : 
\begin{align*}
\mu(E) & = \prod_{i \in \mathbb{N} }{\mu_i(E_i)},   
\end{align*}
where $E =  \prod_{i \in \mathbb{N}} E_i $ with $E_i \in \mathcal{B}_{i}$.

We would like to find a subset $M$ of $W$ such that $\mu(M)<1$ and
\begin{align*}
\lim_{r \rightarrow 0} \frac{\mu(M \cap B(x,r))}{ \mu(B(x,r))} = 1,  
\end{align*}
for $\mu$-almost every $x \in W$.
\end{enumerate}
%%%%%%%%%%%%%%%%%%%%%%%%%%%%%%%%%%%%%%%%%%%%%%%%%%%%%%%%%%%%%%%%%%%%%%%%%%%%%%%%%%%%%%%
\begin{appendices}
\chapter{}
%%%%%%%%%%%%%%%%%%%%%%%%%%%%%%%%%%%%%%%%%%%%%%%%%%%%%%%%%%%%%%%%%%%%%%%%%%%%%%%%%%%%%%%
\section{Auxiliary notions and results}
%%%%%%%%%%%%%%%%%%%%%%%%%%%%%%%%%%%%%%%%%%%%%%%%%%%%%
\begin{lemma}[\cite{Fitzpatrick_2006}]
\label{lem:uni_cts_lip_cts}
Let $(\Omega,\rho)$ be a metric space and $Q \subseteq \Omega$. Let $f : Q \rightarrow [0,1]$ be a uniformly continuous function. Then there exists a uniformly equivalent metric $\rho'$ defined on $\Omega$ such that $f$ is a $1$-Lipschitz continuous function on $Q$ with respect to $\rho'$.
\end{lemma}
\begin{proof}
We want to define $\rho'$ such that for any $\varepsilon >0$ and for every $x,y \in Q$, if $\rho'(x,y) < \varepsilon$ then $|f(x) - f(y)| < \varepsilon $. 
As, $f$ is uniformly continuous, for any $\varepsilon > 0$, there exists $\delta_{\varepsilon}$ such that for any $x,y \in Q$ if $\rho(x,y) < \delta_{\varepsilon}$ then $|f(x) - f(y)| < \varepsilon$. 
Define a function $\mathcal{E} : [0,\infty) \rightarrow [0,1]$ such that for $\delta \leq 1$,
$$  \mathcal{E}(\delta) := \sup_{x,y \in Q}\bigg\{ |f(x) - f(y)| : \rho(x,y) \leq \delta \bigg\},  $$
and $\mathcal{E}(\delta) = 1$ for $\delta > 1$. The function $\mathcal{E}$ is the maximum oscillation of $f$ on any subsets of $Q$ of diameter at most $\delta < 1$, otherwise 1. The function $\mathcal{E}$ is well defined and a monotonically non-decreasing function. From the definition, we have $\mathcal{E}(0) = 0$.
Suppose we define $\rho'(x,y) = \mathcal{E}( d(x,y))$, then in order to prove the triangle inequality for $\rho'$ we need the following inequality 
$$   \mathcal{E}(a + b) \leq \mathcal{E}(a) + \mathcal{E}(b) \ \ \text{for any $a,b \in \Omega$}. $$
But the above inequality is not true: let $\Omega = \{0,1/2,1\}$ and define the distance $\rho(0,1/2) = 1/2, \rho(1/2,1) = 3/2$ and $\rho(0,1) = 1$. Let $f(x) = x^2/2$ be the function on $\Omega$. Then $f$ is a uniformly continuous function. Take $a = b = 1/2$. Then $\mathcal{E}(1/2) = \sup\{|f(x) - f(y)| : \rho(x,y) \leq 1/2\} = 1/8$ but $\mathcal{E}(1) = 1/2 \geq \mathcal{E}(1/2) + \mathcal{E}(1/2) $. So, we try to construct a function $\mathcal{E}'$ such that $\mathcal{E}'\geq \mathcal{E}$ and $\mathcal{E}'(a +b) \leq \mathcal{E}'(a) + \mathcal{E}'(b)$. We take the concave majorant of $\mathcal{E}$, 
\begin{align*}
\mathcal{E}'(\delta) = \sup\bigg\{ t \mathcal{E}(a) + (1-t) \mathcal{E}(b):  a,b,t \in [0,1] , \delta = t a + (1-t)b \bigg\}.
\end{align*}
So, we have the following properties: 
\begin{enumerate}[(i)]
\item $\mathcal{E}'(0)= 0$.  \\ 
We can write $0= ta + (1-t)b$ which is true whenever $a =0 = b$ or $t=0=b$ or $t=1,a=0$. In all these cases, $t \mathcal{E}(a) + (1-t)\mathcal{E}(b) = 0$ because $\mathcal{E}(0) =0$ and hence $\mathcal{E}'(0) =0$.
\item $\mathcal{E}' \geq \mathcal{E}$. \ \\ 
For any $\delta \in [0,\infty)$, $\mathcal{E}'(\delta) = \sup\{ t \mathcal{E}(a) + (1-t) \mathcal{E}(b) : a,b,t \in [0,1], \delta = t a + (1-t)b\}$. Take $a = \delta,t =1$, then $\mathcal{E}'(\delta) \geq \mathcal{E}(\delta)$. 
\item $\mathcal{E}'(\delta) \downarrow 0$ whenever $\delta \downarrow 0$. \\
Suppose $\delta \downarrow 0$ but $\mathcal{E}'(\delta)$ does not decrease to $0$.  Then there exist sequences $a_n,t_n,b_n$ such that $ t_{n} a_{n} + (1- t_{n}) b_{n} \downarrow 0$ but $t_{n} \mathcal{E}(a_{n}) + (1- t_{n}) \mathcal{E}(b_{n}) \geq C $ for some constant $C > 0$. This means that either $t_{n},b_{n}\downarrow 0$ or $a_{n},b_n  \downarrow 0 $ or $ a_n\downarrow 0, t \uparrow 1$. If $a_{n} \downarrow 0$, then $ t_{n} \mathcal{E}(a_{n}) \downarrow 0$ but $ (1-t_{n}) \mathcal{E}(b_{n}) \geq C $ which contradicts $ b_n \downarrow 0 $. We get similar contradiction for other cases also. Hence, $\mathcal{E}'(\delta) \downarrow 0$.
\item \textit{Claim}: Let $\delta =\sum_{i = 1}^{n} t_i a_i $ such that $\sum_{i=1}^{n}t_i = 1$, then there exist $c \leq d$ and $t$ from $[0,1]$ such that $\delta = \sum_{i = 1}^{n} t_i a_i  = tc + (1-t)d$ and $\sum_{i=1}^{n} t_i \mathcal{E}(a_i) \leq t \mathcal{E}(c) + (1-t) \mathcal{E}(d)$.

Let $c_1,\ldots ,c_n \leq \delta $ and $d_1,\ldots ,d_n \geq \delta$. For every pair $(c_i,d_j)$ set 
$$  t = \frac{\delta - c_i}{d_j - c_i}, \ $$
then $\delta = t c_i +(1-t)d_j$. Choose $(l,k) $ such that $t \mathcal{E}(c_l) + (1-t)\mathcal{E}(d_k)$ is the maximum. 

The point $(\delta, \sum_{i=1}^{n} t_i \mathcal{E}(a_i))$ belongs to the convex combination of the points $(a_i,\mathcal{E}(a_i))$. This is a convex polygon. The point $(\delta, \sum_{i=1}^{n} t_i \mathcal{E}(a_i))$ is on the edge joining $(c_i,\mathcal{E}(c_i))$ and $(d_j,\mathcal{E}(d_j))$, then
\begin{align*}
 \delta & = t c_i + (1-t) d_j,
\end{align*}
 and 
\begin{align*}
\sum_{i=1}^{n} \mathcal{E}(a_i) & = t \mathcal{E}(c_i) + (1-t) \mathcal{E}(d_j) \ \\
& \leq t \mathcal{E}(c_l) + (1-t) \mathcal{E}(d_k).
\end{align*}
\item $\mathcal{E}'$ is a concave function, that is, for any $\delta_1,\delta_2 \in [0,\infty)$ and $\alpha \in [0,1]$ 
$$\mathcal{E}'(\alpha \delta_1 + (1-\alpha)\delta_2) \geq \alpha \mathcal{E}'(\delta_1) + (1-\alpha) \mathcal{E}'(\delta_2). $$	
Let $\gamma > 0$. For $\delta_1$, there exist $a_1,b_1,t_1 \in [0,1]$ such that $\delta_1 = t_1 a_1 + (1-t_1)b_1$ and  $ t_1 \mathcal{E}(a_1) + (1-t_1) \mathcal{E}(b_1) > \mathcal{E}'(\delta_1) - \gamma$. Similarly for $\delta_2 > 0$ there exist $a_2,b_2,t_2 \in [0,1]$ such that $\delta_2 = t_2 a_2 + (1-t_2)b_2$ and  $ t_2 \mathcal{E}(a_2) + (1-t_2) \mathcal{E}(b_2) > \mathcal{E}'(\delta_2) - \gamma$. We have,
\begin{align*}
\alpha \mathcal{E}'(\delta_1) + (1 - \alpha) \mathcal{E}'(\delta_2) & \leq \alpha (t_1 \mathcal{E}(a_1) + (1-t_1) \mathcal{E}(b_1) ) + (1-\alpha)( t_2\mathcal{E}(a_2) \\ & \ \  + (1-t_2) \mathcal{E}(b_2)) + \gamma \ \\
& < \alpha t_1 \mathcal{E}(a_1) + \alpha (1-t_1)\mathcal{E}'(b_2) + (1-\alpha)t_2 \mathcal{E}(a_2) \\ & \ \ + (1-\alpha)(1-t_2)\mathcal{E}(b_2) + \gamma. \ 
\end{align*}
From property (iv), there exist $c,d$ and $t$ such that $\alpha \delta_1 + (1-\alpha) \delta_2 = tc + (1-t )d$ and 
\begin{align*}
\alpha \mathcal{E}'(\delta_1) + (1 - \alpha) \mathcal{E}'(\delta_2) & \leq \alpha t_1 \mathcal{E}(a_1) + \alpha (1-t_1)\mathcal{E}(b_2) + (1-\alpha)t_2 \mathcal{E}(a_2) \\ & \ \ +(1-\alpha)(1-t_2)\mathcal{E}(b_2) + \gamma \ \\
& \leq t \mathcal{E}(c) + (1-t) \mathcal{E}(d) + \gamma \ \\
& \leq \mathcal{E}'(t c + (1-t)d) + \gamma \ \\
& = \mathcal{E}'(\alpha \delta_1+ (1-\alpha)\delta_2) + \gamma. 
\end{align*}
As $\gamma$ is arbitrary, we have $ \mathcal{E}'(\alpha \delta_1 + (1-\alpha) \delta_2) \geq \alpha \mathcal{E}'(\delta_1) + (1-\alpha) \mathcal{E}'(\delta_2)$. 
\item $\mathcal{E}'(\delta_1 + \delta_2) \leq \mathcal{E}'(\delta_1) + \mathcal{E}'(\delta_2)$. \\
Let $\alpha \in [0,1]$. As $\mathcal{E}'$ is a concave function,
\begin{align*}
\mathcal{E}'(\alpha \delta) & = \mathcal{E}'(\alpha \delta + (1-\alpha)0) \ \\
& \geq \alpha \mathcal{E}'(\delta) +(1-\alpha) \mathcal{E}'(0) \ \\
& = \alpha \mathcal{E}'(\delta).
\end{align*}
Then,
\begin{align*}
\mathcal{E}'(\delta_1) + \mathcal{E}'(\delta_2) & = \mathcal{E}'\bigg((\delta_1 + \delta_2) \frac{\delta_1}{\delta_1 + \delta_2} \bigg)  +\mathcal{E}'\bigg((\delta_1 + \delta_2) \frac{\delta_2}{\delta_1 + \delta_2} \bigg) \ \\
& \geq  \frac{\delta_1}{\delta_1 + \delta_2}\mathcal{E}'(\delta_1 + \delta_2) + \frac{\delta_2}{\delta_1 + \delta_2}\mathcal{E}'(\delta_1 + \delta_2) \\
& = \mathcal{E}'(\delta_1 + \delta_2).
\end{align*}
\item $\mathcal{E}'$ is a monotonically non-decreasing function. \\
Let $\gamma >0$ and assume $\delta_1 < \delta_2$. So, there exists $a \leq b$ such that $\delta_1 = t a + (1-t)b < \delta _2$ and $ t \mathcal{E}(a) +(1-t)\mathcal{E}(b) > \mathcal{E}'(\delta_1) - \gamma$. Then, there is a $b_2 \geq b$ such that $\delta_2 = t a + (1-t)b_2$. We have,
\begin{align*} 
\mathcal{E}'(\delta_2) & \geq t \mathcal{E}(a) + (1-t) \mathcal{E}(b_2) \ \\
& \geq t \mathcal{E}(a) + (1-t) \mathcal{E}(b)  \ \\
& > \mathcal{E}('\delta_1) - \gamma.
\end{align*}
As $\gamma$ is arbitrary, $\mathcal{E}'(\delta_2) \geq \mathcal{E}'(\delta_1)$.
\end{enumerate}
Now, we define the function $\rho'(x,y) =  \mathcal{E}'(\rho(x,y)) + \rho(x,y) $ for any $x \neq y \in \Omega$. From the property (vi) of $\mathcal{E}'$, it follows that $\rho'$ is a metric. 

Let $\alpha > 0$ and $\gamma_1 = \alpha,\gamma_2 = \alpha/2$. For all $x,y \in \Omega$, if $\rho'(x,y) < \gamma_1 = \alpha$, then $\rho(x,y) < \alpha$. And if $\rho(x,y) < \gamma_2$, then by the property (iii) of $\mathcal{E}'$ we have $\mathcal{E}'(\rho(x,y)) < \alpha/2$. This gives $\rho'(x,y) = \mathcal{E}'(\rho(x,y)) + \rho(x,y) < \gamma_2 + \alpha/2 = \alpha$. Therefore, $\rho$ and $\rho'$ are uniformly equivalent metrics. 

Let $x,y \in Q$ and suppose $\rho'(x,y) < \varepsilon $, then $\mathcal{E}'(\rho(x,y)) < \varepsilon$. This implies $\mathcal{E}(\rho(x,y)) < \varepsilon $ and so, $ |f(x) - f(y)| < \varepsilon $ which means $f$ is a $1$-Lipschitz continuous function.
\end{proof}

\begin{lemma}[\cite{Fitzpatrick_2006}]
\label{lem:lip_cts_lip_cts}
Every $1$-Lipschitz continuous function $f : Q \rightarrow [0,1]$ can be extended to a $1$-Lipschitz continuous function $\bar{f} : \Omega \rightarrow [0,1]$ in the following way,
$$ \bar{f}(x) := \min\bigg\{ 1, \inf_{y\in Q}\{f(y) + \rho(x,y) \}  \bigg\}.  $$
\end{lemma}
\begin{proof} Let $\gamma>0$ and $x_1,x_2 \in \Omega$. 
There are mainly three cases:
\begin{enumerate}[(i)]
\item If $\bar{f}(x_1) = 1 = \bar{f}(x_2)$, then it is trivial.
\item If $\bar{f}(x_1) =1$ and $\bar{f}(x_2) = \inf_{y\in Q}\{f(y) + \rho(x_2,y) \}$, then there exists $y_2 \in Q$ such that  $ 1 \geq \bar{f}(x_2) > f(y_2) + \rho(x_2,y_2) - \gamma $. So, 
\begin{align*}
|\bar{f}(x_1)- \bar{f}(x_2)| & = 1 - f(y_2) - \rho(x_2,y_2) +\gamma  \ \\
& \leq f(y_2) + \rho(x_1,y_2)-f(y_2) - \rho(x_2,y_2) + \gamma\ \\
& = \rho(x_1,y_2) - \rho(x_2,y_2) + \gamma \ \\
& \leq \rho(x_1,x_2) + \gamma.
\end{align*}
\item If $\bar{f}(x_1) = \inf_{y\in Q}\{f(y) + \rho(x_1,y) \}$ and $\bar{f}(x_2) = \inf_{y\in Q}\{f(y) + \rho(x_2,y) \}$, then there exist $y_1, y_2 \in Q$ such that $\bar{f}(x_1) \leq f(y_1) + \rho(x_1,y_1) $ and $\bar{f}(x_2) >  f(y_2) + \rho(x_2,y_2) \} - \gamma$. Therefore,
\begin{align*}
|\bar{f}(x_1)- \bar{f}(x_2)| & = |f(y_1) + \rho(x_1,y_1) - f(y_2) - \rho(x_2,y_2) + \gamma | \ \\
& \leq | f(y_2) + \rho(x_1,y_2) -f(y_2) - \rho(x_2,y_2) + \gamma|\ \\
& = \rho(x_1,y_2) - \rho(x_2,y_2) + \gamma \ \\
& \leq \rho(x_1,x_2) + \gamma.
\end{align*}
\end{enumerate}
As $\gamma$ is arbitrary, $\bar{f}$ is a 1-Lipschitz continuous function.
\end{proof}
We state the important Luzin's theorem in our settings for better understanding. 
\begin{theorem}[Luzin's theorem \cite{Folland_1999}]
\label{thm:Luzin_theorem}
Let $\eta: Q \rightarrow [0,1]$ be measurable function and let $\nu$ be a probability measure on $\Omega$, where $(\Omega,\rho)$ is a separable metric space and $Q \subseteq \Omega$. Given $\varepsilon >0$, there exists a compact set $K \subseteq Q$ such that $\nu(Q \setminus K) < \varepsilon$ and $\eta|_{K}$  is a uniformly continuous function. 
\end{theorem}

\begin{definition}[Paracompact space \cite{Engelking_1989}]
A topological space $\Omega$ is said to be paracompact if every open cover of $\Omega$ has a locally finite open refinement. That is, if $\Omega \subseteq \cup_{i \in I} O_i$, where each $O_i$ is an open set, then there is a collection of open sets $\{V_j : V_j \text{ is open }, j \in J\}$ such that 
\begin{enumerate}[(i)]
\item $\cup_{j \in J} V_j$ is an open cover for $\Omega$,
\item each $V_j$ is a subset of $O_i$ for some $i$ in $I$ and,
\item every element $x$ of $\Omega$ has a neighborhood around $x$ which intersects finitely many $V_j, j \in J$. 
\end{enumerate}
A cover of $\Omega$ is locally finite if it satisfies the above stated property (iii).
\end{definition}
\begin{lemma}[ \cite{Engelking_1989}]
\label{app:para_space}
Every metric space is paracompact.
\end{lemma}

\begin{lemma}[Dieudonn{\'e}'s theorem \cite{Engelking_1989}]
\label{app:countable_closed}
Let $\{V_j: V_j \text{ is open}, j \in \mathbb{N}\}$ be a locally finite countable family. Then, closure of union of $V_j$ is the union of $\bar{V}_{j}$.
\end{lemma}
\begin{definition}[\cite{Engelking_1989}]
A topological space $\Omega$ is called a Lindel{\"o}f space if any open cover of a subset of $\Omega$ has a countable subcover. 
\end{definition}

\begin{lemma}[Lindel{\"o}f theorem \cite{Engelking_1989}]
\label{app:lindelof}
A metric space is a Lindel{\"o}f space if and only if it is separable.
\end{lemma}

\begin{theorem}[Baire Category Theorem \cite{Engelking_1989}]
\label{lem:inter_dense}
Let $(\Omega,\rho)$ be a complete metric space and $\{D_n\}_{n \in \mathbb{N}}$ be a sequence of dense open sets. Then $\cap_{n \in \mathbb{N}} D_n$ is dense.
\end{theorem}

\theoremstyle{theorem}
\begin{lemma}
\label{eg:non_archi_metric}
Let $\Omega = \{ \sigma: \sigma = (\sigma_1,\sigma_2,\ldots)  \}$ and let $\rho$ be defined as, for any $\sigma,\tau \in \Omega$,
\begin{align*}
\rho(\sigma ,\tau) = \ 
\begin{cases}
0 \ \ \ \ \hspace{2.1cm} \text{if $\sigma = \tau$ } \ \\
2^{- \min\{i : \ \sigma_i \neq \tau_i \} } \ \ \ \text{otherwise }
\end{cases}
\end{align*}
Then, $\rho$ is a non-Archimedean metric and $(\Omega, \rho)$ is called a non-Archimedean metric space. 
\end{lemma}
\begin{proof}
By the definition, the function $\rho$ is symmetric and non-negative. Also, $\rho(\sigma,\tau) = 0 $ iff $\sigma = \tau$. We will show that for any $\sigma,\tau,\gamma \in \Omega$, 
$$ \rho(\sigma,\tau) \leq \max \{ \rho(\sigma,\gamma) ,\rho(\gamma,\tau)\}. $$ 
If $\gamma = \sigma$ or $\gamma = \tau$, then the above inequality follows easily. Suppose $\gamma \neq \sigma \neq \tau$ and $ \max \{ \rho(\sigma,\gamma) ,\rho(\gamma,\tau) \} =  \rho(\sigma,\gamma)$. Let $\rho(\sigma,\gamma) = 2^{-i}$. Then, $\sigma_j = \gamma_j $ for $j < i$ and $\sigma_i \neq \gamma_i$. Since $\rho(\sigma,\gamma) \geq \rho(\tau,\gamma)$, this means $\tau_j = \gamma_j$ for $j < i$. So, $\sigma_j = \tau_j $ for $j < i$. So, $\rho(\sigma,\tau) \leq 2^{-i} = \max \{ \rho(\sigma,\gamma) ,\rho(\gamma,\tau)\}$. Similarly, the strong triangle inequality holds when $ \max \{ \rho(\sigma,\gamma) ,\rho(\gamma,\tau)\} =  \rho(\tau,\gamma)$. Hence, $(\Omega,\rho)$ is a non-Archimedean metric space.
\end{proof}
\begin{definition}
\label{app:discrete_metric}
Let $\rho$ on $\Omega$ be defined as: for $x,y \in \Omega$, $\rho(x,y) = 1$ if and only if $x \neq y$. Then $\rho$ is called a 0-1 metric. 
\end{definition}

\begin{theorem}[McDiarmid's inequality \cite{Mcdiarmid_1989}]
\label{thm:mcdiarmid}
Let $(X_1,Y_1),\ldots, (X_n,Y_n)$ be independent pair of random variables taking values in $\Omega\times \{0,1\}$. Let $f$ be a real-valued function defined on $(\Omega \times \{0,1\})^n$ such that for every $1\leq i \leq n$, and for all $(x_1,y_1), \ldots,(x_n,y_n), (\hat{x}_i,\hat{y}_i) \in \Omega \times \{0,1\}$,
\begin{align*}
& \bigg| f\bigg((x_1,y_1),\ldots,(x_i,y_i),\ldots,(x_n,y_n)\bigg)- f\bigg((x_1,y_1),\ldots,(\hat{x}_i,\hat{y}_i),\ldots, (x_n,y_n)\bigg) \bigg| \\ & \leq \alpha_i, 
\end{align*} 
Then, for $\varepsilon >0$
\begin{align*}
\mathbb{P}\bigg( f((x_1,y_1),\ldots,(x_n,y_n)) - \mathbb{E}\{f((x_1,y_1),\ldots,(x_n,y_n))\} \geq \varepsilon \bigg) \leq e^{\frac{-2 \varepsilon^2}{\sum_{i=1}^{n} \alpha_{i}^2}}.
\end{align*}
\end{theorem}

\end{appendices}
%%%%%%%%%%%%%%%%%%%%%%%%%%%%%%%%%%%%%%%%%%%%%%%%%%%%%%%%%%%%%%%%%%%%%%%%%%%%%%%%%%%%%%%

%%%%%%%%%%%%%%%%%%%%%%%%%%%%%%%%%%%%%%%%%%%%%%%%%%%%%%%%%%%%%%%%%%%%%%%%%%%%%%%%%%%%%%%
\bibliographystyle{spmpsci} 
\bibliography{Doctoral_thesis_Sushma_1807}

\begin{thebibliography}{10}
\providecommand{\url}[1]{{#1}}
\providecommand{\urlprefix}{URL }
\expandafter\ifx\csname urlstyle\endcsname\relax
  \providecommand{\doi}[1]{DOI~\discretionary{}{}{}#1}\else
  \providecommand{\doi}{DOI~\discretionary{}{}{}\begingroup
  \urlstyle{rm}\Url}\fi

\bibitem{Assouad_Gromard_2006}
Assouad, P., Quentin~de Gromard, T.: Recouvrements, derivation des mesures et
  dimensions.
\newblock Revista Matemática Iberoamericana \textbf{22}(3), 893--953 (2006)

\bibitem{Besicovitch_1945}
{Besicovitch}, A.: {A general form of the covering principle and relative
  differentiation of additive functions.}
\newblock {Proceedings of the Cambridge Philosophical Society} \textbf{41},
  103--110 (1945)

\bibitem{Biau_Devroye_Lugosi_2008}
Biau, G., Devroye, L., Lugosi, G.: Consistency of random forests and other
  averaging classifiers.
\newblock Journal of Machine Learning Research \textbf{9}, 2015--2033 (2008)

\bibitem{Billingsley_1999}
Billingsley, P.: Convergence of probability measures, second edn.
\newblock Wiley Series in Probability and Statistics: Probability and
  Statistics. John Wiley \& Sons Inc. (1999).
\newblock A Wiley-Interscience Publication

\bibitem{Billingsley_2012}
Billingsley, P.: Probability and Measure, anniversary edition edn.
\newblock Wiley Series in Probability and Statistics. John Wiley \& Sons Inc.
  (2012)

\bibitem{Cerou_Guyader_2006}
C{\'{e}}rou, F., Guyader, A.: {Nearest Neighbor Classification in infinite
  dimension}.
\newblock ESAIM: Probability and Statistics \textbf{10}, 340--355 (2006)

\bibitem{Collins_2003}
Collins, B.: Moments and cumulants of polynomial random variables on
  unitarygroups, the itzykson-zuber integral, and free probability.
\newblock International Mathematics Research Notices \textbf{2003}(17),
  953--982 (2003)

\bibitem{Collins_Matsumoto_Saad_2014}
Collins, B., Matsumoto, S., Saad, N.: {Integration of invariant matrices and
  moments of inverses of Ginibre and Wishart matrices}.
\newblock Journal of Multivariate Analysis pp. 1--13 (2014)

\bibitem{Couillet_Debbah_2011}
Couillet, R., Debbah, M.: Random Matrix Methods for Wireless Communications.
\newblock Cambridge University Press (2011)

\bibitem{Cover_Hart_1967}
Cover, T., Hart, P.: Nearest neighbor pattern classification.
\newblock IEEE Transactions on Information Theory \textbf{13}, 21--27 (1967)

\bibitem{Davies_1971}
Davies, R.O.: Measures not approximable or not specifiable by means of balls.
\newblock Mathematika \textbf{18}(2), 157–160 (1971)

\bibitem{Devroye_1981}
Devroye, L.: On the almost everywhere convergence of nonparametric regression
  function estimates.
\newblock The Annals of Statistics \textbf{9}(6), 1310--1319 (1981)

\bibitem{Devroye_Gyorfi_1985}
Devroye, L., Gy{\"{o}}rfi, L.: Nonparametric Density Estimation: The $L_1$
  View.
\newblock John Wiley \& Sons (1985)

\bibitem{Devroye_Gyorfi_Krzyzak_Lugosi_1994}
Devroye, L., Gy{\"{o}}rfi, L., Krzyzak, A., Lugosi, G.: On the strong universal
  consistency of nearest neighbor regression function estimates.
\newblock The Annals of Statistics \textbf{22}(3), 1371--1385 (1994)

\bibitem{Devroye_Gyorfi_Lugosi_1996}
Devroye, L., Gy{\"{o}}rfi, L., Lugosi, G.: {A Probabilistic Theory of Pattern
  Recognition}.
\newblock Stochastic Modelling and Applied Probability, Springer (1996)

\bibitem{Duan_2014}
Duan, H.H.: Applying Supervised Learning Algorithms and a New Feature Selection
  Method to Predict Coronary Artery Disease.
\newblock Masters thesis, University of Ottawa (2014)

\bibitem{Dumitriu_Edelman_2002}
Dumitriu, I., Edelman, A.: Matrix models for beta ensembles.
\newblock Journal of Mathematical Physics \textbf{43}, 5830--5847 (2002)

\bibitem{Engelking_1989}
Engelking, R.: General topology, revised and completed edn.
\newblock Sigma series in pure mathematics. Berlin: Heldermann (1989)

\bibitem{Fitzpatrick_2006}
Fitzpatrick, P.: Advanced Calculus, second edition edn.
\newblock Wiley Series in Probability and Statistics. American Mathematical
  Society (2006)

\bibitem{Fix_Hodges_1951}
Fix, E., Hodges, J.L.: Discriminatory analysis. nonparametric discrimination:
  Consistency properties.
\newblock Technical Report 4, Project Number 21-59-004  (1951)

\bibitem{Folland_1999}
Folland, G.B.: Real Analysis: Modern Techniques and Their Applications.
\newblock Pure and Applied Mathematics: A Wiley Series of Texts, Monogrpahs and
  Tracts. Wiley (1999)

\bibitem{Forrester_2010}
{Forrester}, P.J.: {Log-gases and random matrices.}
\newblock Princeton, NJ: Princeton University Press (2010)

\bibitem{Foucart_Rauhut_2013}
Foucart, S., Rauhut, H.: A Mathematical Introduction to Compressive Sensing.
\newblock Birkh{\"a}user Basel, Springer (2013)

\bibitem{Goodman_1963}
Goodman, N.R.: {Statistical analysis based on a certain multivariate complex
  Gaussian distribution (An Introduction)}.
\newblock The Annals of Mathematical Statistics \textbf{34}(1), 152--177 (1963)

\bibitem{Graczyk_Letac_Massam_2003}
Graczyk, P., Letac, G., Massam, H.: The complex {W}ishart distribution and the
  symmetric group.
\newblock The Annals of Statistics \textbf{31}(1), 287--309 (2003)

\bibitem{Hatko_2015}
Hatko, S.: k-Nearest Neighbour Classification of Datasets with a Family of
  Distances.
\newblock Masters thesis, University of Ottawa (2015)

\bibitem{Horn_Johnson_1986}
Horn, R.A., Johnson, C.R. (eds.): Matrix Analysis.
\newblock Cambridge University Press (1986)

\bibitem{Kumari_2018}
Kumari, S.: Finiteness of inverse moments of $(m,n,\beta)$-laguerre matrices.
\newblock Infinite Dimensional Analysis, Quantum Probability, and Related
  Topics  (accepted, 2018)

\bibitem{Letac_Massam_2004}
Letac, G., Massam, H.: {All Invariant Moments of the {W}ishart Distribution}
  \textbf{31}(2), 295--318 (2004)

\bibitem{Liao_Couillet_2017}
Liao, Z., Couillet, R.: Random matrices meet machine learning: A large
  dimensional analysis of {LS-SVM}.
\newblock International Conference on Acoustics, Speech and Signal Processing
  (ICASSP) pp. 2397--2401 (2017)

\bibitem{Louart_Liao_Couillet_2017}
Louart, C., Liao, Z., Couillet, R.: A random matrix approach to neural
  networks.
\newblock The Annals of Applied Probability \textbf{28}, 1190--1248 (2018)

\bibitem{Mai_Couillet_2017}
Mai, X., Couillet, R.: A random matrix analysis and improvement of
  semi-supervised learning for large dimensional data.
\newblock Journal of Machine Learning Research  (2017)

\bibitem{Matsumoto_2012}
Matsumoto, S.: {General moments of the inverse real Wishart distribution and
  orthogonal Weingarten functions}.
\newblock Journal of Theoretical Probability \textbf{25}(3), 798--822 (2012)

\bibitem{Mattila_1971}
Mattila, P.: {Differentiation of measures in uniform spaces}.
\newblock Measure theory, Oberwolfach pp. 261--283 (1971)

\bibitem{Mcdiarmid_1989}
McDiarmid, C.: {On the method of bounded differences}.
\newblock Cambridge University Press pp. 148--188 (1989)

\bibitem{Mehta_2004}
Mehta, M.L.: Random Matrices, 3rd edn.
\newblock Elsevier, Academic Press, New York (2004)

\bibitem{Mezzadri_Reynolds_Winn_2017}
Mezzadri, F., Reynolds, A.K., Winn, B.: Moments of the eigenvalue densities and
  of the secular coefficients of {$\beta$} -ensembles.
\newblock Nonlinearity \textbf{30}(3), 1034 (2017)

\bibitem{Preiss_1981}
Preiss, D.: Gaussian measures and the density theorem.
\newblock Commentationes Mathematicae Universitatis Carolinae \textbf{022}(1),
  181--193 (1981)

\bibitem{Preiss_1983}
Preiss, D.: Dimension of metrics and differentiation of measures.
\newblock In General Topology and its Relations to Modern Analysis and Algebra
  V, Heldermann Verlag, Berlin pp. 565--568 (1983)

\bibitem{Ross_2006}
Ross, S.M.: Introduction to Probability Models, Ninth Edition.
\newblock Academic Press, Inc. (2006)

\bibitem{Speicher_1998}
Speicher, R.: {Combinatorial theory of the free product with amalgamation and
  operator-valued free probability theory}.
\newblock American Mathematical Society (1998)

\bibitem{Stone_1977}
Stone, C.J.: Consistent nonparametric regression.
\newblock The Annals of Statistics \textbf{5}, 595--620 (1977)

\bibitem{Tao_2013}
Tao, T.: Topics in Random Matrix Theory, Graduate studies in Mathematics, vol.
  132.
\newblock American Mathematical Society (2013)

\bibitem{Wishart_1928}
Wishart, J.: {The generalised product moment distribution in samples from a
  normal multivariate population}.
\newblock Biometrika \textbf{20A}, 32--52 (1928)

\bibitem{Zakai_Ritov_2009}
Zakai, A., Ritov, Y.: Consistency and localizability.
\newblock Journal of Machine Learning Research \textbf{10}, 827--856 (2009)

\bibitem{Zhao_1987}
Zhao, L.: {Exponential bounds of mean error for the nearest neighbor estimates
  of regression functions}.
\newblock Journal of Multivariate Analysis pp. 168--178 (1987)

\end{thebibliography}
%%%%%%%%%%%%%%%%%%%%%%%%%%%%%%%%%%%%%%%%%%%%%%%%%%%%%%%%%%%%%%%%%%%%%%%%%%%%%%%%%%%%%%%
\end{document}